\documentclass[aos,preprint]{imsart}

\RequirePackage{amsthm,amsmath,amsfonts,amssymb}
\RequirePackage{xcolor}
\RequirePackage{xspace}
\RequirePackage[numbers,sort&compress]{natbib}
\RequirePackage[colorlinks,citecolor=blue,urlcolor=blue,linkcolor=black]{hyperref}
\RequirePackage[nameinlink,capitalize,noabbrev]{cleveref}
\AddToHook{env/theorem/begin}{\crefalias{section}{theorem}}
\AddToHook{env/proposition/begin}{\crefalias{section}{proposition}}
\AddToHook{env/lemma/begin}{\crefalias{section}{lemma}}
\AddToHook{env/corollary/begin}{\crefalias{section}{corollary}}
\AddToHook{env/conjecture/begin}{\crefalias{section}{conjecture}}
\AddToHook{env/assumption/begin}{\crefalias{section}{assumption}}
\AddToHook{env/definition/begin}{\crefalias{section}{definition}}
\AddToHook{env/example/begin}{\crefalias{section}{example}}
\AddToHook{env/remark/begin}{\crefalias{section}{remark}}

\hypersetup{
  pdftitle={Estimating Population-Risk Curves Along Nonconvex Gradient Flows from the Training Sample},
  pdfauthor={Mingzhi Song}
}

\startlocaldefs
\newcommand{\R}{\mathbb{R}}
\newcommand{\E}{\mathbb{E}}
\newcommand{\Pp}{\mathbb{P}}
\newcommand{\cH}{\mathcal{H}}

\newcommand{\eps}{\varepsilon}
\newcommand{\dd}{\,\mathrm{d}}
\newcommand{\op}{\mathrm{op}}
\newcommand{\TV}{\mathrm{TV}}
\newcommand{\LOO}{\mathrm{LOO}}

\newcommand{\Risk}{\mathsf{R}}
\newcommand{\Score}{\mathsf{C}}
\newcommand{\FlowALO}{Flow-ALO\xspace}

\newtheorem{theorem}{Theorem}[section]
\newtheorem{proposition}[theorem]{Proposition}
\newtheorem{lemma}[theorem]{Lemma}
\newtheorem{corollary}[theorem]{Corollary}

\newtheorem{assumption}[theorem]{Assumption}
\theoremstyle{definition}

\newtheorem{example}[theorem]{Example}
\makeatletter
\@ifundefined{th@remark}
  {\theoremstyle{definition}}
  {\theoremstyle{remark}}
\makeatother

\endlocaldefs

\begin{document}

\begin{frontmatter}
\pdfsubject{}
\title{Estimating Population-Risk Curves Along Nonconvex Gradient Flows
from the Training Sample}
\runtitle{Population-Risk Curves from the Training Sample}

\begin{aug}
\author[A]{\fnms{Mingzhi}~\snm{Song}\ead[label=e1]{songmingzhi123@gmail.com}}
\address[A]{Department of Mathematics, The University of Hong Kong\printead[presep={,\ }]{e1}}
\end{aug}

\begin{abstract}
We estimate the conditional population-risk curve of a realized smooth
nonconvex gradient flow from the training sample.  Flow approximate
leave-one-out (Flow-ALO) propagates a deletion response and evaluates
omitted observations at approximate deleted paths.  The risk-curve error
decomposes into response approximation, exact-LOO fluctuation, and
deletion-to-full risk transfer.
On each fixed finite horizon, bounded centered training-loss gradients, a
one-sided Hessian lower bound, locally Lipschitz Hessians, and a strict
tube-closure condition yield an explicit $(n-1)^{-2}$ bound for the
deletion-response error.  Bounded evaluation-loss gradients transfer the
deletion-response bound to the score without requiring the Hessian to be
invertible.
Direct first-order jackknife cancellation and exact-LOO concentration control
deletion-to-full risk transfer and fluctuation, respectively, completing
recovery of the conditional population-risk curve.
For bounded smooth two-layer mean-field networks training both layers, the
score-error bound is uniform in width.
 \end{abstract}


\end{frontmatter}

\section{Introduction}\label{sec:introduction}

Can the fitting sample estimate a fitted model's population-risk curve without
external validation?  Training loss reuses the fitted observations; a
validation split preserves independence but leaves fewer observations for
fitting.  Exact leave-one-out (LOO) cross-validation avoids both problems but
requires $n$ additional training trajectories for the full curve.

Static approximate leave-one-out (ALO) replaces endpoint refits by local
primal--dual approximations or a Newton correction
\cite{wang2018approximate,wilson2020approximate}, while infinitesimal-jackknife
methods linearize fitted procedures with respect to data weights
\cite{giordano2019swiss,giordano2019higher}.  Building on these endpoint
approximations and iterative response methods \cite{luo2023iterative}, we ask
when a continuous-time deletion response estimates the
\emph{sample-conditional population-risk curve of the realized smooth
nonconvex fitted path}.  Estimating the sample-conditional curve requires a
distinct three-part argument:
response approximation, exact-LOO concentration, and comparison of
size-$(n-1)$ deletion learners with the realized size-$n$ learner.

\subsection{Contributions}

The central contribution is the high-probability finite-sample comparison in
\cref{thm:oracle} between the Flow-ALO score and the
sample-conditional risk of the realized full-sample gradient flow, uniformly
over a predetermined compact set of training times and fixed-dimensional
hyperparameters within one model, without an external validation sample.
Absolute exact-LOO concentration recovers risk levels; chained exact-LOO
concentration identifies the curve only up to a common sample-dependent shift.
Under either
route, \cref{cor:regularity-event-oracle} bounds the excess risk of a measurable
approximate Flow-ALO minimizer.

The Flow-ALO-to-exact-LOO comparison is a pathwise approximation for nonconvex
training dynamics.  On a common finite-horizon tube, \cref{ass:tube} allows
indefinite deleted-objective Hessians subject to a uniform lower bound and
local Hessian Lipschitzness, with bounded centered forcing and a strict
tube-closure margin.
Under \cref{ass:tube} and the evaluation-gradient bound in
\eqref{eq:evaluation-gradient-bound}, \cref{thm:finite-horizon} gives a uniform
order-$(n-1)^{-2}$ error between the Flow-ALO and exact-LOO scores.

Exact-LOO concentration and deletion-to-full transfer solve the statistical
problem left by a deletion-response approximation.
Under a deterministic finite loss range, uniform replace-one stability, and
uniform continuity for size-$(n-1)$ learners on a compact index set,
\cref{thm:loo-concentration} gives a uniform high-probability bound between the
exact-LOO score and the average population risk of the deletion fits.
\Cref{thm:chained-loo-concentration} instead controls anchored increments by
mixed stability and metric entropy, replacing absolute level control by one
sample-dependent shift common to all training-time--hyperparameter pairs.
For deletion-to-full transfer, \cref{thm:jackknife-cancellation} combines
cancellation of the average first deletion-chord derivative with a uniform
second-derivative bound to give an order-$(n-1)^{-2}$ cross-size error.

The output-space route expresses the response approximation through a
dynamic-kernel difference; for the two-layer verification, this route yields
constants uniform in width.
\Cref{thm:kernel-mismatch,thm:output-space-deletion}
express prediction and score errors through the difference between the exact
deletion and response dynamic kernels.  For a bounded smooth scalar-output
two-layer mean-field network with both layers trained and nonnegative weight
decay, \cref{thm:two-layer-mean-field} compares the exact and response
prediction paths on each common finite horizon and at each common finite
width.  It bounds the dynamic-kernel difference by $O(n^{-1})$ with constants
uniform in width.
The $O(n^{-1})$ kernel bound gives an $O(n^{-2})$ error between the Flow-ALO
and exact-LOO scores.
Combining that error with single-anchor bounded-difference concentration,
chained increment concentration, and deletion-to-full transfer gives the
finite-width-family risk-curve bound in \cref{cor:mean-field-width-oracle}.

\subsection{Problem formulation}

Let $(\Omega,\mathcal F,\Pp)$ be a probability space,
$(\mathsf Z,\mathcal Z)$ a measurable observation space, and $P_0$ a
probability measure on $(\mathsf Z,\mathcal Z)$.  Let
$(Z_i)_{i\geq1}$ be independent and identically distributed measurable maps
from $(\Omega,\mathcal F)$ to $(\mathsf Z,\mathcal Z)$ with common law $P_0$,
and, for every $n\geq2$, put $[n]:=\{1,\ldots,n\}$ and let
$S_n:\Omega\to\mathsf Z^n$, $S_n=(Z_1,\ldots,Z_n)$, be the
$\mathcal F/\mathcal Z^{\otimes n}$-measurable random sample.  At fixed $n$,
write $S:=S_n$.  A \emph{predetermined} object is deterministic, independent
of the training sample and auxiliary randomness, and may depend on $n$ unless
stated otherwise.
Let $\Theta$ be a finite-dimensional real Hilbert space, let
$d_\lambda\in\{1,2,\ldots\}$ be deterministic and fixed independently of
$n$, let $\Lambda\subset\R^{d_\lambda}$ be a nonempty compact set of
within-model hyperparameter vectors, with scalar weight decay as the
$d_\lambda=1$ specialization, and fix a deterministic
$\theta_0\in\Theta$.  For every
$(\lambda,z)\in\Lambda\times\mathsf Z$, let
$\ell^{\rm tr}_{\lambda,z},\ell^{\rm ev}_{\lambda,z}:\Theta\to\R$ be the
training and evaluation losses, respectively; every loss integral in the manuscript is
assumed measurable and finite.  Fix a deterministic $T\in[0,\infty)$ and a
metric $d_{\cH}:([0,\infty)\times\Lambda)^2\to[0,\infty)$, and let
$\cH\subset[0,T]\times\Lambda$ be a predetermined nonempty
$d_{\cH}$-compact set, equipped with the $d_{\cH}$-Borel $\sigma$-field
$\mathcal B(\cH)$.  Write $\delta_z$ for
the Dirac probability measure at
$z\in\mathsf Z$.  The sample-dependent empirical law on $(\mathsf Z,\mathcal Z)$
is $P_n=n^{-1}\sum_{j=1}^n\delta_{Z_j}$.  For $\lambda\in\Lambda$, define
the full-sample empirical training objective $F_{P_n,\lambda}:\Theta\to\R$;
whenever the gradient flow of $F_{P_n,\lambda}$ from $\theta_0$ exists on
$[0,T]$, denote the gradient flow by $\theta^\lambda:[0,T]\to\Theta$:
\begin{equation}\label{eq:intro-full-flow}
 F_{P_n,\lambda}(\theta)
 =\frac1n\sum_{j=1}^n
    \ell^{\rm tr}_{\lambda,Z_j}(\theta),
 \qquad
 \dot\theta_t^\lambda=-\nabla F_{P_n,\lambda}(\theta_t^\lambda),
 \quad \theta_0^\lambda=\theta_0.
\end{equation}
For $i\in[n]$, let $S_{-i}:\Omega\to\mathsf Z^{n-1}$ be the
$\mathcal F/\mathcal Z^{\otimes(n-1)}$-measurable deleted sample
$(Z_j)_{j\in[n]\setminus\{i\}}$, and let
$P_{-i}=(n-1)^{-1}\sum_{j\in[n]\setminus\{i\}}\delta_{Z_j}$ be the
sample-dependent empirical law of $S_{-i}$ on $(\mathsf Z,\mathcal Z)$.
For $\lambda\in\Lambda$, define the
deleted-sample empirical training objective $F_{P_{-i},\lambda}:\Theta\to\R$ by
$F_{P_{-i},\lambda}(\theta)
=(n-1)^{-1}\sum_{j\in[n]\setminus\{i\}}
\ell^{\rm tr}_{\lambda,Z_j}(\theta)$ for
$\theta\in\Theta$.  Whenever the gradient flow of $F_{P_{-i},\lambda}$ from
$\theta_0$ exists on $[0,T]$, denote the gradient flow by
$\theta_{-i}^\lambda:[0,T]\to\Theta$.

Whenever all full and deleted paths required on $\cH$ exist on $[0,T]$,
define the sample-dependent map
$(\Score_n^{\LOO},\Risk_S^-,\Risk_S):\cH\to\R^3$ by
\begin{align}
 \Score_n^{\LOO}(t,\lambda)
 &=\frac1n\sum_{i=1}^n
   \ell^{\rm ev}_{\lambda,Z_i}(\theta_{-i,t}^\lambda),\label{eq:loo-score}\\
 \Risk_S^-(t,\lambda)
 &=\frac1n\sum_{i=1}^n\int
   \ell^{\rm ev}_{\lambda,z}(\theta_{-i,t}^\lambda)\,\dd P_0(z),\label{eq:minus-risk}\\
 \Risk_S(t,\lambda)
 &=\int\ell^{\rm ev}_{\lambda,z}(\theta_t^\lambda)\,\dd P_0(z).
 \label{eq:conditional-risk}
\end{align}
$\Score_n^{\LOO}$ evaluates observation $Z_i$ under $\theta_{-i}^\lambda$;
$\Risk_S^-$ and $\Risk_S$ instead integrate the deletion and full fits
against $P_0$.  The subscript $S$ records the realized fitting sample, held
fixed in \eqref{eq:minus-risk} and \eqref{eq:conditional-risk}; throughout,
$\Risk_S$ is called the
\emph{sample-conditional risk of the realized fitted path}, or
sample-conditional risk for short.

Flow-ALO replaces each deleted path $\theta_{-i}^\lambda$ by
$\theta^\lambda+d_i^\lambda$, where $d_i^\lambda$ is a linear response
driven along the full path, and averages the omitted losses at those response
paths.  Sections~\ref{sec:setup}--\ref{sec:flow-alo} give the forcing,
response ODE, and score definitions once.

\begin{figure}[t]
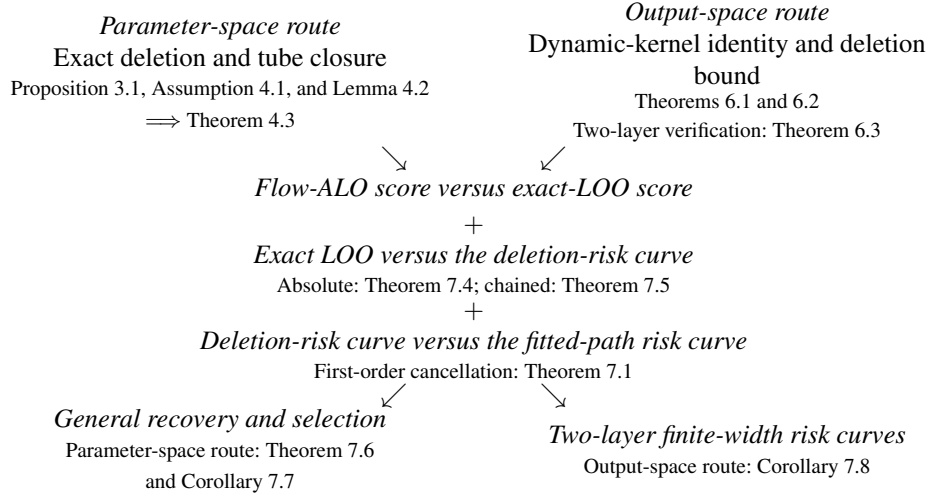

\centering
\setlength{\tabcolsep}{3pt}
\renewcommand{\arraystretch}{1.05}
\small
\begin{tabular}{c@{\hspace{.05\linewidth}}c}
\parbox{.42\linewidth}{\centering
  \textit{Parameter-space route}\\
  Exact deletion and tube closure\\[-1pt]
  {\scriptsize
  \cref{prop:exact-deletion,ass:tube,lem:full-path-tube}
  $\Longrightarrow$ \cref{thm:finite-horizon}}}
&
\parbox{.42\linewidth}{\centering
  \textit{Output-space route}\\
  Dynamic-kernel identity and deletion bound\\[-1pt]
  {\scriptsize
  \cref{thm:kernel-mismatch,thm:output-space-deletion}\\[-1pt]
  Two-layer verification: \cref{thm:two-layer-mean-field}}}\\[2pt]
\multicolumn{2}{c}{$\searrow\hspace{5em}\swarrow$}\\[-1pt]
\multicolumn{2}{c}{\textit{Flow-ALO score versus exact-LOO score}}\\[1pt]
\multicolumn{2}{c}{$+$}\\[1pt]
\multicolumn{2}{c}{\parbox{.66\linewidth}{\centering
  \textit{Exact LOO versus the deletion-risk curve}\\[-1pt]
  {\scriptsize Absolute: \cref{thm:loo-concentration};
  chained: \cref{thm:chained-loo-concentration}}}}\\[1pt]
\multicolumn{2}{c}{$+$}\\[1pt]
\multicolumn{2}{c}{\parbox{.66\linewidth}{\centering
  \textit{Deletion-risk curve versus the fitted-path risk curve}\\[-1pt]
  {\scriptsize First-order cancellation:
  \cref{thm:jackknife-cancellation}}}}\\[2pt]
\multicolumn{2}{c}{$\swarrow\hspace{5em}\searrow$}\\[-1pt]
\parbox{.42\linewidth}{\centering
  \textit{General recovery and selection}\\[-1pt]
  {\scriptsize Parameter-space route:
  \cref{thm:oracle,cor:regularity-event-oracle}}}
&
\parbox{.42\linewidth}{\centering
  \textit{Two-layer finite-width risk curves}\\[-1pt]
  {\scriptsize Output-space route:
  \cref{cor:mean-field-width-oracle}}}
\end{tabular}
\caption{Theorem dependencies for the central continuous-time results.
\Cref{thm:oracle,cor:regularity-event-oracle} use the parameter-space response
and either exact-LOO concentration route.  \Cref{cor:mean-field-width-oracle}
uses the two-layer output-space response and chained exact-LOO concentration.
\Cref{thm:jackknife-cancellation} supplies deletion-to-full transfer in
\mbox{\cref{thm:oracle,cor:regularity-event-oracle,cor:mean-field-width-oracle}}.}
\label{fig:theorem-dependency}
\end{figure}

\subsection{Related work}\label{sec:related}

Static endpoint ALO uses local corrections at a fitted regularized estimator
for risk assessment and tuning
\cite{wang2018approximate,rad2020scalable}.  Wilson, Kasy and Mackey
\cite[Theorem~2]{wilson2020approximate} give a deterministic $O(n^{-2})$
approximate-versus-exact-CV assessment bound, uniform over $\lambda$ at
regularized-ERM endpoints under the
curvature, derivative-moment, and Hessian-smoothness conditions stated by
Wilson--Kasy--Mackey.  The Wilson--Kasy--Mackey comparisons concern fitted
endpoints and exact CV rather than population risk along a training path, the
target studied here.

Finite data-weight reweighting is treated by first- and higher-order
infinitesimal-jackknife expansions
\cite{giordano2019swiss,giordano2019higher}.  Training-path influence methods
propagate data sensitivity through paths or iterate sequences
\cite{hara2019data,schioppa2023perspectives,nickl2023memory}, and Wang et al.
\cite[Sections~2--3]{wang2025temporal} formulate iterate-sequence-specific LOO
influence and approximate terminal influence for ordered SGD.  Litman and Guo
\cite[Theorem~G.10 and Corollary~G.13]{litman2026theory} derive an
infinitesimal reweighting derivative and a finite-perturbation Taylor
remainder.  The reweighting and path-influence results provide sensitivity
approximations; \cref{thm:finite-horizon} instead controls the normalized
delete-one nonlinear path and held-out score to second order, uniformly on a
predetermined compact time--hyperparameter set.

Iterative approximate CV propagates delete-one surrogates along optimization
iterates.  For full-batch gradient descent at a fixed objective, Luo, Ren and
Barber \cite[Assumptions~4.1--4.3 and~4.5 and Theorems~4.4
and~4.6]{luo2023iterative} prove iteration-uniform $O(n^{-2})$
IACV-to-exact-LOO parameter and CV-score errors.  The full-batch-GD result of
Luo--Ren--Barber does not assume global convexity, but requires every
deleted-objective Hessian to be uniformly positive and bounded along all
full-data iterates, together with global Hessian Lipschitzness, pathwise
gradient bounds, common initialization, and step- and sample-size restrictions
\cite[Assumptions~4.1--4.3 and~4.5 and Theorem~4.4]{luo2023iterative}.
Luo--Ren--Barber
\cite[Equation~(16), Theorem~4.6, and the discussion after
Theorem~4.6]{luo2023iterative} control the exact iterative LOO score and discuss
using the exact iterative LOO score for early stopping, but give neither
population-risk concentration nor an oracle for the selected iterate.

Direct same-training-data risk work studies optimization iterates.  In
proportional least squares, Patil, Wu and Tibshirani
\cite[Equation~(12) and Theorems~2--3]{patil2024failures} prove uniform
convergence of exact LOOCV along a growing gradient-descent path to
sample-conditional prediction risk and asymptotically risk-optimal early
stopping.  Bellec and Tan
\cite[Theorem~2.1 and Corollary~2.3]{bellec2024uncertainty} construct
root-$n$-accurate conditional-risk estimators and a stopping oracle for a fixed
number of broad updates in proportional Gaussian linear models.  At fixed
iterations, Tan and Bellec
\cite[Theorems~3.6--3.7]{tan2024estimating} recover robust-regression
prediction error up to a tuning-independent noise term.  Han and Xu
\cite[Definition~2.4 and Theorem~3.3]{han2026gradient} give a fully data-driven
estimator of conditional generalization error defined using a fresh covariate
and noise resampled from the realized empirical noises.  Han and Imaizumi
\cite[Theorems~4.2--4.3]{han2025precise} give state-evolution-based same-data
generalization estimates for proportional finite-width networks.

Adjacent continuous-time work addresses different inferential targets.  For
squared-loss gradient flow governed by a fixed positive-semidefinite prediction
operator, Yao et al.
\cite[Theorem~2 and Corollary~1]{yao2026gradientflow} prove asymptotic
risk-ratio optimality of a training-sample REML stopping rule.  Li and
Giessing \cite[Theorems~2--3]{li2026statistical} establish an infinite-horizon
functional central limit theorem and an algorithm-aware covariance estimator
for a parameter path relative to a population flow.  For kernel regression,
Cheng, Chen and Lin \cite[Theorems~4 and~11]{cheng2026optimal} give sup-norm
rates for continuous kernel gradient flow and simultaneous bands over
covariates at scaled times.

Conditioning on the fitted sample changes the target of cross-validation.
Litman and Guo \cite[Lemma~6.1 and Theorem~F.6]{litman2026theory} give
expectation-level cross-validation and generalization-gap identities.  In
Gaussian ordinary least squares, Bates, Hastie and Tibshirani
\cite[Theorems~1--2 and Corollary~2]{bates2024crossvalidation} show that LOOCV
can be asymptotically uncorrelated with fluctuations of realized conditional
error.

Algorithmic stability relates one-sample perturbations to generalization
\cite{bousquet2002stability,hardt2016train}.  For continuous-time
sum-separable optimizers with bounded updates and Lipschitz loss that contract
in a uniformly conditioned Riemannian metric, Kozachkov, Wensing and Slotine
\cite[Section~2.1, Theorem~3, and Remark~1]{kozachkov2023generalization}
obtain an exponentially decaying initialization transient and an $O(n^{-1})$
replace-one stability floor; common initialization removes the transient and
gives the analogous horizon-independent LOO stability.

\pagebreak
For iid Euclidean
data, Avelin and Viitasaari
\cite[Definitions~2.3 and~2.6, Assumption~3.1, Theorem~3.6, and
Remark~3.4]{avelin2023concentration} obtain two-sided pointwise concentration
of the LOO score about the realized full-sample learner's conditional risk
under the log-Sobolev, evaluated-loss dataset-gradient, error-stability, and
expected-loss-growth conditions of Avelin--Viitasaari.
Their model-selection discussion notes that the same bounds can be applied to
the loss difference of two estimators for model selection
\cite[Section~7]{avelin2023concentration}.  Their full-sample centering,
however, assumes an expected deletion-to-full error-stability bound, and the
cited analysis controls exact LOO rather than approximating it or proving a
realized-sample deletion-to-full comparison
\cite[Definition~2.6, Assumption~3.1, Theorem~3.6, and
Remark~3.4]{avelin2023concentration}.

Neural-network verification draws on several dynamic regimes.  Jacot, Gabriel
and Hongler \cite[Theorem~2]{jacot2018neural} prove finite-horizon uniform NTK
convergence to a constant infinite-width kernel.  Chen et al.
\cite[Definition~4, Theorems~2--3, and Corollary~1]{chen2023losspath}
introduce the loss path kernel, represent terminal loss as a general kernel
machine, and derive a Rademacher generalization-gap bound.  Chen et al.
\cite[Assumption~4.2, Lemmas~4.3--4.5, and
Theorem~5.2]{chen2025losspath} give a stability-localized, one-sided
high-probability terminal-time bound whose leading term is computed from the
training-set loss path kernel. Sun and Valaee
\cite[Propositions~3.1--3.3 and Equations~(5)--(13)]{sun2026kernel} derive a
dual NTK-coordinate endpoint-removal influence approximation for a
ridge-regularized linearized model.  Mean-field analyses give fixed-horizon
dimension-free approximation \cite{mei2019meanfield}, evolving two-time
kernels \cite{bordelon2022selfconsistent}, finite-width kernel and prediction
fluctuations \cite{bordelon2023finitewidth}, and finite-to-mean-field control
over structured polynomial horizons
\cite[Theorems~7 and~9]{glasgow2025meanfield}.  Shallow-network stability is
studied by Richards and Kuzborskij and Lei et al.
\cite{richards2021stability,lei2022stability}.  The model verification in
\cref{thm:two-layer-mean-field} instead trains both layers and controls the
same-width delete-one dynamic-kernel mismatch with constants uniform in width.
 \section{Setup and targets}\label{sec:setup}

Section~\ref{sec:introduction} defines the probability space, observation
space, parameter space, hyperparameter set, and loss maps.  An
unsubscripted $\|\cdot\|$ denotes
the norm of the ambient Hilbert or Euclidean space.  For a linear map $A$,
$\|A\|_{\op}=\sup_{\|v\|\leq1}\|Av\|$; for a multilinear map, the supremum
is over one vector of norm at most one in each argument.
$I_E:E\to E$ denotes the identity on a finite-dimensional real Hilbert space
$E$; when a typed operator expression fixes $E$, we abbreviate $I_E$ by $I$.
For each finite signed measure
$P$ on $(\mathsf Z,\mathcal Z)$ with $P(\mathsf Z)=1$ and each
$\lambda\in\Lambda$, define the real-valued objective $F_{P,\lambda}$ on the
set of $\theta\in\Theta$ for which
$z\mapsto\ell^{\rm tr}_{\lambda,z}(\theta)$ is $\mathcal Z$-measurable and
$|P|$-integrable, where $|P|$ is the total-variation measure of $P$, by
\begin{equation}\label{eq:objective-vector-field}
 F_{P,\lambda}(\theta)
 =\int\ell^{\rm tr}_{\lambda,z}(\theta)\,\dd P(z).
\end{equation}
More generally, for any $|P|$-integrable map $g$ from $\mathsf Z$ to a
finite-dimensional real vector space, write
$Pg:=\int g(z)\,\dd P(z)$.  For an integrable family $z\mapsto g_z$, the
notation $Pg_Z$ means $\int g_z\,\dd P(z)$; the capital $Z$ marks the
integration coordinate.
Every population-risk evaluation-loss integrand is assumed
$\mathcal Z$-measurable and $P_0$-integrable.
Equip $\Theta$ and $\Lambda$ with the respective Borel $\sigma$-fields.  The phrase
\emph{one common evaluation rule} means that there are a measurable
output space $(\mathsf V,\mathcal B_{\mathsf V})$, a jointly measurable map
$\mathfrak f:\Lambda\times\Theta\times\mathsf Z\to\mathsf V$, and one fixed measurable
map $L_{\rm ev}:\mathsf V\times\mathsf Z\to\R$ such that
$\ell^{\rm ev}_{\lambda,z}(\theta)=L_{\rm ev}(\mathfrak f(\lambda,\theta,z),z)$.
In the supervised special case, let
$(\mathcal X,\mathcal B_{\mathcal X})$ and
$(\mathcal Y,\mathcal B_{\mathcal Y})$ be measurable spaces, take
$(\mathsf Z,\mathcal Z)=(\mathcal X\times\mathcal Y,
\mathcal B_{\mathcal X}\otimes\mathcal B_{\mathcal Y})$, let
$f_{\lambda,\theta}:\mathcal X\to\mathsf V$ for
$(\lambda,\theta)\in\Lambda\times\Theta$, and let
$\varphi:\mathsf V\times\mathcal Y\to\R$ be measurable.  Then one common
evaluation rule means, for every
$(\lambda,\theta,x,y,v)\in
\Lambda\times\Theta\times\mathcal X\times\mathcal Y\times\mathsf V$,
\[
 \mathfrak f(\lambda,\theta,(x,y))=f_{\lambda,\theta}(x),
 \qquad L_{\rm ev}(v,(x,y))=\varphi(v,y).
\]
When $d_\lambda=1$ and $\Lambda\subset[0,\infty)$, the scalar
weight-decay setup is the special case
$\ell^{\rm tr}_{\lambda,z}=\ell_z+\lambda R$ and
$\ell^{\rm ev}_{\lambda,z}=\ell_z$, where
$R:\Theta\to\R$ and $\ell_z:\Theta\to\R$ for $z\in\mathsf Z$ do not depend
on $\lambda$.
For each mass-one finite signed measure $P$ admitted by
\eqref{eq:objective-vector-field} and $\lambda\in\Lambda$, let
$\theta^{P,\lambda}:[0,T]\to\Theta$ solve
\begin{equation}\label{eq:full-flow}
 \dot\theta_t^{P,\lambda}=-\nabla F_{P,\lambda}(\theta_t^{P,\lambda}),
 \qquad \theta_0^{P,\lambda}=\theta_0\in\Theta.
\end{equation}
All deletions and values of $\lambda$ share the deterministic $\theta_0$ from
\eqref{eq:intro-full-flow}.
\Cref{sec:introduction} defines the predetermined horizon $T$, metric
$d_{\cH}$, and candidate set $\cH$.  For $\lambda\in\Lambda$, write the whole
path as $\theta^\lambda:=\theta^{P_n,\lambda}$.

The normalized measure $P_{-i}$ and the directions
$P_n-\delta_{Z_i}$ satisfy
\begin{equation}\label{eq:measure-deletion}
 P_{-i}=P_n+\frac{P_n-\delta_{Z_i}}{n-1},
 \qquad \frac1n\sum_{i=1}^n(P_n-\delta_{Z_i})=0.
\end{equation}
For $i\in[n]$, $t\in[0,T]$, and
$\lambda\in\Lambda$, whenever the required first and second parameter
derivatives exist, define along the full trajectory the centered gradient
$c_{i,t}^\lambda\in\Theta$ and centered Hessian
$C_{i,t}^\lambda:\Theta\to\Theta$ by
\begin{align}
 c_{i,t}^\lambda
 &=\nabla\ell^{\rm tr}_{\lambda,Z_i}(\theta_t^\lambda)
   -\frac1n\sum_{j=1}^n
      \nabla\ell^{\rm tr}_{\lambda,Z_j}(\theta_t^\lambda),
 \label{eq:centered-gradient}\\
 C_{i,t}^\lambda
 &=\nabla^2\ell^{\rm tr}_{\lambda,Z_i}(\theta_t^\lambda)
   -\frac1n\sum_{j=1}^n
      \nabla^2\ell^{\rm tr}_{\lambda,Z_j}(\theta_t^\lambda).
 \label{eq:centered-hessian}
\end{align}
For $i\in[n]$, $t\in[0,T]$, and $\lambda\in\Lambda$, the deletion identities
along the full path are
\begin{equation}\label{eq:deleted-gradient-hessian}
 \begin{split}
 \nabla F_{P_{-i},\lambda}(\theta_t^\lambda)
   &=\nabla F_{P_n,\lambda}(\theta_t^\lambda)-\frac{c_{i,t}^\lambda}{n-1},\\
 \nabla^2F_{P_{-i},\lambda}(\theta_t^\lambda)
   &=\nabla^2F_{P_n,\lambda}(\theta_t^\lambda)-\frac{C_{i,t}^\lambda}{n-1}.
 \end{split}
\end{equation}

The exact leave-one-out score, average deletion-learner population risk, and
full-sample population risk were defined in
\cref{eq:loo-score,eq:minus-risk,eq:conditional-risk}.
The training observations determine the approximate risk-curve family
$\widetilde\Score_n(t,\lambda)$ in \eqref{eq:flow-alo-score}; the primary
target is the uniform error of $\widetilde\Score_n(t,\lambda)$ as an estimator
of $\Risk_S(t,\lambda)$, either in absolute level or up to one common
sample-dependent shift.  Approximate minimization over $\cH$, as in
\cref{cor:regularity-event-oracle}, remains a consequence.

\begin{table}[t]
\caption{Assumption and notation glossary.  Every bound is required only on
the index set specified by the invoking result.}
\label{tab:assumption-glossary}
\centering
\scriptsize
\setlength{\tabcolsep}{3pt}
\renewcommand{\arraystretch}{1.12}
\begin{tabular}{@{}p{.29\linewidth}p{.64\linewidth}@{}}
\hline
Object or constant & Meaning and principal use \\
\hline
$P_0,P_n,P_{-i}$ & Population, empirical, and normalized delete-one laws;
risk target and deletion chords. \\
$\theta^\lambda,\theta_{-i}^\lambda,d_i^\lambda$ & Full flow, deleted flow,
and linear deletion response; \cref{prop:exact-deletion,thm:finite-horizon}. \\
$\widetilde\Score_n,\Score_n^{\LOO},\Risk_S^-,\Risk_S$ & Approximate score,
exact-LOO score, average deletion risk, and realized full-sample risk
in the three transfers of \cref{fig:theorem-dependency}. \\
$\|\cdot\|,\|\cdot\|_{\op},\|\cdot\|_{\TV},d_{\cH}$ & Parameter,
operator, signed-measure, and candidate-index geometries for dynamic bounds,
jackknife chords, and covering arguments. \\
$M,L_3,\kappa,G_{\rm out}$ & Centered-gradient bound, Hessian-Lipschitz
constant, one-sided-curvature parameter, and evaluation-gradient bound;
in \cref{ass:tube,thm:finite-horizon}. \\
$C_{PP}$ & Bound for the population-risk second derivative along each
deletion chord after division by $\|h_i\|_{\TV}^2$;
\cref{thm:jackknife-cancellation}. \\
$B_{n-1}^{\rm rng},\beta_{n-1},\omega_{n-1}$ & Deterministic loss-range
length, replace-one stability, and candidate modulus at learner size $n-1$
in \cref{thm:loo-concentration}. \\
$L_{n-1}^{\rm hyp},\beta_{n-1}^{\rm mix}$ & Candidate Lipschitz and mixed
replace-one stability constants in \cref{thm:chained-loo-concentration}. \\
$D_{\rm dyn}$ & Composed dynamic-kernel mismatch in
\cref{thm:kernel-mismatch,thm:output-space-deletion}. \\
\hline
\end{tabular}
\end{table}
 \section{Individual-deletion response and \FlowALO}\label{sec:flow-alo}

For $i\in[n]$, $\lambda\in\Lambda$, and any interval
$\mathfrak I\subseteq[0,T]$ on which the full and deleted flows exist, define
the exact deletion displacement $\Delta_i^\lambda:\mathfrak I\to\Theta$ by
$\Delta_{i,t}^\lambda=\theta_{-i,t}^\lambda-\theta_t^\lambda$ for
$t\in\mathfrak I$.
Integrating the deleted-objective Hessian along the segment from
$\theta_t^\lambda$ to $\theta_{-i,t}^\lambda$ yields the exact dynamics of
$\Delta_{i,t}^\lambda$.

\begin{proposition}[Exact deletion dynamics]\label[proposition]{prop:exact-deletion}
Fix $i\in[n]$, $\lambda\in\Lambda$, and $T'\in[0,T]$, and suppose both
flows exist on $[0,T']$.  Suppose $F_{P_{-i},\lambda}$ is twice continuously
differentiable on the line segments joining $\theta_t^\lambda$ to
$\theta_{-i,t}^\lambda$ for every $t\in[0,T']$.  Then, for $0\leq t\leq T'$,
\begin{equation}\label{eq:exact-deletion-ode}
 \dot\Delta_{i,t}^\lambda
 =-\overline B_{i,t}^\lambda\Delta_{i,t}^\lambda
   +\frac{c_{i,t}^\lambda}{n-1},
 \qquad \Delta_{i,0}^\lambda=0,
\end{equation}
where the secant operator $\overline B_{i,t}^\lambda:\Theta\to\Theta$ is
\begin{equation}\label{eq:secant-hessian}
 \overline B_{i,t}^\lambda
 =\int_0^1
   \nabla^2F_{P_{-i},\lambda}
      (\theta_t^\lambda+u\Delta_{i,t}^\lambda)\,\dd u.
\end{equation}
\end{proposition}
Supplement Section~S.A proves \cref{prop:exact-deletion}.
\FlowALO replaces the unknown
secant operator in
\cref{eq:exact-deletion-ode} by the deleted Hessian evaluated on the observed
full-sample path,
\begin{equation}\label{eq:resummed-B}
 B_{i,t}^\lambda=\nabla^2F_{P_{-i},\lambda}(\theta_t^\lambda)
        =\nabla^2F_{P_n,\lambda}(\theta_t^\lambda)
         -\frac{C_{i,t}^\lambda}{n-1},
\end{equation}
and define the Flow-ALO response by
\begin{equation}\label{eq:flow-alo-response}
 \dot d_{i,t}^\lambda=-B_{i,t}^\lambda d_{i,t}^\lambda
   +\frac{c_{i,t}^\lambda}{n-1},
 \qquad d_{i,0}^\lambda=0.
\end{equation}
With simultaneous staging on the full-sample path, a forward-Euler
discretization of \eqref{eq:flow-alo-response} matches the full-batch IACV
recursion of Luo, Ren and Barber
\cite[Section~2.3, Equation~(12)]{luo2023iterative}, after accounting for the
present objective normalization.
For $i\in[n]$, $\lambda\in\Lambda$, and $0\leq s\leq t\leq T$, let
$U_i^\lambda(t,s):\Theta\to\Theta$ be the fundamental solution defined by
\begin{equation}\label{eq:flow-alo-propagator}
 \partial_tU_i^\lambda(t,s)=-B_{i,t}^\lambda U_i^\lambda(t,s),
 \qquad U_i^\lambda(s,s)=I.
\end{equation}
The response in \eqref{eq:flow-alo-response} therefore has the Duhamel
representation
\begin{equation}\label{eq:flow-alo-duhamel}
 d_{i,t}^\lambda
 =\frac1{n-1}\int_0^tU_i^\lambda(t,s)c_{i,s}^\lambda\,\dd s.
\end{equation}
The correction $-C_{i,t}^\lambda/(n-1)$ in \cref{eq:resummed-B} acts at
every later propagation time; we call the accumulated centered-Hessian
propagation \emph{resummation}.

The term resummation distinguishes $d_i^\lambda$ from a literal derivative at the
full empirical law.  Put
$H_t^\lambda=\nabla^2F_{P_n,\lambda}(\theta_t^\lambda)$.  Suppose the
$C^1([0,T];\Theta)$-valued map
$\varepsilon\mapsto
\theta^{P_n+\varepsilon(P_n-\delta_{Z_i}),\lambda}$ is defined near zero and
differentiable at zero.  Suppose also that
$(\varepsilon,x)\mapsto
\nabla F_{P_n+\varepsilon(P_n-\delta_{Z_i}),\lambda}(x)$ is continuously
differentiable on a neighborhood of
$\{0\}\times\{\theta_t^\lambda:0\leq t\leq T\}$.  Define the scaled derivative
of the empirical-weight path
\[
 g_{i,t}^\lambda
 =\frac1{n-1}\left.\frac{\partial}{\partial\varepsilon}
 \theta_t^{P_n+\varepsilon(P_n-\delta_{Z_i}),\lambda}
 \right|_{\varepsilon=0}.
\]
Differentiating the flow gives the conventional full-Hessian influence ODE
\begin{equation}\label{eq:full-hessian-influence-response}
 \dot g_{i,t}^\lambda=-H_t^\lambda g_{i,t}^\lambda
    +\frac{c_{i,t}^\lambda}{n-1},
 \qquad g_{i,0}^\lambda=0.
\end{equation}
Hence $g_i^\lambda$, not $d_i^\lambda$, is the scaled data-weight derivative
at $P_n$.  Flow-ALO instead incorporates the centered-Hessian correction in
\eqref{eq:resummed-B} throughout propagation.

The resulting same-sample \FlowALO score is
\begin{equation}\label{eq:flow-alo-score}
 \widetilde\Score_n(t,\lambda)
   =\frac1n\sum_{i=1}^n
      \ell^{\rm ev}_{\lambda,Z_i}
        (\theta_t^\lambda+d_{i,t}^\lambda).
\end{equation}
\Cref{eq:flow-alo-score} evaluates each omitted loss at the approximate
deleted parameter $\theta_t^\lambda+d_{i,t}^\lambda$ and, unlike a
first-order loss correction, retains quadratic and higher-order effects when
the training-point evaluation gradient vanishes at interpolation.
 \section{Deterministic approximation theory}\label{sec:deterministic}

\Cref{eq:flow-alo-response} approximates the exact deletion displacement in
\cref{eq:exact-deletion-ode} by replacing the secant Hessian with
$B_{i,t}^\lambda=\nabla^2F_{P_{-i},\lambda}(\theta_t^\lambda)$ on the full
path; \cref{thm:finite-horizon} bounds the error uniformly on $[0,T]$.
Pathwise constants in Section~\ref{sec:deterministic} refer to the realized
sample $S$ unless declared deterministic.
For $x\in\Theta$ and $r>0$, write
$B(x,r)=\{y\in\Theta:\|y-x\|<r\}$ for the open radius-$r$ ball.

For $\kappa\geq0$, define the nonnegative propagation-gain functions
$\phi_\kappa,\mathcal J_\kappa:[0,\infty)\to[0,\infty)$ by
\begin{equation}\label{eq:phi-J-def}
 \phi_\kappa(t)=
 \begin{cases}(e^{\kappa t}-1)/\kappa,&\kappa>0,\\t,&\kappa=0,\end{cases}
 \qquad
 \mathcal J_\kappa(t)=
 \begin{cases}
 \{e^{2\kappa t}-1-2\kappa t e^{\kappa t}\}/\kappa^3,&\kappa>0,\\
 t^3/3,&\kappa=0.
 \end{cases}
\end{equation}

\begin{assumption}[Full-path tube regularity]\label[assumption]{ass:tube}
For every $\lambda\in\Lambda$, the full path $\theta_t^\lambda$ exists on
$[0,T]$.  There are common constants
$r_{\rm tube}>0$ and $\kappa,L_3,M\in[0,\infty)$ such that, with the open full-path tube
\begin{equation}\label{eq:full-path-R-tube}
 \mathfrak T_{r_{\rm tube}}^\lambda\subset\Theta,\qquad
 \mathfrak T_{r_{\rm tube}}^\lambda
 =\bigcup_{0\leq t\leq T}B(\theta_t^\lambda,r_{\rm tube}),
\end{equation}
$F_{P_{-i},\lambda}$ is twice continuously differentiable on
$\mathfrak T_{r_{\rm tube}}^\lambda$ for every $i\in[n]$ and
$\lambda\in\Lambda$.  Uniformly over $i\in[n]$ and $\lambda\in\Lambda$,
\begin{align}
 \nabla^2F_{P_{-i},\lambda}(x)&\succeq-\kappa I,
 &&x\in\mathfrak T_{r_{\rm tube}}^\lambda,
 \label{eq:one-sided-curvature}\\
 \|\nabla^2F_{P_{-i},\lambda}(x)
       -\nabla^2F_{P_{-i},\lambda}(y)\|_{\op}
     &\leq L_3\|x-y\|,
 &&x,y\in B(\theta_t^\lambda,r_{\rm tube}),\quad 0\leq t\leq T.
 \label{eq:hessian-lipschitz}
\end{align}
Moreover,
\begin{equation}\label{eq:full-path-tube-closure-condition}
 \sup_{\substack{1\leq i\leq n,\ 0\leq t\leq T\\\lambda\in\Lambda}}
 \|c_{i,t}^\lambda\|\leq M,
 \qquad
 (n-1)^{-1}M\phi_\kappa(T)<r_{\rm tube}.
\end{equation}
\end{assumption}

A uniform Hessian bound on a common neighborhood of the tubes gives
\eqref{eq:one-sided-curvature}, a uniform third-derivative bound gives
\eqref{eq:hessian-lipschitz}, and the strict margin in
\eqref{eq:full-path-tube-closure-condition} keeps the deleted path and response
inside the tube.

\begin{lemma}[Full-path tube closure]\label[lemma]{lem:full-path-tube}
Under \cref{ass:tube}, for every $i\in[n]$ and $\lambda\in\Lambda$ the deleted flow
$\theta_{-i,t}^\lambda$ and the response $d_{i,t}^\lambda$ exist uniquely on
$[0,T]$.
Moreover, for $0\leq t\leq T$,
\begin{equation}\label{eq:full-path-tube-closure-bound}
 \max\{\|\theta_{-i,t}^\lambda-\theta_t^\lambda\|,\|d_{i,t}^\lambda\|\}
 \leq(n-1)^{-1}M\phi_\kappa(t)<r_{\rm tube}.
\end{equation}
Consequently, at each time $t$, the line segments joining any two of
$\theta_t^\lambda$, $\theta_{-i,t}^\lambda$, and
$\theta_t^\lambda+d_{i,t}^\lambda$ are contained in
$B(\theta_t^\lambda,r_{\rm tube})$.
\end{lemma}
Supplement Section~S.A proves \cref{lem:full-path-tube}.
\Cref{lem:full-path-tube} localizes
the exact and response curves in $B(\theta_t^\lambda,r_{\rm tube})$, enabling
full-horizon comparison of the two curves.

\begin{theorem}[Uniform finite-horizon error]\label{thm:finite-horizon}
Under \cref{ass:tube},
\begin{equation}\label{eq:deletion-response-size}
 \sup_{\substack{1\leq i\leq n\\\lambda\in\Lambda}}
 \max\{\|\Delta_{i,t}^\lambda\|,\|d_{i,t}^\lambda\|\}
 \leq (n-1)^{-1}M\phi_\kappa(t),\qquad 0\leq t\leq T,
\end{equation}
and
\begin{equation}\label{eq:parameter-response-error}
 \sup_{\substack{1\leq i\leq n,\ 0\leq t\leq T\\\lambda\in\Lambda}}
 \|\Delta_{i,t}^\lambda-d_{i,t}^\lambda\|
 \leq \frac{L_3M^2}{2}(n-1)^{-2}\mathcal J_\kappa(T).
\end{equation}
If, in addition, there is a finite $G_{\rm out}\geq0$ such that
\begin{equation}\label{eq:evaluation-gradient-bound}
 \sup_{\substack{1\leq i\leq n,\ 0\leq t\leq T,\ \lambda\in\Lambda\\
                   x\in B(\theta_t^\lambda,r_{\rm tube})}}
 \|\nabla\ell^{\rm ev}_{\lambda,Z_i}(x)\|\leq G_{\rm out},
\end{equation}
then
\begin{equation}\label{eq:score-approximation-bound}
 \sup_{(t,\lambda)\in\cH}
 |\widetilde\Score_n(t,\lambda)-\Score_n^{\LOO}(t,\lambda)|
 \leq\frac{G_{\rm out}L_3M^2}{2}(n-1)^{-2}\mathcal J_\kappa(T).
\end{equation}
\end{theorem}
Supplement Section~S.A proves \cref{thm:finite-horizon}.
For full-batch gradient descent, Luo, Ren and Barber
\cite[Assumptions~4.1--4.3 and~4.5 and Theorems~4.4 and~4.6]{luo2023iterative}
establish all-iteration $O(n^{-2})$ IACV parameter and CV-score accuracy under
common initialization of the full, deleted, and IACV iterates.  Their
conditions also include two-sided Hessian-eigenvalue bounds,
individual-gradient and Hessian-Lipschitz bounds, a held-out-loss condition,
and step- and sample-size restrictions.  \Cref{thm:finite-horizon} gives the continuous-time finite-horizon
counterpart under one-sided tube curvature and makes the dependence on
$\mathcal J_\kappa(T)$ explicit.
\Cref{thm:finite-horizon} allows $\nabla^2F_{P_{-i},\lambda}\succeq-\kappa I$ on a fixed
continuous-time tube, with horizon dependence in $\mathcal J_\kappa(T)$.
Population-risk comparison also uses
\cref{thm:loo-concentration,thm:chained-loo-concentration,thm:jackknife-cancellation}.
The parameter and score radii in \eqref{eq:parameter-response-error} and
\eqref{eq:score-approximation-bound} are
dimension-independent whenever $L_3,M,G_{\rm out}$, and
$\mathcal J_\kappa(T)$ are common scalar envelopes.

Under \cref{ass:tube} and the empirical-weight-path differentiability
conditions used to derive \eqref{eq:full-hessian-influence-response}, an extra envelope
$\sup_{1\leq i\leq n,\,0\leq t\leq T,\,\lambda\in\Lambda}
\|C_{i,t}^\lambda\|_{\op}\leq K_C<\infty$ yields the
conventional full-Hessian response comparison
\begin{equation}\label{eq:full-hessian-response-comparison}
 \begin{aligned}
 \sup_{1\leq i\leq n,\,\lambda\in\Lambda}\sup_{0\leq t\leq T}
 \|\Delta_{i,t}^\lambda-g_{i,t}^\lambda\|
 &\leq\frac1{(n-1)^2}\int_0^T
 e^{\{\kappa+K_C/(n-1)\}(T-s)}\\[-2pt]
 &\quad\times\left\{K_CM\phi_\kappa(s)
       +\frac{L_3M^2}{2}\phi_\kappa(s)^2\right\}\,\dd s.
 \end{aligned}
\end{equation}
Supplement Section~S.A derives \eqref{eq:full-hessian-response-comparison}.
The comparison radius in \eqref{eq:full-hessian-response-comparison} is
$O(n^{-2})$ along sequences for which
$T,\kappa,K_C,L_3$, and $M$ remain bounded.  Resummation removes the separate
$K_C$ premise; the resummed response is also the formulation used for the width-uniform
output-space verification in \cref{sec:defect}.

Affine deleted gradients eliminate the secant-Hessian remainder entirely.

\begin{corollary}[Quadratic exactness]\label[corollary]{cor:quadratic-exactness}
Fix $T'\in[0,T]$ and suppose every full and deleted flow indexed by
$i\in[n]$ and $\lambda\in\Lambda$ exists on $[0,T']$.  If
$\nabla F_{P_{-i},\lambda}$ is affine in $\theta$
for every $i\in[n]$ and $\lambda\in\Lambda$, then
$d_{i,t}^\lambda=\Delta_{i,t}^\lambda$ for every $i\in[n]$,
$\lambda\in\Lambda$, and $t\in[0,T']$.  Thus, for every
$(t,\lambda)\in\cH\cap([0,T']\times\Lambda)$,
\(
 \widetilde\Score_n(t,\lambda)=\Score_n^{\LOO}(t,\lambda)
\)
exactly.
\end{corollary}
Supplement Section~S.A proves \cref{cor:quadratic-exactness}.
The exactness is independent of propagator size.

\section{A long-time obstruction}\label{sec:long-time-obstruction}

We construct a one-dimensional pitchfork that exhibits logarithmic-time
instability under smoothness.

\begin{example}[Long-time amplification at a pitchfork]\label[example]{prop:pitchfork}
For each even $n$, specialize
$(\mathsf Z,\mathcal Z)=(\{-1,+1\},2^{\{-1,+1\}})$ and $\Theta=\R$, and
take the deterministic sample having $n/2$ observations
equal to $+1$ and $n/2$ equal to $-1$.  Fix an additive loss offset
$C\in\R$.  For $z\in\{-1,+1\}$ let
\begin{equation}\label{eq:pitchfork-loss}
 \ell_z(x)=\frac14(x^2-1)^2-zx+C,
 \qquad R(x)=\frac12x^2.
\end{equation}
Set $d_\lambda=1$, fix $\lambda\in(0,1)$, take
$\Lambda=\{\lambda\}$ and the common initialization $\theta_0=x_0=0$, specialize
$\ell^{\rm tr}_{\lambda,z}=\ell_z+\lambda R$ and
$\ell^{\rm ev}_{\lambda,z}=\ell_z$, and let $P_0$ be uniform on
$\{-1,+1\}$.  For every fixed $\vartheta\in(0,\sqrt{1-\lambda})$ and all
sufficiently large even $n$, define a logarithmic time and risk gap by
\begin{equation}\label{eq:pitchfork-risk-gap}
 \begin{gathered}
 T_{n,\vartheta}^\lambda
 =\frac{\log\{1+(n-1)(1-\lambda-\vartheta^2)\vartheta\}}
        {1-\lambda-\vartheta^2}
 =O_{\vartheta,\lambda}(\log n),\\
 \left|
 \frac1n\sum_{i=1}^n
    \int\ell_z(\theta_{-i,T_{n,\vartheta}^\lambda}^\lambda)\,\dd P_0(z)
 -\int\ell_z(\theta_{T_{n,\vartheta}^\lambda}^\lambda)\,\dd P_0(z)
 \right|\geq \frac{\vartheta^2}{4}.
 \end{gathered}
\end{equation}
At $t=0$, the absolute risk difference in
\eqref{eq:pitchfork-risk-gap} is zero.
\end{example}

Supplement Section~S.A verifies \cref{prop:pitchfork}.  Deleting a $+1$ or
$-1$ observation changes the
forcing from $0$ to $\mp(n-1)^{-1}$.  The full path stays at the unstable
equilibrium $0$, while each deleted path reaches distance at least
$\vartheta$ at $T_{n,\vartheta}^\lambda$.  Thus
\eqref{eq:pitchfork-risk-gap} gives an $O_{\vartheta,\lambda}(\log n)$ time
and a risk gap at least $\vartheta^2/4$, explaining the explicit horizon in
\cref{thm:finite-horizon}.
Equation~\eqref{eq:pitchfork-risk-gap} is a deterministic balanced-sample
statement; an iid fixed-confidence lower bound would require a separate
random-sample argument.
 \section{Exact output-space error representations}\label{sec:defect}

The parameter-space bound in \cref{thm:finite-horizon} can inherit
width-dependent derivative envelopes for networks.  This section instead expresses the
prediction error through an output-space dynamic-kernel mismatch and verifies
the resulting dimension-free criterion for a smooth two-layer mean-field
network with both layers included in the gradient flow.

Fix a baseline training probability measure $P$ on the underlying measurable
observation space $\mathsf Z$, a finite signed perturbation direction
$h$ on $(\mathsf Z,\mathcal Z)$ with $h(\mathsf Z)=0$, and a perturbation
scale $\eps\in\R$.
Write $|h|$ for the total-variation measure of $h$, define the perturbed
unit-mass finite signed measure
$P^\eps=P+\eps h$, and, for each $\lambda\in\Lambda$, define the
signed-measure forcing map
\[
 g_h^\lambda:\Theta\to\Theta,\qquad
 g_h^\lambda(\theta)
 =\int\nabla\ell^{\rm tr}_{\lambda,z}(\theta)\,\dd h(z).
\]
Throughout the shared output-space development, differentiation under the integral sign is valid for
the loss integrals with respect to $P$, $P^\eps$, and $h$.
In the shared output-space development, a Bochner integral is the norm limit of simple-function
integrals; for finite-dimensional $\Theta$, the Bochner integral is simply
coordinatewise Lebesgue integration
\cite[Chapter~II, Section~2]{diestel1977vector}.  Each result in the shared development
assumes that all required flows exist on $[0,T]$ and all required Bochner
integrals are well defined, except when flow existence is part of the
conclusion.
For each $\lambda\in\Lambda$, let
$\theta^{P,\lambda},\theta^{P^\eps,\lambda}:[0,T]\to\Theta$ be the flows
trained under $P$ and $P^\eps$, respectively, from the common deterministic
$\theta_0\in\Theta$ in \eqref{eq:intro-full-flow}, and define the
perturbation displacement path
\(
 \Delta^\lambda:[0,T]\to\Theta,\qquad \Delta_t^\lambda
 =\theta_t^{P^\eps,\lambda}-\theta_t^{P,\lambda}
\).
For $\lambda\in\Lambda$ and $t\in[0,T]$, define the
perturbed-objective Hessian operator
$B_t^{0,\lambda}:\Theta\to\Theta$ along the unperturbed flow by
\[
 B_t^{0,\lambda}
 =\nabla^2F_{P^\eps,\lambda}(\theta_t^{P,\lambda}).
\]
For each $\lambda\in\Lambda$, let the perturbation-response path
$d^\lambda:[0,T]\to\Theta$ solve
\begin{equation}
\label{eq:generic-resummed-response}
 \dot d_t^\lambda=-B_t^{0,\lambda}d_t^\lambda
 -\eps g_h^\lambda(\theta_t^{P,\lambda}),
 \qquad d_0^\lambda=0.
\end{equation}
For sample deletion and $i\in[n]$, take $P=P_n$, $h=P_n-\delta_{Z_i}$, and
$\eps=(n-1)^{-1}$.  Then
$(\theta_t^{P,\lambda},\theta_t^{P^\eps,\lambda},d_t^\lambda)
=(\theta_t^\lambda,\theta_{-i,t}^\lambda,d_{i,t}^\lambda)$ and
$-g_h^\lambda(\theta_t^\lambda)=c_{i,t}^\lambda$.

For the output-space results, let the output dimension
$k\in\{1,2,\ldots\}$, let
$(\mathcal X,\mathcal B_{\mathcal X})$ and
$(\mathcal Y,\mathcal B_{\mathcal Y})$ be measurable spaces, and specialize
$(\mathsf Z,\mathcal Z)$ to the product $\sigma$-field.  Write
$z=(x_z,y_z)\in\mathcal X\times\mathcal Y$.  Let
$f:\Theta\times\mathcal X\to\R^k$ and
$\varphi:\R^k\times\mathcal Y\to\R$, with
$f_\theta(x):=f(\theta,x)$.  Specialize the common output space and
evaluation rule by setting
$(\mathsf V,\mathcal B_{\mathsf V})=(\R^k,\mathcal B(\R^k))$,
$\mathfrak f(\lambda,\theta,(x,y))=f_\theta(x)$, and
$L_{\rm ev}(v,(x,y))=\varphi(v,y)$ for
$\lambda\in\Lambda$, $\theta\in\Theta$, $v\in\R^k$, and
$(x,y)\in\mathcal X\times\mathcal Y$.  Set
$\ell^{\rm ev}_{\lambda,z}(\theta)=:\ell_z(\theta)
=\varphi(f_\theta(x_z),y_z)$.  Suppose that, for every $\lambda\in\Lambda$,
there is a map $r_\lambda:\Theta\to\R$ such that
$\ell^{\rm tr}_{\lambda,z}=\ell_z+r_\lambda$ for every $z\in\mathsf Z$.
For $\lambda\in\Lambda$, $t\in[0,T]$, $x\in\mathcal X$, and
$z\in\mathsf Z$, whenever the prediction Jacobian and outer-loss gradient
exist, define
\[
 J_t^\lambda(x)=D_\theta f_{\theta_t^{P,\lambda}}(x):\Theta\to\R^k,
 \qquad
 a_t^\lambda(z)
 =\nabla_f\varphi(f_{\theta_t^{P,\lambda}}(x_z),y_z)\in\R^k.
\]
Here $\nabla_f$ is the Euclidean gradient in the first argument of $\varphi$,
and $^*$ denotes the Hilbert adjoint, so $J_t^\lambda(x)^*:\R^k\to\Theta$.
For $\lambda\in\Lambda$ and $0\leq s\leq t\leq T$, define the
segment-averaged Hessian $\overline B_t^\lambda:\Theta\to\Theta$, and let
$\overline U^\lambda(t,s):\Theta\to\Theta$ be the fundamental solution
generated by $-\overline B_t^\lambda$:
\begin{equation}\label{eq:generic-secant-propagator}
 \overline B_t^\lambda=\int_0^1
   \nabla^2F_{P^\eps,\lambda}
     (\theta_t^{P,\lambda}+u\Delta_t^\lambda)\,\dd u,
 \qquad
 \partial_t\overline U^\lambda(t,s)
 =-\overline B_t^\lambda\overline U^\lambda(t,s),
 \quad \overline U^\lambda(s,s)=I.
\end{equation}
Let $U^{0,\lambda}(t,s):\Theta\to\Theta$ be the unperturbed propagator,
generated by $-B_t^{0,\lambda}$ and initialized by
$U^{0,\lambda}(s,s)=I$.  For $\lambda\in\Lambda$, $t\in[0,T]$, and
$x\in\mathcal X$, define the chord Jacobians
$\overline J_t^{\Delta,\lambda}(x),
\overline J_t^{d,\lambda}(x):\Theta\to\R^k$ by
\begin{align}
 \overline J_t^{\Delta,\lambda}(x)
 &=\int_0^1D_\theta
   f_{\theta_t^{P,\lambda}+u\Delta_t^\lambda}(x)\,\dd u,
 &
 \overline J_t^{d,\lambda}(x)
 &=\int_0^1D_\theta
   f_{\theta_t^{P,\lambda}+ud_t^\lambda}(x)\,\dd u.
 \label{eq:chord-jacobians}
\end{align}

\begin{theorem}[Exact dynamic-kernel mismatch identity]
\label{thm:kernel-mismatch}
Fix $\lambda\in\Lambda$ and $x\in\mathcal X$.  Suppose
$F_{P^\eps,\lambda}$ is twice continuously
differentiable near every point
$\theta_s^{P,\lambda}+u\Delta_s^\lambda$ and
$f_\theta(x)$ is continuously differentiable near every point in
$\{\theta_t^{P,\lambda}+u\Delta_t^\lambda,
    \theta_t^{P,\lambda}+ud_t^\lambda\}$,
for $s,t\in[0,T]$ and $u\in[0,1]$.
Suppose, for $|h|$-almost every $z$ and every $s\in[0,T]$, that
$\theta\mapsto f_\theta(x_z)$ is differentiable at
$\theta_s^{P,\lambda}$ and $\varphi(\cdot,y_z)$ is differentiable at
$f_{\theta_s^{P,\lambda}}(x_z)$.
For $0\leq s\leq t\leq T$ and $|h|$-almost every $z\in\mathsf Z$, define the
chord and resummed dynamic-kernel operators
$\mathcal K^{\rm ch,\lambda}(t,s;x,z),
\mathcal K^{\rm res,\lambda}(t,s;x,z):
\R^k\to\R^k$ by
\begin{align}
 \mathcal K^{\rm ch,\lambda}(t,s;x,z)
 &=\overline J_t^{\Delta,\lambda}(x)
   \overline U^\lambda(t,s)J_s^\lambda(x_z)^*,
 \label{eq:chord-kernel}\\
 \mathcal K^{\rm res,\lambda}(t,s;x,z)
 &=\overline J_t^{d,\lambda}(x)
   U^{0,\lambda}(t,s)J_s^\lambda(x_z)^*.
 \label{eq:resummed-kernel}
\end{align}
Then
\begin{align}
 f_{\theta_t^{P^\eps,\lambda}}(x)-f_{\theta_t^{P,\lambda}}(x)
 &=-\eps\int_0^t\int
   \mathcal K^{\rm ch,\lambda}(t,s;x,z)a_s^\lambda(z)
   \,\dd h(z)\,\dd s,
 \label{eq:true-prediction-response}\\
 f_{\theta_t^{P,\lambda}+d_t^\lambda}(x)
  -f_{\theta_t^{P,\lambda}}(x)
 &=-\eps\int_0^t\int
   \mathcal K^{\rm res,\lambda}(t,s;x,z)a_s^\lambda(z)
   \,\dd h(z)\,\dd s.
 \label{eq:alo-prediction-response}\\
 f_{\theta_t^{P^\eps,\lambda}}(x)
  -f_{\theta_t^{P,\lambda}+d_t^\lambda}(x)
 &=-\eps\int_0^t\int
   \{\mathcal K^{\rm ch,\lambda}
      -\mathcal K^{\rm res,\lambda}\}(t,s;x,z)
   a_s^\lambda(z)\,\dd h(z)\,\dd s.
 \label{eq:kernel-mismatch-identity}
\end{align}
for every $t\in[0,T]$.
\end{theorem}
Supplement Section~S.B proves \cref{thm:kernel-mismatch}.
Sun and Valaee
\cite[Propositions~3.1--3.3 and Equations~(5)--(13)]{sun2026kernel} reduce an
endpoint-removal influence system from parameter to dataset--output dimension
for their ridge-regularized linearized model under the stated convexity and
rank conditions.  Here \eqref{eq:kernel-mismatch-identity} exactly
compares nonlinear deleted and response trajectories, and
\cref{thm:two-layer-mean-field} verifies the $O(n^{-1})$ finite-horizon kernel
mismatch in \eqref{eq:mean-field-kernel-mismatch}.

Fix a realized sample.  For deletion $i$, write $Z_i=(x_i,y_i)$ and
$h_i=P_n-\delta_{Z_i}$, and put $\mathbf h=(h_i)_{i=1}^n$.  Let
$\mathcal K_i^{\rm ch,\lambda}$ and
$\mathcal K_i^{\rm res,\lambda}$ denote the kernels in
\cref{eq:chord-kernel,eq:resummed-kernel} specialized to deletion $i$.

$|h_i|$ denotes the total-variation measure of $h_i$.  For a measurable map
$\eta:\mathsf Z\to[0,\infty]$,
the $|h_i|$-essential supremum is the least $C\in[0,\infty]$ satisfying
$|h_i|(\{z\in\mathsf Z:\eta(z)>C\})=0$.  For empirical $h_i$, the essential supremum is the
maximum over atoms of positive $|h_i|$-mass, with value zero when $|h_i|=0$.
Since, for $i\in[n]$,
$h_i=n^{-1}\sum_{j\in[n]\setminus\{i\}}(\delta_{Z_j}-\delta_{Z_i})$ and
$\|h_i\|_{\TV}=|h_i|(\mathsf Z)$, the triangle inequality gives
$\|h_i\|_{\TV}\leq2(n-1)/n\leq2$.
For a nonempty set $\varnothing\ne\mathcal X_{\rm ev}\subseteq\mathcal X$,
let $\mathcal L_i^\lambda(t,s;x,z):\R^k\to\R^k$ be a linear operator for every
$i\in[n]$, $(t,\lambda)\in\cH$, $0\leq s\leq t$,
$x\in\mathcal X_{\rm ev}$, and $z\in\mathsf Z$.  Assume
$z\mapsto\|\mathcal L_i^\lambda(t,s;x,z)\|_{\op}$ is
$\mathcal Z$-measurable for every such $(i,t,s,\lambda,x)$ and define
the extended nonnegative uniform kernel seminorm
$\|\mathcal L\|_{\infty,\op;\mathcal X_{\rm ev},\mathbf h}\in[0,\infty]$ by
\begin{equation}\label{eq:uniform-output-kernel-norm}
 \|\mathcal L\|_{\infty,\op;\mathcal X_{\rm ev},\mathbf h}
 =\sup_{\substack{1\leq i\leq n,\ (t,\lambda)\in\cH,\ 0\leq s\leq t\\
                   x\in\mathcal X_{\rm ev}}}
   \operatorname*{ess\,sup}_{z\in\mathsf Z}^{|h_i|}
   \|\mathcal L_i^\lambda(t,s;x,z)\|_{\op}.
\end{equation}
For the chord and response kernels, set
$\mathcal L_i^\lambda=\mathcal K_i^{\rm ch,\lambda}
-\mathcal K_i^{\rm res,\lambda}$,
so that \cref{eq:uniform-output-kernel-norm} bounds the kernel difference
$\mathcal K_i^{\rm ch,\lambda}-\mathcal K_i^{\rm res,\lambda}$ uniformly over
$i$, $(t,\lambda)\in\cH$, $0\leq s\leq t$, $x\in\mathcal X_{\rm ev}$, and
$|h_i|$-almost every $z$.

\begin{theorem}[Output-space deletion bound]
\label{thm:output-space-deletion}
Fix a realized sample, $\cH\subset[0,T]\times\Lambda$, and a nonempty set of
evaluation inputs $\varnothing\ne\mathcal X_{\rm ev}\subseteq\mathcal X$.
Suppose the identities in \cref{thm:kernel-mismatch} hold for every
$i\in[n]$, $(t,\lambda)\in\cH$, and $x\in\mathcal X_{\rm ev}$.
Suppose there are finite nonnegative constants $A_{\rm tr}$ and
$D_{\rm dyn}$ such that
\begin{equation}\label{eq:output-kernel-envelopes}
 \sup_{\substack{1\leq i\leq n,\ (t,\lambda)\in\cH\\0\leq s\leq t}}
 \operatorname*{ess\,sup}_{z\in\mathsf Z}^{|h_i|}
 \|a_s^\lambda(z)\|\leq A_{\rm tr},
 \qquad
 \|\mathcal K^{\rm ch,\lambda}-\mathcal K^{\rm res,\lambda}\|
       _{\infty,\op;\mathcal X_{\rm ev},\mathbf h}\leq D_{\rm dyn}.
\end{equation}
Then
\begin{equation}\label{eq:output-deletion-bound}
 \sup_{\substack{1\leq i\leq n,\ (t,\lambda)\in\cH\\x\in\mathcal X_{\rm ev}}}
 \|f_{\theta_{-i,t}^\lambda}(x)
     -f_{\theta_t^\lambda+d_{i,t}^\lambda}(x)\|
 \leq\frac{2T}{n}A_{\rm tr}D_{\rm dyn}.
\end{equation}
If $x_i\in\mathcal X_{\rm ev}$ for every $i\in[n]$ and
$\varphi(\cdot,y_i)$ is continuously differentiable on a neighborhood of the
line segment joining $f_{\theta_t^\lambda+d_{i,t}^\lambda}(x_i)$ and
$f_{\theta_{-i,t}^\lambda}(x_i)$ for every $i$ and $(t,\lambda)\in\cH$, define
\[
 \bar a_{i,t}^{\rm out,\lambda}
 =\int_0^1\nabla_f\varphi\!\left(
 f_{\theta_t^\lambda+d_{i,t}^\lambda}(x_i)
 +u\{f_{\theta_{-i,t}^\lambda}(x_i)
      -f_{\theta_t^\lambda+d_{i,t}^\lambda}(x_i)\},y_i
 \right)\,\dd u.
\]
If, for a finite $A_{\rm ev}\geq0$,
$\|\bar a_{i,t}^{\rm out,\lambda}\|\leq A_{\rm ev}$ uniformly over
$1\leq i\leq n$ and $(t,\lambda)\in\cH$, then
\begin{equation}\label{eq:output-score-bound}
 \sup_{(t,\lambda)\in\cH}
 |\Score_n^{\LOO}(t,\lambda)-\widetilde\Score_n(t,\lambda)|
 \leq\frac{2T}{n}A_{\rm ev}A_{\rm tr}D_{\rm dyn}.
\end{equation}
\end{theorem}
Supplement Section~S.B proves \cref{thm:output-space-deletion}.

The kernel-difference premise in \eqref{eq:output-kernel-envelopes} admits a
factorwise check.  For deletion $i$, write
$\overline U_i^\lambda$, $\overline J_{i,t}^{\Delta,\lambda}$, and
$\overline J_{i,t}^{d,\lambda}$ for the deletion-specialized objects in
\cref{eq:generic-secant-propagator,eq:chord-jacobians}.  The deletion
specialization of $U^{0,\lambda}$ is the response propagator $U_i^\lambda$
defined in \eqref{eq:flow-alo-propagator}.
Suppose finite nonnegative constants
$J_{\rm tr},G_{\rm ch},J_d,E_J,E_U$ satisfy, uniformly over
$i\in[n]$, $(t,\lambda)\in\cH$, $0\leq s\leq t$, and
$x\in\mathcal X_{\rm ev}$,
\begin{equation}\label{eq:factorwise-kernel-envelopes}
 \begin{gathered}
 \operatorname*{ess\,sup}_{z\in\mathsf Z}^{|h_i|}
 \|J_s^\lambda(x_z)\|_{\op}\leq J_{\rm tr},\qquad
 \|\overline U_i^\lambda(t,s)\|_{\op}\leq G_{\rm ch},\\
 \|\overline J_{i,t}^{d,\lambda}(x)\|_{\op}\leq J_d,\qquad
 \|\overline U_i^\lambda(t,s)-U_i^\lambda(t,s)\|_{\op}\leq E_U,\\
 \|\overline J_{i,t}^{\Delta,\lambda}(x)
       -\overline J_{i,t}^{d,\lambda}(x)\|_{\op}\leq E_J.
 \end{gathered}
\end{equation}
The kernel definitions in \cref{eq:chord-kernel,eq:resummed-kernel}, the
bounds in \eqref{eq:factorwise-kernel-envelopes}, operator-norm
submultiplicativity, and
$\|J_s^\lambda(x_z)^*\|_{\op}=\|J_s^\lambda(x_z)\|_{\op}$ give
\begin{equation}\label{eq:factorwise-kernel-check}
 \begin{gathered}
 \begin{aligned}
 &\mathcal K_i^{\rm ch,\lambda}(t,s;x,z)
   -\mathcal K_i^{\rm res,\lambda}(t,s;x,z)\\
 &\quad=\{\overline J_{i,t}^{\Delta,\lambda}(x)
             -\overline J_{i,t}^{d,\lambda}(x)\}
          \overline U_i^\lambda(t,s)J_s^\lambda(x_z)^*\\
 &\qquad\quad+\overline J_{i,t}^{d,\lambda}(x)
          \{\overline U_i^\lambda(t,s)-U_i^\lambda(t,s)\}
          J_s^\lambda(x_z)^*.
 \end{aligned}\\
 \bigl\|\mathcal K^{\rm ch,\lambda}
          -\mathcal K^{\rm res,\lambda}\bigr\|
       _{\infty,\op;\mathcal X_{\rm ev},\mathbf h}
 \leq J_{\rm tr}\{G_{\rm ch}E_J+J_dE_U\}.
 \end{gathered}
\end{equation}
Equation~\eqref{eq:factorwise-kernel-check} verifies the second inequality in
\eqref{eq:output-kernel-envelopes} with
$D_{\rm dyn}=J_{\rm tr}\{G_{\rm ch}E_J+J_dE_U\}$.
A common $K_{\rm dyn}$ envelope for both kernels permits
$D_{\rm dyn}=2K_{\rm dyn}$; the conditions in
\cref{cor:quadratic-exactness} with $T'=T$ permit $D_{\rm dyn}=0$.

For the scalar-output two-layer mean-field specialization with nonnegative
weight decay, set $d_\lambda=1$, $k=1$, let the input dimension
$d\in\{1,2,\ldots\}$,
$\Lambda\subset[0,\infty)$ be the predetermined nonempty compact set, set
the deterministic $\lambda_{\max}=\max\Lambda\in[0,\infty)$, and specialize $\mathcal X$ to a measurable
subset of $\R^d$ equipped with the Borel sigma-field of $\mathcal X$.  Let the activation be
$\sigma:\R\to\R$.  For each width $m\geq1$, let
$\Theta_m=(\R\times\R^d)^m$ and fix the deterministic array
$q_0^{(m)}=(q_{1,0}^{(m)},\ldots,q_{m,0}^{(m)})\in\Theta_m$.  For the
width-$m$ model, take the generic parameter space and initialization to be
$\Theta=\Theta_m$ and $\theta_0=q_0^{(m)}$.  Identify the generic parameter
$\theta$ with the particle array $q=(q_1,\ldots,q_m)$.  For
$r\in\{1,\ldots,m\}$, write $q_r=(a_r,w_r)\in\R\times\R^d$.  Use
$q_0^{(m)}$ as the common
initialization of every width-$m$ full and deleted flow; each width-$m$
deletion response starts from $0\in\Theta_m$.  When $m$ is fixed, suppress
the superscript $(m)$ on the initial particles.  For $q\in\Theta_m$, let
$f_{m,q}:\mathcal X\to\R$ be the predictor and
$R_m:\Theta_m\to[0,\infty)$ the regularizer.  Equip $\Theta_m$ with the
inner product $\langle\cdot,\cdot\rangle_m:\Theta_m^2\to\R$ and the induced
norm $\|\cdot\|_m:\Theta_m\to[0,\infty)$.  For
$u=(u_r)_{r=1}^m,v=(v_r)_{r=1}^m\in\Theta_m$, define
$f_{m,q}$, $\langle\cdot,\cdot\rangle_m$, $\|\cdot\|_m$, and $R_m$ by
\begin{equation}\label{eq:mean-field-network}
 \begin{aligned}
 f_{m,q}(x)&=\frac1m\sum_{r=1}^m a_r\sigma(w_r^\top x),
 &\langle u,v\rangle_m&=\frac1m\sum_{r=1}^m u_r^\top v_r,\\
 \|u\|_m&=\langle u,u\rangle_m^{1/2},
 &R_m(q)&=\frac12\|q\|_m^2.
 \end{aligned}
\end{equation}
Define $\ell^{\rm ev}_{\lambda,z},\ell^{\rm tr}_{\lambda,z}:\Theta_m\to\R$ by
$\ell^{\rm ev}_{\lambda,z}(q)=:\ell_z(q)
=\varphi(f_{m,q}(x_z),y_z)$ and
$\ell^{\rm tr}_{\lambda,z}(q)=\ell_z(q)+\lambda R_m(q)$.
Although the candidate set remains $\Lambda$, \eqref{eq:mean-field-network}
and the definitions of $\ell^{\rm ev}_{\lambda,z}$ and
$\ell^{\rm tr}_{\lambda,z}$ define the objective and flows for every
$\lambda\in[0,\lambda_{\max}]$; every $\lambda$-derivative in the main paper or
supplement refers to the extended objectives and flows.  All gradients and
adjoints use $\langle\cdot,\cdot\rangle_m$.  Let
$\mathcal K_{m,i}^{\rm ch,\lambda}$ and
$\mathcal K_{m,i}^{\rm res,\lambda}$ denote the
deletion-$i$ kernels in \cref{eq:chord-kernel,eq:resummed-kernel} for the
width-$m$ model in \eqref{eq:mean-field-network}.  For
$\star\in\{{\rm ch},{\rm res}\}$, write
$\mathcal K_m^{\star,\lambda}
:=(\mathcal K_{m,i}^{\star,\lambda})_{i=1}^n$ for the corresponding
deletion-indexed kernel field.  For scalar $g:\R\to\R$, define the extended
sup norm $\|g\|_\infty=\sup_{u\in\R}|g(u)|\in[0,\infty]$, and set
$\sigma^{(0)}=\sigma$.

\begin{theorem}[Two-layer mean-field verification]
\label{thm:two-layer-mean-field}
For $m\geq1$, $n\geq2$, and a realized sample
$S=((x_j,y_j))_{j=1}^n\in\mathsf Z^n$,
\cref{eq:mean-field-network} defines the model and Hilbert geometry.  Fix
$T\in[0,\infty)$ and let
$\cH\subset[0,T]\times\Lambda$.  Suppose
$\sigma\in C^3(\R)$, $\varphi(\cdot,y)\in C^3(\R)$ for every
$y\in\mathcal Y$.  Let
$X,Q_0,(\mathsf S_j)_{j=0}^3$, and $(\mathsf L_j)_{j=1}^3$ be finite nonnegative
deterministic constants.  Suppose
$\max_{1\leq j\leq n}\|x_j\|\leq X$, and
$\max_{1\leq r\leq m}\|q_{r,0}\|\leq Q_0$.  Suppose also that
\begin{equation}\label{eq:mean-field-smoothness}
 \|\sigma^{(j)}\|_\infty\leq \mathsf S_j\quad(0\leq j\leq3),
 \qquad
 \sup_{\substack{u\in\R\\y\in\mathcal Y}}|\partial_u^j\varphi(u,y)|\leq \mathsf L_j
 \quad(1\leq j\leq3).
\end{equation}
Then, for $1\leq i\leq n$ and $\lambda\in\Lambda$, the paths
$\theta_\cdot^\lambda$,
$\theta_{-i,\cdot}^\lambda$, and the deletion response
$d_{i,\cdot}^\lambda$ in \cref{eq:flow-alo-response} exist uniquely on
$[0,T]$.  Moreover, there are
finite deterministic constants
$K_{{\rm ker},T},C_{{\rm ker},T}\in[0,\infty)$, depending only on
$T,\lambda_{\max},X,Q_0,(\mathsf S_j)_{j=0}^3$, and
$(\mathsf L_j)_{j=1}^3$, such that, uniformly for $\lambda\in\Lambda$:
\begin{align}
 \max_{\star\in\{{\rm ch},{\rm res}\}}
 \|\mathcal K_m^{\star,\lambda}\|_{\infty,\op;\{x\in\mathcal X:\|x\|\leq X\},\mathbf h}
 &\leq K_{{\rm ker},T},
 \label{eq:mean-field-kernel-envelope}\\
 \|\mathcal K_m^{\rm ch,\lambda}-\mathcal K_m^{\rm res,\lambda}\|
       _{\infty,\op;\{x\in\mathcal X:\|x\|\leq X\},\mathbf h}
 &\leq\frac{C_{{\rm ker},T}}{n-1}.
 \label{eq:mean-field-kernel-mismatch}
\end{align}
Consequently,
\begin{align}
 \sup_{\substack{1\leq i\leq n,\ (t,\lambda)\in\cH\\
                  x\in\mathcal X,\ \|x\|\leq X}}
 \|f_{m,\theta_{-i,t}^\lambda}(x)
     -f_{m,\theta_t^\lambda+d_{i,t}^\lambda}(x)\|
 &\leq\frac{2T\mathsf L_1C_{{\rm ker},T}}{n(n-1)},
 \label{eq:mean-field-prediction-rate}\\
 \sup_{(t,\lambda)\in\cH}
 |\Score_n^{\LOO}(t,\lambda)-\widetilde\Score_n(t,\lambda)|
 &\leq\frac{2T\mathsf L_1^2C_{{\rm ker},T}}{n(n-1)}.
 \label{eq:mean-field-score-rate}
\end{align}
\end{theorem}

Supplement Section~S.B proves \cref{thm:two-layer-mean-field}.

A direct substitution illustrates the primitive model class.  For
$y\in\{-1,1\}$, $\sigma(v)=\tanh v$, and
$\varphi(u,y)=\log(1+e^{-yu})$, one may take
\[
 (\mathsf S_0,\mathsf S_1,\mathsf S_2,\mathsf S_3)=(1,1,2,6),
 \qquad
 (\mathsf L_1,\mathsf L_2,\mathsf L_3)=(1,1/4,1/4),
\]
with $c_\varphi=\log2$ and $L_0=0$ in the later risk-curve specialization.

The mean-field approximation of Mei, Misiakiewicz and Montanari
\cite{mei2019meanfield} compares finite-width particle dynamics with a
distributional limit, while Bordelon and Pehlevan
\cite{bordelon2022selfconsistent} study evolving two-time kernels at infinite
width and leading finite-width kernel and prediction fluctuations
\cite{bordelon2023finitewidth}.  At each finite $m$,
\eqref{eq:mean-field-kernel-mismatch} compares the
exact-deletion and response kernels, both at width $m$, with a bound uniform in
$m$.

 \section{Estimating population-risk curves from the training sample}\label{sec:statistical}

Call a deterministic anchor $(t_0,\lambda_0)\in\cH$ integrable at sample size
$n$ if
\begin{equation}\label{eq:loo-anchor-integrability}
 P_0|\ell^{\rm ev}_{\lambda_0,Z}
 (\theta_{-i,t_0}^{\lambda_0})|<\infty
 \quad\text{almost surely for every }i\in[n].
\end{equation}
For such an anchor, the
uniform errors are
\begin{align*}
 &\sup_{(t,\lambda)\in\cH}
 |\widetilde\Score_n(t,\lambda)-\Score_n^{\LOO}(t,\lambda)|
 &&\text{\cref{thm:finite-horizon,thm:output-space-deletion}},\\
 &\sup_{(t,\lambda)\in\cH}
 |\Score_n^{\LOO}(t,\lambda)-\Risk_S^-(t,\lambda)|
 &&\text{\cref{thm:loo-concentration}},\\
 &\sup_{(t,\lambda)\in\cH}
 \left|[\Score_n^{\LOO}-\Risk_S^-](t,\lambda)
       -[\Score_n^{\LOO}-\Risk_S^-](t_0,\lambda_0)\right|
 &&\text{\cref{thm:chained-loo-concentration}},\\
 &\sup_{(t,\lambda)\in\cH}
 |\Risk_S^-(t,\lambda)-\Risk_S(t,\lambda)|
 &&\text{\cref{thm:jackknife-cancellation}}.
\end{align*}
\Cref{thm:oracle} gives the explicit regularity-event curve bound
\eqref{eq:risk-curve-recovery}.  \Cref{cor:regularity-event-oracle} gives
observable time-zero centering under chained concentration and the excess-risk
bound \eqref{eq:main-oracle-bound} for a measurable approximate \FlowALO
minimizer.

\subsection{Second-order comparison of deletion and full-sample risks}
\label{sec:second-order-deletion-risk}

Fix $(t,\lambda)\in\cH$.  For a probability measure $P$ on
$(\mathsf Z,\mathcal Z)$, let $\theta_t^{P,\lambda}$ solve
\[
 \dot\theta_t^{P,\lambda}
 =-\nabla F_{P,\lambda}(\theta_t^{P,\lambda}),
 \qquad \theta_0^{P,\lambda}=\theta_0.
\]
The average population risk of the deletion learners $\Risk_S^-$ and the
full-sample risk $\Risk_S$ were defined in
\cref{eq:minus-risk,eq:conditional-risk}.
For $i\in[n]$, put $h_i:=P_n-\delta_{Z_i}$; then
\eqref{eq:measure-deletion} gives
\begin{equation}\label{eq:delete-measure}
 P_{-i}=P_n+\frac{h_i}{n-1},
 \qquad \frac1n\sum_{i=1}^nh_i=0.
\end{equation}
The zero-average deletion-direction identity in \eqref{eq:delete-measure} is
the signed-measure counterpart of the first-order bias cancellation underlying
the classical delete-one jackknife
\cite{quenouille1956bias,miller1974jackknife}.

For a finite signed measure $h$ on $(\mathsf Z,\mathcal Z)$, let $|h|$ be the
total-variation measure of $h$, write
$\|h\|_{\TV}=|h|(\mathsf Z)$, and let $\mathcal M_0$ be the real normed space,
under $\|\cdot\|_{\TV}$, of finite signed measures $h$ on
$(\mathsf Z,\mathcal Z)$ satisfying $h(\mathsf Z)=0$.
For the realized sample $S$, let
$\mathfrak P_S$ be the sample-dependent set of probability measures on
$(\mathsf Z,\mathcal Z)$ forming the deletion chords
\begin{equation}\label{eq:deletion-chords}
 \mathfrak P_S
 =\{P_n+s h_i/(n-1):i\in[n],\ s\in[0,1]\}.
\end{equation}
Let $\mathcal V_S\subset\mathcal M_0$ be the finite-dimensional subspace
\(
 \mathcal V_S=\operatorname{span}\{h_1,\ldots,h_n\}
\)
and, for $i\in[n]$ and $(t,\lambda)\in\cH$, define the sample-dependent chord
map $\psi_{i,t,\lambda}:[0,(n-1)^{-1}]\to\R$ by
\begin{equation}\label{eq:deletion-chord-functional}
 \psi_{i,t,\lambda}(u)
 =P_0\ell^{\rm ev}_{\lambda,Z}
    (\theta_t^{P_n+u h_i,\lambda}).
\end{equation}

With $a=(n-1)^{-1}$, the definition of the deletion chords gives
\begin{equation}\label{eq:jackknife-chord-average}
 \Risk_S^-(t,\lambda)-\Risk_S(t,\lambda)
 =\frac1n\sum_{i=1}^n
   \{\psi_{i,t,\lambda}(a)-\psi_{i,t,\lambda}(0)\}.
\end{equation}

\begin{theorem}[Dynamic jackknife cancellation]\label{thm:jackknife-cancellation}
Suppose the deletion-chord differentiability condition
\begin{equation}\label{eq:deletion-chord-differentiability}
 \psi_{i,t,\lambda}\in C^2([0,(n-1)^{-1}];\R)
 \quad\text{for every }i\in[n]\text{ and }(t,\lambda)\in\cH
\end{equation}
holds.  Suppose also that the direct averaged first-order cancellation condition
\begin{equation}\label{eq:average-first-order-cancellation}
 \sup_{(t,\lambda)\in\cH}
 \left|\frac1n\sum_{i=1}^n\psi_{i,t,\lambda}'(0)\right|=0
\end{equation}
holds, and that there is a finite $C_{PP}\geq0$, possibly depending
on the realized sample $S$, such that \eqref{eq:measure-curvature-bound} holds uniformly over
$i\in[n]$ and $(t,\lambda)\in\cH$:
\begin{equation}\label{eq:measure-curvature-bound}
 \sup_{0\leq u\leq(n-1)^{-1}}|\psi_{i,t,\lambda}''(u)|
 \leq C_{PP}\|h_i\|_{\TV}^2.
\end{equation}
Then,
\begin{equation}\label{eq:cross-size-absolute}
 \sup_{(t,\lambda)\in\cH}
 |\Risk_S^-(t,\lambda)-\Risk_S(t,\lambda)|
 \leq\frac{2C_{PP}}{(n-1)^2}.
\end{equation}
Moreover, for any anchor $(t_0,\lambda_0)\in\cH$,
\begin{equation}\label{eq:cross-size-relative}
 \sup_{(t,\lambda)\in\cH}
 \left|
 [\Risk_S-\Risk_S^-](t,\lambda)
 -[\Risk_S-\Risk_S^-](t_0,\lambda_0)
 \right|
 \leq\frac{4C_{PP}}{(n-1)^2}.
\end{equation}
\end{theorem}
Supplement Section~S.C proves \cref{thm:jackknife-cancellation}.
The curvature premise has the following direct functional check.  The factor
$\|h_i\|_{\TV}^2$ in
\eqref{eq:measure-curvature-bound} comes from taking two derivatives in the
signed training-measure direction $h_i$.  A direct sufficient condition is
that, uniformly over $(t,\lambda)\in\cH$, the risk map
\[
 Q\longmapsto P_0\ell^{\rm ev}_{\lambda,Z}(\theta_t^{Q,\lambda})
\]
have a twice Fr\'echet-differentiable extension to a relative TV-norm
neighborhood of the whole deletion-chord set $\mathfrak P_S$ in
$P_n+\mathcal V_S$, and that
\[
 \left|D^2\!\left[Q\mapsto
 P_0\ell^{\rm ev}_{\lambda,Z}(\theta_t^{Q,\lambda})\right](P)[g,g]\right|
 \leq C_{PP}\|g\|_{\TV}^2,
 \qquad P\in\mathfrak P_S,\quad g\in\mathcal V_S.
\]
The chain rule along $P_n+u h_i$ then gives
\[
 \psi_{i,t,\lambda}''(u)
 =D^2[Q\mapsto P_0\ell^{\rm ev}_{\lambda,Z}(\theta_t^{Q,\lambda})]
   (P_n+u h_i)[h_i,h_i].
\]
Since $\mathcal V_S$ is finite dimensional and
$\mathfrak P_S\times\cH$ is compact, joint continuity in $(P,t,\lambda)$ of
the Hessian in its TV-bilinear operator norm is one simple fixed-sample check
for a finite, possibly sample-dependent $C_{PP}$.  Differentiability only at
$P_n$ is not enough, because \eqref{eq:measure-curvature-bound} ranges over
every point of every deletion chord.

Supplement Corollary~S.C.1 verifies
\eqref{eq:average-first-order-cancellation} and
\eqref{eq:measure-curvature-bound} from normalized-direction continuous
measure-response bounds.

Litman and Guo \cite[Lemma~6.1 and Theorem~F.6]{litman2026theory} use exchangeability to
identify the expectation of leave-one-out evaluation with the population
risk of a size-$(n-1)$ learner and express an expected generalization gap
through sample-replacement self-influences.
The twice-differentiable deletion chords,
\eqref{eq:average-first-order-cancellation}, and
\eqref{eq:measure-curvature-bound} in \cref{thm:jackknife-cancellation}
control the distinct realized-sample difference $\Risk_S^- - \Risk_S$.

Example~\ref{ex:quadratic-mean-flow} verifies
\eqref{eq:average-first-order-cancellation} directly and also gives an
absolute curve bound for an unbounded evaluation loss.
\begin{example}[Quadratic mean flow]\label[example]{ex:quadratic-mean-flow}
Let $(\mathsf Z,\mathcal Z)=(\R,\mathcal B(\R))$, $\Theta=\R$,
$d_\lambda=1$, $\Lambda=\{0\}$, $\cH=[0,T]\times\{0\}$, and $\theta_0=0$, and suppress the
singleton hyperparameter from the notation.  Set
\[
 \ell^{\rm tr}_{0,z}(\theta)=\ell^{\rm ev}_{0,z}(\theta)
 =\frac12(\theta-z)^2.
\]
Write $a_t=1-e^{-t}$, $\overline Z_n=P_nZ$, and
$\widehat m_{2,n}=P_nZ^2$.  For every mass-one finite signed measure $P$ with
$|P|Z^2<\infty$, the quantity $PZ$ is finite, and the identity
$\nabla F_P(\theta)=P(\theta-Z)=\theta-PZ$ follows because $P$ has mass one.
Hence
\eqref{eq:full-flow} becomes $\dot\theta_t^P=PZ-\theta_t^P$ and, with
$\theta_0^P=0$, has solution
$\theta_t^P=(1-e^{-t})PZ=a_tPZ$.  Consequently,
\begin{equation}\label{eq:quadratic-mean-deletion-response}
 \theta_{-i,t}-\theta_t=d_{i,t}
 =\frac{a_t(\overline Z_n-Z_i)}{n-1},
 \qquad \widetilde\Score_n(t)=\Score_n^{\LOO}(t).
\end{equation}
If $\mu=P_0Z$ and $m_2=P_0Z^2<\infty$, the risk functional
$\rho_t(P)=\tfrac12P_0(a_tPZ-Z)^2$ satisfies, for $h\in\mathcal V_S$,
\[
 D\rho_t(P_n)[h]=a_t(a_t\overline Z_n-\mu)hZ,
 \qquad D^2\rho_t(P_n)[h,h]=a_t^2(hZ)^2.
\]
The first derivative is linear in $h$, and $n^{-1}\sum_{i=1}^n h_i=0$.
For the fixed horizon $T$, define the finite sample-dependent constant
\begin{equation}\label{eq:quadratic-cpp-bound}
 C_{PP}:=(1-e^{-T})^2
 \max\left(
 \left\{\frac{(\overline Z_n-Z_i)^2}{\|h_i\|_{\TV}^2}:
 1\leq i\leq n,\ h_i\ne0\right\}\cup\{0\}\right).
\end{equation}
For every $i\in[n]$, $u\in[0,(n-1)^{-1}]$, and $t\in[0,T]$,
$h_iZ=\overline Z_n-Z_i$ and
$|\psi_{i,t,0}''(u)|=a_t^2(h_iZ)^2$.  If $h_i\ne0$, the maximum in
\eqref{eq:quadratic-cpp-bound} and $a_t^2\leq(1-e^{-T})^2$ give
\eqref{eq:measure-curvature-bound}.  If $h_i=0$, then $P_n+u h_i=P_n$ and
$|\psi_{i,t,0}''(u)|=0$.  Expanding $\Risk_S^-(t)-\Risk_S(t)$ and
$\widetilde\Score_n(t)-\Risk_S(t)$ gives
\begin{align}
 \Risk_S^-(t)-\Risk_S(t)
 &=\frac{a_t^2}{2n(n-1)^2}
      \sum_{i=1}^n(Z_i-\overline Z_n)^2,
 \label{eq:quadratic-jackknife-gap}\\
 \widetilde\Score_n(t)-\Risk_S(t)
 &=\frac12(\widehat m_{2,n}-m_2)
   +a_t\overline Z_n(\mu-\overline Z_n)\nonumber\\
 &\quad+\left\{\frac{a_t}{n-1}
              +\frac{a_t^2}{2(n-1)^2}\right\}
       (\widehat m_{2,n}-\overline Z_n^2).
 \label{eq:quadratic-risk-identity}
\end{align}
Supplement Section~S.C derives
\eqref{eq:quadratic-mean-deletion-response}--\eqref{eq:quadratic-risk-identity},
including the intervening derivative and curvature formulas.
Thus the individual displacement in
\eqref{eq:quadratic-mean-deletion-response} is ordinarily $O(n^{-1})$, while
the first-order term $D\rho_t(P_n)[h_i]$ cancels after averaging; if the sample
variance is $O(1)$, \eqref{eq:quadratic-jackknife-gap} is $O(n^{-2})$ uniformly
over $0\leq t\leq T$.
\end{example}

\begin{corollary}[Unbounded quadratic-loss risk curve]
\label[corollary]{cor:quadratic-unbounded-risk}
In \cref{ex:quadratic-mean-flow}, suppose deterministic $0<M_2,M_4<\infty$
satisfy $P_0Z^2\leq M_2$ and $P_0Z^4\leq M_4$.  For
$\delta\in(0,1)$ define
\[
 u_{n,\delta}=\sqrt{\frac{2M_2}{n\delta}},
 \qquad
 v_{n,\delta}=\sqrt{\frac{2M_4}{n\delta}}.
\]
With probability at least $1-\delta$,
\begin{multline}\label{eq:quadratic-unbounded-band}
 \sup_{0\leq t\leq T}|\widetilde\Score_n(t)-\Risk_S(t)|
 \leq \frac{v_{n,\delta}}2
 +\bigl(\sqrt{M_2}+u_{n,\delta}\bigr)u_{n,\delta}\\
 {}+\left\{\frac1{n-1}+\frac1{2(n-1)^2}\right\}
       (M_2+v_{n,\delta}).
\end{multline}
For fixed $\delta$, the bound in \eqref{eq:quadratic-unbounded-band} is
$O(n^{-1/2})$.
\end{corollary}
Supplement Section~S.C proves \cref{cor:quadratic-unbounded-risk};
\cref{ex:quadratic-mean-flow,cor:quadratic-unbounded-risk} need no global loss
or stability bound because $P_0Z^2\leq M_2$ and $P_0Z^4\leq M_4$ control the
empirical moments in \eqref{eq:quadratic-risk-identity}.

\subsection{Concentration of the exact LOO path}

For the generic concentration layer, we first introduce notation for a
candidate-indexed learner.
Let $(\mathcal U,\rho)$ be a predetermined nonempty compact metric candidate
space, and equip $\mathcal U$ with the Borel sigma-field generated by $\rho$.
For $N\geq1$ and $u\in\mathcal U$, let
$L_{N,u}:\mathsf Z^N\times\mathsf Z\to\R$ be the scalar loss obtained by
evaluating the learner indexed by $u$ under one common evaluation rule.  Let
$\mathsf s_N,\mathsf s_N'\in\mathsf Z^N$ denote generic size-$N$ samples, and
write $\mathsf s_N\sim \mathsf s_N'$ when the two samples differ in at most one
coordinate.  Assume $P_0L_{n-1,u}(S_{-i},Z)$ is finite for every
$i\in[n]$ and $u\in\mathcal U$, and define the
sample-dependent fields
$\Score_{n,\mathcal U}^{\LOO},\Risk_{S,\mathcal U}^-:\mathcal U\to\R$
by the samplewise formulas
\[
 \Score_{n,\mathcal U}^{\LOO}(u)
 =\frac1n\sum_{i=1}^nL_{n-1,u}(S_{-i},Z_i),\qquad
 \Risk_{S,\mathcal U}^-(u)
 =\frac1n\sum_{i=1}^nP_0L_{n-1,u}(S_{-i},Z).
\]
For exact LOO, the learner sample size is $N=n-1$, and the stability term is
$\beta_{n-1}$ from \eqref{eq:uniform-stability}.  The response in
\eqref{eq:flow-alo-response}, by
contrast, is propagated along the size-$n$ full-sample path through
\eqref{eq:resummed-B}.

The two concentration routes use the joint-measurability condition
\begin{equation}\label{eq:loo-joint-measurability}
 (\mathsf s_{n-1},z,u)\longmapsto
 L_{n-1,u}(\mathsf s_{n-1},z)
 \quad\text{is jointly measurable}.
\end{equation}

Assume there is a deterministic interval $I_{n-1}\subset\mathbb R$ of finite
length $B_{n-1}^{\rm rng}$ such that
\begin{equation}\label{eq:loss-range-length}
 L_{n-1,u}(\mathsf s_{n-1},z)\in I_{n-1}
 \quad\text{for every }u\in\mathcal U,\quad
 \mathsf s_{n-1}\in\mathsf Z^{n-1},\quad z\in\mathsf Z.
\end{equation}
For each $N\geq1$, define the deterministic extended-nonnegative global
uniform replace-one stability $\beta_N\in[0,\infty]$ by
\begin{equation}\label{eq:uniform-stability}
 \beta_N
 =\sup_{\substack{u\in\mathcal U,\ \mathsf s_N,\mathsf s_N'\in\mathsf Z^N,\ \mathsf s_N\sim \mathsf s_N'\\
                   z\in\mathsf Z}}
 |L_{N,u}(\mathsf s_N,z)-L_{N,u}(\mathsf s_N',z)|.
\end{equation}
Equation~\eqref{eq:uniform-stability} is a same-size replace-one form of global
uniform stability.  Bousquet and Elisseeff
\cite[Definition~6 and the following paragraph]{bousquet2002stability} define
the delete-one form, which implies the corresponding replace-one bound with
constant $2\beta$.  Hardt, Recht and Singer
\cite[Definition~2.1]{hardt2016train} use the same-size replace-one formulation
and analyze stochastic gradient methods under loss- and schedule-specific
conditions.
For iid Euclidean data, under their finite-log-Sobolev, evaluated-loss
dataset-gradient, error-stability, and linear-or-quadratic expected-loss-growth
conditions, Avelin and Viitasaari
\cite[Definitions~2.3 and~2.6, Assumption~3.1, Theorem~3.6, and
Remark~3.4]{avelin2023concentration} give two-sided pointwise concentration of
the LOO score about the full-sample learner's conditional risk.
\Cref{thm:loo-concentration,thm:chained-loo-concentration} instead control a
compact-index supremum about average deletion-learner risk before the separate
deletion-to-full transfer in \cref{thm:jackknife-cancellation}.
Suppose also that $\omega_{n-1}:[0,\infty)\to[0,\infty)$ is deterministic and
nondecreasing, $\omega_{n-1}(r)\downarrow0$ as $r\downarrow0$, and the
size-$(n-1)$ evaluated-loss process has, for every $r\geq0$ and
$u,v\in\mathcal U$, the modulus
\begin{equation}\label{eq:hyper-modulus}
 \rho(u,v)\leq r
 \quad\Longrightarrow\quad
 \sup_{\substack{\mathsf s_{n-1}\in\mathsf Z^{n-1}\\z\in\mathsf Z}}
 |L_{n-1,u}(\mathsf s_{n-1},z)-L_{n-1,v}(\mathsf s_{n-1},z)|
 \leq\omega_{n-1}(r).
\end{equation}
For $r>0$, define $\mathcal N(\mathcal U,\rho,r)$ as the minimum cardinality
of a finite $\mathcal T\subset\mathcal U$ such that every $u\in\mathcal U$
has $v\in\mathcal T$ with $\rho(u,v)\leq r$; compactness of $\mathcal U$ makes
$\mathcal N(\mathcal U,\rho,r)$ finite.
Every entropy integral in the manuscript with lower endpoint zero is
understood as an
improper integral: for every $D\in(0,\infty)$ and nonnegative measurable
$g:(0,D]\to[0,\infty]$, write
$\int_0^D g(u)\,\dd u:=\lim_{\varepsilon\downarrow0}
\int_\varepsilon^D g(u)\,\dd u$.

For the single-model flow, take $\mathcal U=\cH$, $u=(t,\lambda)$, and
$\rho=d_{\cH}$.  For $\mathsf s_N=(z_1,\ldots,z_N)$, define the empirical
measure $P_{\mathsf s_N}=N^{-1}\sum_{j=1}^N\delta_{z_j}$.  Then
whenever the corresponding flow exists through $t$, set
$L_{N,(t,\lambda)}(\mathsf s_N,z)=\ell^{\rm ev}_{\lambda,z}
(\theta_t^{P_{\mathsf s_N},\lambda})$.
Under the single-model specialization,
\eqref{eq:loo-joint-measurability} with $\mathcal U=\cH$ and
$u=(t,\lambda)$ is the required total joint-measurability condition.  By
\cref{eq:loo-score,eq:minus-risk},
$\Score_{n,\mathcal U}^{\LOO}(t,\lambda)=\Score_n^{\LOO}(t,\lambda)$ and
$\Risk_{S,\mathcal U}^-(t,\lambda)=\Risk_S^-(t,\lambda)$.
Whenever $\Score_n^{\LOO}(t,\lambda)$ and $\Risk_S^-(t,\lambda)$ are
real-valued, define the sample-dependent LOO
discrepancy process $D_S:\cH\to\R$ by
\begin{equation}\label{eq:absolute-loo-process}
 D_S(t,\lambda)=\Score_n^{\LOO}(t,\lambda)-\Risk_S^-(t,\lambda).
\end{equation}
\begin{theorem}[Uniform concentration of LOO]\label{thm:loo-concentration}
Assume \eqref{eq:loss-range-length}, finite $\beta_{n-1}$ in
\eqref{eq:uniform-stability}, the size-$(n-1)$ modulus
\eqref{eq:hyper-modulus}, and \eqref{eq:loo-joint-measurability}.  For
$r>0$ and $\delta\in(0,1)$, let $s_n(r,\delta),c_n^-\in[0,\infty)$ be the
finite deterministic concentration radius and bounded-difference coefficient
defined in \cref{eq:sn-definition,eq:bounded-difference-c}.  Then, with probability
at least $1-\delta$,
\begin{equation}\label{eq:loo-concentration-bound}
 \sup_{u\in\mathcal U}
 |\Score_{n,\mathcal U}^{\LOO}(u)-\Risk_{S,\mathcal U}^-(u)|
 \leq s_n(r,\delta),
\end{equation}
where
\begin{align}
 s_n(r,\delta)
 &=2\omega_{n-1}(r)
 +c_n^-\sqrt{\frac n2
 \log\frac{2\mathcal N(\mathcal U,\rho,r)}{\delta}},\label{eq:sn-definition}\\
 c_n^-&=\frac{B_{n-1}^{\rm rng}}n+\frac{2(n-1)}n\beta_{n-1}.
 \label{eq:bounded-difference-c}
\end{align}
\end{theorem}
Supplement Section~S.C proves \cref{thm:loo-concentration}.
For the single-model flow specialization, Supplement Proposition~S.C.3 gives
gradient-flow sufficient conditions for both the finite stability premise
$\beta_{n-1}<\infty$ and the modulus in \eqref{eq:hyper-modulus}.  Under
sample-size-uniform metric-comparison, curvature, and primitive bounds at
fixed $T$, Supplement Proposition~S.C.3 gives
$\beta_{n-1}=O((n-1)^{-1})$ and $\omega_{n-1}(r)=O(r)$; thus the choice
$r=n^{-1/2}$ contributes $2\omega_{n-1}(r)=O(n^{-1/2})$ to
\eqref{eq:sn-definition}.  Supplement Theorem~S.C.2 gives a
horizon-independent absolute-route result under uniform positive contraction.
For $\mathcal U=\cH$, \cref{thm:loo-concentration} controls the LOO level;
relative concentration controls anchored increments and hence the risk curve
up to a selection-invariant common shift.  Under the standing
$d_{\cH}$-compactness of $\cH$, for each $N\geq1$ define the deterministic
extended-nonnegative global hyperparameter-Lipschitz constant
$L_N^{\rm hyp}\in[0,\infty]$ and mixed replace-one stability constant
$\beta_N^{\rm mix}\in[0,\infty]$ by
\begin{align}
 L_N^{\rm hyp}
 &=\sup_{\substack{p,p'\in\cH,\ p\ne p'\\
                    \mathsf s_N\in\mathsf Z^N,\ z\in\mathsf Z}}
 \frac{|L_{N,p}(\mathsf s_N,z)-L_{N,p'}(\mathsf s_N,z)|}
      {d_{\cH}(p,p')},\label{eq:hyper-lipschitz-constant}\\
 \beta_N^{\rm mix}
 &=\sup_{\substack{p,p'\in\cH,\ p\ne p'\\
                    \mathsf s_N,\mathsf s_N'\in\mathsf Z^N,\
                    \mathsf s_N\sim \mathsf s_N'\\ z\in\mathsf Z}}
 \frac{\left|\{L_{N,p}(\mathsf s_N,z)-L_{N,p}(\mathsf s_N',z)\}
 -\{L_{N,p'}(\mathsf s_N,z)-L_{N,p'}(\mathsf s_N',z)\}\right|}
 {d_{\cH}(p,p')}.
 \label{eq:mixed-stability}
\end{align}
When $\cH$ is a singleton, both suprema in
\cref{eq:hyper-lipschitz-constant,eq:mixed-stability} are defined to be zero.
Define the deterministic diameter $D_{\cH}\in[0,\infty)$ and entropy
integral $\mathfrak E(\cH,d_{\cH})\in[0,\infty]$ by
\begin{equation}\label{eq:entropy-integral}
 D_{\cH}=\operatorname{diam}(\cH,d_{\cH}),\qquad
 \mathfrak E(\cH,d_{\cH})
 =\int_0^{D_{\cH}}\sqrt{\log \mathcal N(\cH,d_{\cH},u)}\,\dd u.
\end{equation}

\begin{theorem}[Chained relative concentration of LOO]
\label{thm:chained-loo-concentration}
Assume $L_{n-1}^{\rm hyp},\beta_{n-1}^{\rm mix}<\infty$ for the quantities in
\cref{eq:hyper-lipschitz-constant,eq:mixed-stability}, assume
$\mathfrak E(\cH,d_{\cH})<\infty$ for the entropy integral in
\eqref{eq:entropy-integral}, and assume \eqref{eq:loo-joint-measurability}
with $\mathcal U=\cH$ and $u=(t,\lambda)$.  Define the
finite deterministic mixed-increment scale $a_n^{\rm mix}\in[0,\infty)$ by
\begin{equation}\label{eq:mixed-increment-constant}
 a_n^{\rm mix}=L_{n-1}^{\rm hyp}+(n-1)\beta_{n-1}^{\rm mix}.
\end{equation}
For a deterministic anchor $(t_0,\lambda_0)\in\cH$, define the
sample-dependent anchored increment process
$D_S^\circ(\mathord\cdot;t_0,\lambda_0):\cH\to\R$ by
\begin{equation}\label{eq:relative-loo-process}
\begin{aligned}
 D_S^\circ(t,\lambda;t_0,\lambda_0)
 &=\frac1n\sum_{i=1}^n\bigl[
   L_{n-1,(t,\lambda)}(S_{-i},Z_i)
   -L_{n-1,(t_0,\lambda_0)}(S_{-i},Z_i)\\
 &\qquad
   -P_0\{L_{n-1,(t,\lambda)}(S_{-i},Z)
   -L_{n-1,(t_0,\lambda_0)}(S_{-i},Z)\}\bigr].
\end{aligned}
\end{equation}
There is a universal constant $C_{\rm ch}\in(0,\infty)$ such that, for every
deterministic anchor $(t_0,\lambda_0)$ and every $\delta\in(0,1)$, the finite deterministic chained radius
$s_n^{\rm ch}(\delta)\in[0,\infty)$ is defined in
\eqref{eq:chained-loo-relative-bound}, and, with probability at least $1-\delta$,
\begin{equation}\label{eq:chained-loo-relative-bound}
 \sup_{(t,\lambda)\in\cH}|D_S^\circ(t,\lambda;t_0,\lambda_0)|
 \leq s_n^{\rm ch}(\delta)
 :=\frac{C_{\rm ch}a_n^{\rm mix}}{\sqrt n}
 \left\{\mathfrak E(\cH,d_{\cH})
       +D_{\cH}\sqrt{\log\frac4\delta}\right\}.
\end{equation}
If the deterministic anchor satisfies \eqref{eq:loo-anchor-integrability},
then $D_S:\cH\to\R$ is real-valued and
\[
 D_S^\circ(t,\lambda;t_0,\lambda_0)
 =D_S(t,\lambda)-D_S(t_0,\lambda_0),
 \qquad (t,\lambda)\in\cH.
\]
\end{theorem}
Supplement Section~S.C proves \cref{thm:chained-loo-concentration}.
Supplement Proposition~S.C.4 gives
finite-horizon gradient-flow conditions under which
$L_N^{\rm hyp}=O(1)$ and $\beta_N^{\rm mix}=O(N^{-1})$, and hence
$a_n^{\rm mix}=O(1)$ under a sample-size-uniform metric-comparison constant
and sample-size-uniform displayed primitives.  The bounds from Supplement
Proposition~S.C.4 are finite-horizon; the constants implicit in
$L_N^{\rm hyp}=O(1)$ and $\beta_N^{\rm mix}=O(N^{-1})$ may depend on $T$.
Supplement
Theorem~S.C.2's horizon-independent result concerns the absolute route.
Equation
\eqref{eq:hyper-lipschitz-constant} propagates anchor integrability.

\subsection{Training-sample recovery of population-risk curves}

Patil, Wu and Tibshirani
\cite[Equation~(12) and Theorems~2--3]{patil2024failures} obtain asymptotically
path-uniform LOOCV consistency in proportional high-dimensional least squares.
Bellec and Tan
\cite[Theorem~2.1 and Corollary~2.3]{bellec2024uncertainty} estimate
conditional risk and justify selection over a fixed number of broad updates
in proportional Gaussian linear models.  At an endpoint, Adusumilli, Kasy and
Wilson \cite[Lemma~4]{adusumilli2026sure} show for ridge and lasso ERM in a
fixed-dimensional local-to-zero regime that $n$-fold CV is
asymptotically uniformly equivalent to SURE up to a data-dependent,
tuning-independent constant.  They also give the corresponding limiting
tuned-loss and truncated-risk conclusions
\cite[Theorem~1 and Corollary~1]{adusumilli2026sure}; see also
\cref{sec:related}.  \Cref{thm:oracle} gives a high-probability finite-sample
bound on a predetermined compact continuous-time domain for a smooth,
possibly nonconvex flow.  \Cref{thm:oracle} separates response approximation,
exact-LOO fluctuation, and cross-size terms and permits either absolute
recovery or recovery up to a sample-dependent common shift.

\begin{theorem}[Finite-sample population-risk curve recovery]
\label{thm:oracle}
Let $\tau_n\in[0,1)$ be deterministic, and let
\(
 \overline M_n,\overline L_{3,n},\overline\kappa_n,
 \overline G_{{\rm out},n},\overline C_{PP,n}
\)
be deterministic nonnegative scalars.  Suppose a measurable event
$\mathcal E_n^{\rm reg}$ satisfies
$\Pp(\mathcal E_n^{\rm reg})\geq1-\tau_n$.  On
$\mathcal E_n^{\rm reg}$, suppose \cref{ass:tube} and
\eqref{eq:evaluation-gradient-bound} hold.  On
$\mathcal E_n^{\rm reg}$, suppose also that
\cref{eq:deletion-chord-differentiability,eq:average-first-order-cancellation}
hold and that, for a finite $C_{PP}\geq0$,
\eqref{eq:measure-curvature-bound} holds uniformly over $i\in[n]$ and
$(t,\lambda)\in\cH$.  On $\mathcal E_n^{\rm reg}$, suppose furthermore
that the constants
$M,L_3,\kappa,G_{\rm out},C_{PP}$ satisfy
\begin{equation}\label{eq:regularity-event-envelopes}
 M\leq\overline M_n,\quad
 L_3\leq\overline L_{3,n},\quad
 \kappa\leq\overline\kappa_n,\quad
 G_{\rm out}\leq\overline G_{{\rm out},n},\quad
 C_{PP}\leq\overline C_{PP,n}.
\end{equation}
Assume globally one of the two enumerated sets of concentration hypotheses.
\begin{enumerate}
\item[\emph{Absolute.}] For $\mathcal U=\cH$ and $\rho=d_{\cH}$, a
deterministic interval $I_{n-1}$ of finite length
$B_{n-1}^{\rm rng}$ satisfies \eqref{eq:loss-range-length},
$\beta_{n-1}<\infty$ for the quantity in \eqref{eq:uniform-stability}, the
deterministic nondecreasing modulus $\omega_{n-1}$ satisfies
$\lim_{v\downarrow0}\omega_{n-1}(v)=0$ and
\eqref{eq:hyper-modulus}, and \eqref{eq:loo-joint-measurability} holds.
Fix $r>0$ with $\mathcal N(\cH,d_{\cH},r)<\infty$.
\item[\emph{Chained.}] The quantities in
\cref{eq:hyper-lipschitz-constant,eq:mixed-stability} satisfy
$L_{n-1}^{\rm hyp},\beta_{n-1}^{\rm mix}<\infty$, the entropy integral in
\eqref{eq:entropy-integral} satisfies
$\mathfrak E(\cH,d_{\cH})<\infty$, and
\eqref{eq:loo-joint-measurability} holds with $\mathcal U=\cH$ and
$u=(t,\lambda)$.  A deterministic anchor $(t_0,\lambda_0)\in\cH$ satisfies
\eqref{eq:loo-anchor-integrability}.
\end{enumerate}
For the absolute route, set the deterministic $C_S=0$ and the finite
deterministic radius $s_n^\star(\delta)=s_n(r,\delta)$; for the chained
route, set the sample-dependent $C_S=D_S(t_0,\lambda_0)$ and the finite
deterministic radius $s_n^\star(\delta)=s_n^{\rm ch}(\delta)$.
Thus $s_n^\star:(0,1)\to[0,\infty)$.  For $\delta\in(0,1-\tau_n)$, define
the finite deterministic risk-curve recovery radius
$\mathfrak r_n:(0,1-\tau_n)\to[0,\infty)$ by
\begin{equation}\label{eq:regularity-event-radius}
 \mathfrak r_n(\delta)
 =\frac{\overline G_{{\rm out},n}\overline L_{3,n}
              \overline M_n^2}{2(n-1)^2}
       \mathcal J_{\overline\kappa_n}(T)
  +s_n^\star(\delta)
  +\frac{2\overline C_{PP,n}}{(n-1)^2}.
\end{equation}
Then there is a measurable event $\mathcal O_{n,\delta}$ with
$\Pp(\mathcal O_{n,\delta})\geq1-\delta-\tau_n$ such that, on
$\mathcal O_{n,\delta}$,
\begin{equation}\label{eq:risk-curve-recovery}
 \sup_{(t,\lambda)\in\cH}
 \big|\widetilde\Score_n(t,\lambda)-\Risk_S(t,\lambda)-C_S\big|
 \leq\mathfrak r_n(\delta).
\end{equation}
\end{theorem}
Supplement Section~S.C proves \cref{thm:oracle}.
\begin{corollary}[Time-zero centering and same-sample selection]
\label[corollary]{cor:regularity-event-oracle}
Under the hypotheses and notation of \cref{thm:oracle}, fix
$\delta\in(0,1-\tau_n)$ and let $\mathcal O_{n,\delta}$ be the event supplied
by \cref{thm:oracle}.
\begin{enumerate}
\item If the chained route uses the deterministic integrable anchor
$(t_0,\lambda_0)=(0,\lambda_0)\in\cH$ and all flows start from the common
deterministic initialization in \eqref{eq:full-flow}, then, on
$\mathcal O_{n,\delta}$,
\begin{equation}\label{eq:time-zero-centered-risk-curve}
 \sup_{(t,\lambda)\in\cH}
 \left|\bigl\{\widetilde\Score_n(t,\lambda)
              -\widetilde\Score_n(0,\lambda_0)\bigr\}
       -\bigl\{\Risk_S(t,\lambda)-\Risk_S(0,\lambda_0)\bigr\}\right|
 \leq\mathfrak r_n(\delta).
\end{equation}
\item For either concentration route, let
$\zeta_n:\Omega\to[0,\infty)$ be $\mathcal F$-measurable.  If an
$\mathcal F/\mathcal B(\cH)$-measurable selector
$(\widehat t,\widehat\lambda):\Omega\to\cH$ satisfies
\begin{equation}\label{eq:approximate-score-minimizer}
 \widetilde\Score_n(\widehat t,\widehat\lambda)
 \leq\inf_{(t,\lambda)\in\cH}\widetilde\Score_n(t,\lambda)+\zeta_n,
\end{equation}
then, on $\mathcal O_{n,\delta}$,
\begin{equation}\label{eq:main-oracle-bound}
 \Risk_S(\widehat t,\widehat\lambda)
 -\inf_{(t,\lambda)\in\cH}\Risk_S(t,\lambda)
 \leq2\mathfrak r_n(\delta)+\zeta_n.
\end{equation}
\end{enumerate}
\end{corollary}
Supplement Section~S.C proves \cref{cor:regularity-event-oracle}.

In \cref{cor:mean-field-width-oracle}, $d_\lambda=1$ and $\lambda$ is
nonnegative weight decay.
Assume the hypotheses of \cref{thm:two-layer-mean-field} and deterministic
$c_\varphi\in\R$, $L_0\geq0$, and measurable
$\mathsf Z_0\subseteq\mathsf Z$ such that $P_0(\mathsf Z_0)=1$, $\|x\|\leq X$,
and $c_\varphi\leq\varphi(0,y)\leq c_\varphi+L_0$ on $\mathsf Z_0$.
Use the trace $\sigma$-field
$\mathcal Z_0=\{A\cap\mathsf Z_0:A\in\mathcal Z\}$ and
restrict $P_0$ to $\mathsf Z_0$.

For each $n\geq2$, fix a nonempty finite deterministic width set
$\mathcal W_n\subset\{1,2,\ldots\}$ and deterministic initial particles
$q_{r,0}^{(m)}\in\R\times\R^d$ satisfying, for some
deterministic $Q_0<\infty$,
$\max_{m\in\mathcal W_n,\,1\leq r\leq m}\|q_{r,0}^{(m)}\|\leq Q_0$.
For every $m\in\mathcal W_n$, fix a deterministic nonempty compact
$\cH_m\subset[0,T]\times\Lambda$ under
$d_1((t,\lambda),(t',\lambda'))=|t-t'|+|\lambda-\lambda'|$.
Equip $\mathcal U_n=\bigsqcup_{m\in\mathcal W_n}\{m\}\times\cH_m$ with the
finite topological-sum Borel field.  For $N\geq1$ and
$u=(m,t,\lambda)\in\mathcal U_n$,
define
$L_{N,u}(\mathsf s_N,(x,y))=
\varphi(f_{m,\theta_t^{P_{\mathsf s_N},\lambda}}(x),y)$ whenever the flow
exists, and require joint measurability on
$\mathsf Z_0^N\times\mathsf Z_0\times\mathcal U_n$.
Write $\widetilde\Score_{n,m}$ and $\Risk_{S,m}$ for the width-$m$
specializations on $\cH_m$.  For $m\in\mathcal W_n$ and $v>0$, define
$\mathcal N_m(v)=\mathcal N(\cH_m,d_1,v)$ and
$D_m=\operatorname{diam}(\cH_m,d_1)$.  For $\delta\in(0,1)$, set
\[
\begin{aligned}
 \mathfrak E_m
 &=\int_0^{D_m}\sqrt{\log \mathcal N_m(v)}\,\dd v,
 &D_{\mathcal W}&=\max_{m\in\mathcal W_n}D_m,\\
 \mathfrak E_{\mathcal W}
 &=\max_{m\in\mathcal W_n}\mathfrak E_m,
 &\mathfrak L_{n,\mathcal W,\delta}
 &=\log\frac{8|\mathcal W_n|}{\delta}.
\end{aligned}
\]
$\mathcal N_m,D_m,\mathfrak E_m,D_{\mathcal W},\mathfrak E_{\mathcal W}$ and
$\mathfrak L_{n,\mathcal W,\delta}$ may depend on $n$.
\begin{corollary}[Two-layer mean-field risk curves]
\label[corollary]{cor:mean-field-width-oracle}
For the two-layer model in \eqref{eq:mean-field-network} and the deterministic
finite width family $\mathcal W_n$, for $N\geq1$ write
$\mathsf s_N\sim\mathsf s_N'$ when
$\mathsf s_N,\mathsf s_N'\in\mathsf Z_0^N$ differ in at most one coordinate.
For $m\in\mathcal W_n$ and $N\geq1$, define
\begin{align*}
 \beta_{N,m}
 &=\sup_{\substack{(t,\lambda)\in\cH_m,\
                    \mathsf s_N,\mathsf s_N'\in\mathsf Z_0^N,\
                    \mathsf s_N\sim\mathsf s_N'\\ z\in\mathsf Z_0}}
 \big|L_{N,(m,t,\lambda)}(\mathsf s_N,z)
       -L_{N,(m,t,\lambda)}(\mathsf s_N',z)\big|,\\
 \beta_{N,m}^{\rm mix}
 &=\sup_{\substack{(t,\lambda),(t',\lambda')\in\cH_m,\
                    (t,\lambda)\ne(t',\lambda')\\
                    \mathsf s_N,\mathsf s_N'\in\mathsf Z_0^N,\
                    \mathsf s_N\sim\mathsf s_N',\ z\in\mathsf Z_0}}
 \frac{1}{d_1((t,\lambda),(t',\lambda'))}\\[-2pt]
 &\quad{}\times
 \left|\begin{aligned}
  &L_{N,(m,t,\lambda)}(\mathsf s_N,z)
   -L_{N,(m,t,\lambda)}(\mathsf s_N',z)\\
  &\quad-L_{N,(m,t',\lambda')}(\mathsf s_N,z)
   +L_{N,(m,t',\lambda')}(\mathsf s_N',z)
 \end{aligned}\right|.
\end{align*}
When $\cH_m$ is a singleton, define $\beta_{N,m}^{\rm mix}=0$.
There are finite nonnegative deterministic constants
$B_T,C_{\beta,T},C_{{\rm hyp},T}$ and $C_{\Delta,T},C_{PP,T}$, depending only
on $Q_0,X,T,\lambda_{\max},L_0$, $(\mathsf S_j)_{j=0}^3$, and
$(\mathsf L_j)_{j=1}^3$, such that
$L_{N,(m,t,\lambda)}(\mathsf s_N,z)$ lies in one deterministic interval of
length $B_T$, uniformly over $m\in\mathcal W_n$, $N\geq1$,
$(t,\lambda)\in\cH_m$, $\mathsf s_N\in\mathsf Z_0^N$, and $z\in\mathsf Z_0$.
Uniformly over $m\in\mathcal W_n$ and $N\geq1$, the bounds
$\beta_{N,m}\leq C_{\beta,T}/N$ and
$\beta_{N,m}^{\rm mix}\leq C_{\Delta,T}/N$ hold,
\eqref{eq:hyper-modulus} holds on
$(\mathcal U,\rho)=(\cH_m,d_1)$ with
$\omega_{n-1}(r)=C_{{\rm hyp},T}r$, and the
deletion-chord hypotheses of \cref{thm:jackknife-cancellation} hold with
curvature $C_{PP,T}$.  Define the finite deterministic
simultaneous concentration radius
$s_{n,\mathcal W}^{\rm mf,ch}:(0,1)\to[0,\infty)$ by
\begin{equation}\label{eq:mean-field-width-concentration}
 s_{n,\mathcal W}^{\rm mf,ch}(\delta)
 =\frac{B_T+2C_{\beta,T}}{\sqrt{2n}}
       \sqrt{\mathfrak L_{n,\mathcal W,\delta}}
 +\frac{C_{\rm ch}(C_{{\rm hyp},T}+C_{\Delta,T})}{\sqrt n}
  \left\{\mathfrak E_{\mathcal W}
           +D_{\mathcal W}\sqrt{\mathfrak L_{n,\mathcal W,\delta}}\right\}.
\end{equation}
For every $\delta\in(0,1)$, with probability at least $1-\delta$,
\begin{multline}\label{eq:mean-field-width-risk-curve}
 \max_{m\in\mathcal W_n}\sup_{(t,\lambda)\in\cH_m}
 |\widetilde\Score_{n,m}(t,\lambda)-\Risk_{S,m}(t,\lambda)|\leq\frac{2T\mathsf L_1^2C_{{\rm ker},T}}{n(n-1)}
 +\frac{2C_{PP,T}}{(n-1)^2}
 +s_{n,\mathcal W}^{\rm mf,ch}(\delta).
\end{multline}
\end{corollary}
Supplement Section~S.C proves \cref{cor:mean-field-width-oracle}.

\clearpage
\appendix
\renewcommand{\thesection}{S.\Alph{section}}
\renewcommand{\thetable}{S.\arabic{table}}
\renewcommand{\thefigure}{S.\arabic{figure}}
\numberwithin{equation}{section}
\section{Proofs for deletion dynamics and deterministic approximation}
\label{app:proofs}
\label{app:proofs-deterministic-approximation}

\begin{proof}[Proof of \cref{prop:exact-deletion}]
By \cref{eq:full-flow} with $P=P_{-i}$, the deleted path satisfies
$\dot\theta_{-i,t}^\lambda
=-\nabla F_{P_{-i},\lambda}(\theta_{-i,t}^\lambda)$, while
\cref{eq:full-flow} with $P=P_n$ gives
$\dot\theta_t^\lambda=-\nabla F_{P_n,\lambda}(\theta_t^\lambda)$.
Subtracting the full-sample equation from the deleted-path equation and using
$\Delta_{i,t}^\lambda=\theta_{-i,t}^\lambda-\theta_t^\lambda$ gives
\[
 \dot\Delta_{i,t}^\lambda
 =-\{\nabla F_{P_{-i},\lambda}
          (\theta_t^\lambda+\Delta_{i,t}^\lambda)
      -\nabla F_{P_{-i},\lambda}(\theta_t^\lambda)\}
   -\{\nabla F_{P_{-i},\lambda}(\theta_t^\lambda)
      -\nabla F_{P_n,\lambda}(\theta_t^\lambda)\}.
\]
The first brace equals
$\overline B_{i,t}^\lambda\Delta_{i,t}^\lambda$ by the fundamental
theorem of calculus.  By \cref{eq:deleted-gradient-hessian}, the second brace
equals $-(n-1)^{-1}c_{i,t}^\lambda$.  Therefore,
\cref{eq:exact-deletion-ode} holds.
\end{proof}

\begin{proof}[Proof of \cref{lem:full-path-tube}]
Fix $i\in[n]$ and $\lambda\in\Lambda$.  Along the full path,
$B_{i,t}^\lambda=\nabla^2F_{P_{-i},\lambda}(\theta_t^\lambda)$ and
$c_{i,t}^\lambda$ are continuous, so the linear response equation
\eqref{eq:flow-alo-response} has a unique solution on $[0,T]$
\cite[Theorem~3.2]{khalil2002nonlinear}.
Define the response-norm path $\gamma:[0,T]\to[0,\infty)$ by
$\gamma(t)=\|d_{i,t}^\lambda\|$.  For $t<T$, the upper-right Dini derivative
of $\gamma$,
$D^+\gamma(t):=\limsup_{\varepsilon\downarrow0}
\{\gamma(t+\varepsilon)-\gamma(t)\}/\varepsilon$ is the upper limiting
right slope and does not require differentiability; see Tao
\cite[Section~1.6]{tao2011measure}.

If $d_{i,t}^\lambda\neq0$, \eqref{eq:flow-alo-response} gives
\[
 D^+\gamma(t)
 =\frac{\langle d_{i,t}^\lambda,
       -B_{i,t}^\lambda d_{i,t}^\lambda
       +c_{i,t}^\lambda/(n-1)\rangle}
       {\|d_{i,t}^\lambda\|}.
\]
The curvature bound
\eqref{eq:one-sided-curvature} and the forcing bound
\eqref{eq:full-path-tube-closure-condition} then give
$D^+\gamma(t)\leq\kappa\gamma(t)+M/(n-1)$.  If $d_{i,t}^\lambda=0$,
\eqref{eq:flow-alo-response} and differentiability give
$d_{i,t+\varepsilon}^\lambda=\varepsilon c_{i,t}^\lambda/(n-1)+o(\varepsilon)$
as $\varepsilon\downarrow0$.  Therefore the
definition of $D^+$ gives
$D^+\gamma(t)=\lim_{\varepsilon\downarrow0}
\|d_{i,t+\varepsilon}^\lambda\|/\varepsilon
=\|c_{i,t}^\lambda\|/(n-1)\leq M/(n-1)$.  Hence
$D^+\gamma(t)\leq\kappa\gamma(t)+M/(n-1)$ holds for every $t<T$.  Apply the
scalar comparison lemma of Khalil
\cite[Lemma~3.4]{khalil2002nonlinear} to the scalar equation
$\dot\eta=\kappa\eta+M/(n-1)$ with $\eta(0)=\gamma(0)=0$, where
$\eta:[0,T]\to[0,\infty)$.  Khalil's Lemma 3.4 gives
$\gamma(t)\leq\eta(t)$ for $t<T$, while solving the scalar equation and using
\eqref{eq:phi-J-def} gives
$\eta(t)=M(n-1)^{-1}\int_0^t e^{\kappa(t-s)}\,ds
=M(n-1)^{-1}\phi_\kappa(t)$.  Continuity of $\gamma$ gives
$\gamma(T)\leq M(n-1)^{-1}\phi_\kappa(T)$.

By \cref{ass:tube}, the deleted vector field
$x\mapsto-\nabla F_{P_{-i},\lambda}(x)$ is continuously differentiable, and
hence locally Lipschitz, on $\mathfrak T_{r_{\rm tube}}^\lambda$.  The deleted
initial-value problem therefore has a unique local solution from $\theta_0$.
Put
$\Delta_{i,t}^\lambda=\theta_{-i,t}^\lambda-\theta_t^\lambda$, and stop the local
solution, if necessary, at the first time
$\|\Delta_{i,t}^\lambda\|=r_{\rm tube}$.  Before that time, the line segment from
$\theta_t^\lambda$ to
$\theta_{-i,t}^\lambda$ lies in $B(\theta_t^\lambda,r_{\rm tube})$, and the
exact deletion equation \eqref{eq:exact-deletion-ode}, the one-sided Hessian bound
\eqref{eq:one-sided-curvature}, and the forcing bound in
\eqref{eq:full-path-tube-closure-condition} yield
\[
 D^+\|\Delta_{i,t}^\lambda\|
 \leq\kappa\|\Delta_{i,t}^\lambda\|+(n-1)^{-1}M.
\]
The differential inequality gives
$\|\Delta_{i,t}^\lambda\|\leq(n-1)^{-1}M\phi_\kappa(t)<r_{\rm tube}$, so the stopping
boundary cannot be reached.  Choose $r$ with
$(n-1)^{-1}M\phi_\kappa(T)<r<r_{\rm tube}$.  Were the maximal deleted solution
to end at or before $T$, the maximal deleted solution would remain in the compact set
\[
 \{\theta_t^\lambda+v:0\leq t\leq T,\ v\in\Theta,\ \|v\|\leq r\}
 \subset\mathfrak T_{r_{\rm tube}}^\lambda.
\]
The continuation criterion for locally Lipschitz finite-dimensional ordinary
differential equations \cite[Theorem~3.3]{khalil2002nonlinear}
would then extend the maximal deleted solution, a contradiction.  Consequently,
the deleted solution exists on $[0,T]$ and
\cref{eq:full-path-tube-closure-bound} holds.  The segment claim follows from
convexity of $B(\theta_t^\lambda,r_{\rm tube})$ for each $t\in[0,T]$.
\end{proof}

\begin{proof}[Proof of \cref{thm:finite-horizon}]
Fix $i$ and $\lambda$.  The lower Hessian bound in
\eqref{eq:one-sided-curvature} and the integral form of the gradient
increment imply
\[
 \left\langle\Delta_{i,t}^\lambda,
 \nabla F_{P_{-i},\lambda}(\theta_t^\lambda+\Delta_{i,t}^\lambda)
 -\nabla F_{P_{-i},\lambda}(\theta_t^\lambda)\right\rangle
 \geq-\kappa\|\Delta_{i,t}^\lambda\|^2.
\]
Taking an upper right derivative of the norm in
\cref{eq:exact-deletion-ode} therefore yields
\(
 D^+\|\Delta_{i,t}^\lambda\|
 \leq\kappa\|\Delta_{i,t}^\lambda\|+(n-1)^{-1}M
\).
Equation~\eqref{eq:flow-alo-response} and the inequality
$B_{i,t}^\lambda\succeq-\kappa I$ from
\eqref{eq:one-sided-curvature} give
\(
 D^+\|d_{i,t}^\lambda\|\leq\kappa\|d_{i,t}^\lambda\|+(n-1)^{-1}M
\).
Let $\eta:[0,T]\to[0,\infty)$ solve
\begin{equation*}
 \tag{S.A.0a}\label{eq:scalar-comparison-ivp}
 \dot\eta(t)=\kappa\eta(t)+\frac{M}{n-1},\qquad \eta(0)=0.
\end{equation*}
Since $\|\Delta_{i,0}^\lambda\|=\|d_{i,0}^\lambda\|=0$, Khalil's
Dini-derivative comparison lemma
\cite[Lemma~3.4]{khalil2002nonlinear} applied to
\eqref{eq:scalar-comparison-ivp} gives
$\|\Delta_{i,t}^\lambda\|\leq\eta(t)$ and
$\|d_{i,t}^\lambda\|\leq\eta(t)$.  The scalar solution has the integral form
\begin{equation*}
 \tag{S.A.0b}\label{eq:scalar-comparison-integral-equation}
 \eta(t)=\frac{Mt}{n-1}+\kappa\int_0^t \eta(s)\,\dd s.
\end{equation*}
Since $\kappa,M\geq0$ by \cref{ass:tube}, the Gr\"onwall--Bellman inequality
\cite[Lemma~A.1]{khalil2002nonlinear} applied to
\eqref{eq:scalar-comparison-integral-equation} gives
\[
 \eta(t)\leq\frac{Mt}{n-1}
 +\frac{\kappa M}{n-1}\int_0^t s e^{\kappa(t-s)}\,\dd s
 =\frac{M}{n-1}\phi_\kappa(t).
\]
The inequalities $\|\Delta_{i,t}^\lambda\|\leq\eta(t)$ and
$\|d_{i,t}^\lambda\|\leq\eta(t)$, together with
$\eta(t)\leq M(n-1)^{-1}\phi_\kappa(t)$, give
\cref{eq:deletion-response-size}.

Define the remainder path $r_{i,\cdot}^\lambda:[0,T]\to\Theta$ by
$r_{i,t}^\lambda=\Delta_{i,t}^\lambda-d_{i,t}^\lambda$.  Subtracting
\eqref{eq:flow-alo-response} from \eqref{eq:exact-deletion-ode} gives
\begin{equation*}
 \tag{S.A.0c}\label{eq:nonlinear-response-error-ode}
 \dot r_{i,t}^\lambda=-B_{i,t}^\lambda r_{i,t}^\lambda
 -\{\nabla F_{P_{-i},\lambda}(\theta_t^\lambda+\Delta_{i,t}^\lambda)
    -\nabla F_{P_{-i},\lambda}(\theta_t^\lambda)
    -B_{i,t}^\lambda\Delta_{i,t}^\lambda\}.
\end{equation*}
By \cref{lem:full-path-tube}, the deletion chord is contained in
$B(\theta_t^\lambda,r_{\rm tube})$.  Using
$B_{i,t}^\lambda=\nabla^2F_{P_{-i},\lambda}(\theta_t^\lambda)$, the integral
Taylor formula and the Hessian Lipschitz bound in
\eqref{eq:hessian-lipschitz} give
\begin{equation*}
\tag{S.A.0d}\label{eq:nonlinear-gradient-remainder-bound}
\begin{aligned}
&
 \bigl\|\nabla F_{P_{-i},\lambda}
       (\theta_t^\lambda+\Delta_{i,t}^\lambda)
      -\nabla F_{P_{-i},\lambda}(\theta_t^\lambda)
      -B_{i,t}^\lambda\Delta_{i,t}^\lambda\bigr\| \\
&\quad=\left\|\int_0^1
   \bigl\{\nabla^2F_{P_{-i},\lambda}
      (\theta_t^\lambda+u\Delta_{i,t}^\lambda)
      -\nabla^2F_{P_{-i},\lambda}(\theta_t^\lambda)\bigr\}
      \Delta_{i,t}^\lambda\,\dd u\right\| \\
&\quad\leq\int_0^1L_3u\|\Delta_{i,t}^\lambda\|^2\,\dd u
 =\frac{L_3}{2}\|\Delta_{i,t}^\lambda\|^2.
\end{aligned}
\end{equation*}
For $v\in\Theta$ and $0\leq s\leq\tau\leq t\leq T$, the initial condition
in \eqref{eq:flow-alo-propagator} gives $U_i^\lambda(s,s)v=v$.  The curvature
bound
\eqref{eq:one-sided-curvature}, integration from $s$ to $t$, and maximization
over $\|v\|=1$ give
\begin{equation*}
 \tag{S.A.0e}\label{eq:deleted-propagator-energy-bound}
 \begin{aligned}
 \frac12\frac{\dd}{\dd\tau}
   \|U_i^\lambda(\tau,s)v\|^2
 &=-\langle U_i^\lambda(\tau,s)v,
       B_{i,\tau}^\lambda U_i^\lambda(\tau,s)v\rangle
 \leq\kappa\|U_i^\lambda(\tau,s)v\|^2,\\
 \|U_i^\lambda(t,s)\|_{\op}
 &\leq e^{\kappa(t-s)}.
 \end{aligned}
\end{equation*}
Since $r_{i,0}^\lambda=0$, variation of constants applied to
\eqref{eq:nonlinear-response-error-ode} gives
\begin{equation*}
\tag{S.A.0f}\label{eq:nonlinear-response-duhamel}
\begin{aligned}
 r_{i,t}^\lambda
 &=-\int_0^tU_i^\lambda(t,s)
 \bigl\{\nabla F_{P_{-i},\lambda}
      (\theta_s^\lambda+\Delta_{i,s}^\lambda)
      -\nabla F_{P_{-i},\lambda}(\theta_s^\lambda)
      -B_{i,s}^\lambda\Delta_{i,s}^\lambda\bigr\}\,\dd s.
\end{aligned}
\end{equation*}
Taking norms in \eqref{eq:nonlinear-response-duhamel} and applying
\eqref{eq:nonlinear-gradient-remainder-bound},
\eqref{eq:deleted-propagator-energy-bound}, and
\eqref{eq:deletion-response-size} gives
\begin{equation*}
\tag{S.A.0g}\label{eq:nonlinear-response-integral-bound}
\begin{aligned}
 \|r_{i,t}^\lambda\|
 &\leq\frac{L_3}{2}\int_0^t
      e^{\kappa(t-s)}\|\Delta_{i,s}^\lambda\|^2\,\dd s \\
 &\leq\frac{L_3M^2}{2}(n-1)^{-2}
      \int_0^t e^{\kappa(t-s)}\phi_\kappa(s)^2\,\dd s.
\end{aligned}
\end{equation*}
Using \eqref{eq:phi-J-def}, direct integration gives
\begin{equation*}
 \tag{S.A.0h}\label{eq:nonlinear-response-gain-integral}
 \int_0^t e^{\kappa(t-s)}\phi_\kappa(s)^2\,\dd s
 =\begin{cases}
 \{e^{2\kappa t}-1-2\kappa t e^{\kappa t}\}/\kappa^3,&\kappa>0,\\
 t^3/3,&\kappa=0,
 \end{cases}
 =\mathcal J_\kappa(t).
\end{equation*}
Differentiating the integral in \eqref{eq:nonlinear-response-gain-integral}
gives
$\mathcal J_\kappa'(t)=\phi_\kappa(t)^2
+\kappa\mathcal J_\kappa(t)\geq0$, so
$\mathcal J_\kappa(t)\leq\mathcal J_\kappa(T)$ for $0\leq t\leq T$.
Substituting
$r_{i,t}^\lambda=\Delta_{i,t}^\lambda-d_{i,t}^\lambda$ into
\eqref{eq:nonlinear-response-integral-bound} and taking the supremum over
$1\leq i\leq n$, $0\leq t\leq T$, and $\lambda\in\Lambda$ gives
\cref{eq:parameter-response-error}.  The mean-value theorem for
$\ell^{\rm ev}_{\lambda,Z_i}$ gives
\[
 |\ell^{\rm ev}_{\lambda,Z_i}
      (\theta_t^\lambda+\Delta_{i,t}^\lambda)
   -\ell^{\rm ev}_{\lambda,Z_i}
      (\theta_t^\lambda+d_{i,t}^\lambda)|
 \leq G_{\rm out}\|\Delta_{i,t}^\lambda-d_{i,t}^\lambda\|.
\]
Average over $i$ to obtain \cref{eq:score-approximation-bound}.
\end{proof}

\begin{proof}[Derivation of \cref{eq:full-hessian-response-comparison}]
Fix $i$ and $\lambda$, and put
$e_{i,t}^\lambda=\Delta_{i,t}^\lambda-g_{i,t}^\lambda$.
Subtracting \eqref{eq:full-hessian-influence-response} from the exact secant
equation gives
\[
 \dot e_{i,t}^\lambda=-H_t^\lambda e_{i,t}^\lambda
 +(H_t^\lambda-\overline B_{i,t}^\lambda)\Delta_{i,t}^\lambda,
 \qquad e_{i,0}^\lambda=0.
\]
Because
$B_{i,t}^\lambda=H_t^\lambda-C_{i,t}^\lambda/(n-1)$,
\cref{ass:tube} and the extra Hessian envelope imply
\[
 H_t^\lambda\succeq-\{\kappa+K_C/(n-1)\}I.
\]
The identity
$B_{i,t}^\lambda=H_t^\lambda-C_{i,t}^\lambda/(n-1)$, the integral definition of
$\overline B_{i,t}^\lambda$, and \eqref{eq:hessian-lipschitz} give
\[
 \|H_t^\lambda-\overline B_{i,t}^\lambda\|_{\op}
 \leq \frac{K_C}{n-1}
      +\frac{L_3}{2}\|\Delta_{i,t}^\lambda\|.
\]
The propagator generated by $-H_t^\lambda$ therefore has operator norm at
most $\exp[\{\kappa+K_C/(n-1)\}(t-s)]$.  Variation of constants, followed by
\eqref{eq:deletion-response-size}, yields
\begin{align*}
 \|e_{i,t}^\lambda\|
 &\leq\int_0^t e^{\{\kappa+K_C/(n-1)\}(t-s)}
 \left\{\frac{K_C}{n-1}
      +\frac{L_3}{2}\|\Delta_{i,s}^\lambda\|\right\}
 \|\Delta_{i,s}^\lambda\|\,\dd s\\
 &\leq\frac1{(n-1)^2}\int_0^t
 e^{\{\kappa+K_C/(n-1)\}(t-s)}
 \left\{K_CM\phi_\kappa(s)
       +\frac{L_3M^2}{2}\phi_\kappa(s)^2\right\}\,\dd s.
\end{align*}
Increasing the upper limit and replacing $t$ by $T$ in the nonnegative
exponential factor proves \cref{eq:full-hessian-response-comparison}.
\end{proof}

\begin{proof}[Proof of \cref{cor:quadratic-exactness}]
Fix $T'$ from \cref{cor:quadratic-exactness}, $i\in[n]$, and
$\lambda\in\Lambda$.
If the deleted objective gradient is affine, then
\[
 \nabla F_{P_{-i},\lambda}(\theta_t^\lambda+\Delta_{i,t}^\lambda)
 -\nabla F_{P_{-i},\lambda}(\theta_t^\lambda)
 -B_{i,t}^\lambda\Delta_{i,t}^\lambda=0.
\]
The difference $\Delta_{i,t}^\lambda-d_{i,t}^\lambda$ therefore starts from
zero and satisfies
\[
 \partial_t(\Delta_{i,t}^\lambda-d_{i,t}^\lambda)
 =-B_{i,t}^\lambda(\Delta_{i,t}^\lambda-d_{i,t}^\lambda).
\]
Since $t\mapsto B_{i,t}^\lambda$ is continuous on $[0,T']$, the vector field
$\xi\mapsto-B_{i,t}^\lambda\xi$ on $\Theta$ is globally Lipschitz in $\xi$,
uniformly in $t$.
Khalil \cite[Theorem~3.2]{khalil2002nonlinear} therefore gives uniqueness on
$[0,T']$ for the initial-value problem
$\dot\xi_t=-B_{i,t}^\lambda\xi_t$, $\xi_0=0$.  Since $\xi_t=0$ is a solution,
$\Delta_{i,t}^\lambda=d_{i,t}^\lambda$ for every $t\in[0,T']$.
\end{proof}

\begin{proof}[Verification of \cref{prop:pitchfork}]
By \eqref{eq:pitchfork-loss},
$\partial_x\ell_z(x)=x^3-x-z$ and $R'(x)=x$.  For the full empirical measure
or any deleted empirical measure $P$, write the scalar empirical mean as
$\bar z_P:=\int_{\{-1,1\}}z\,\dd P(z)\in\R$.  Averaging the loss
derivative and negating the resulting objective gradient give
\begin{equation}\label{eq:pitchfork-flow-proof}
 \begin{aligned}
 \partial_xF_{P,\lambda}(x)
 &=P\partial_x\ell_Z(x)+\lambda R'(x)
   =x^3-(1-\lambda)x-\bar z_P,\\
 \dot x_t^{P,\lambda}
 &=-\partial_xF_{P,\lambda}(x_t^{P,\lambda})
   =(1-\lambda)x_t^{P,\lambda}-(x_t^{P,\lambda})^3+\bar z_P.
 \end{aligned}
\end{equation}
For each fixed $P$, the drift $x\mapsto(1-\lambda)x-x^3+\bar z_P$ in
\eqref{eq:pitchfork-flow-proof} is polynomial in $x$, hence locally Lipschitz.
For the balanced full sample,
$\bar z_{P_n}=0$ and $x_0^{P_n,\lambda}=0$ by \cref{prop:pitchfork}.  At
$x=0$, the drift equals
$(1-\lambda)0-0^3+\bar z_{P_n}=0$.  Thus the constant path solves the
full-sample initial-value problem, and local Lipschitz continuity makes the
constant solution unique: $x_t^{P_n,\lambda}=0$ for every $t$.

Deleting a $+1$ observation gives
$\bar z_{P_{-i}}=-(n-1)^{-1}$, and deleting a $-1$ observation gives
$\bar z_{P_{-i}}=(n-1)^{-1}$.  By symmetry, the magnitudes of all deleted
trajectories equal the maximal solution
$\mathfrak x_{n,\lambda}:[0,\tau_{\max})\to\R$, with
$\tau_{\max}\in(0,\infty]$, of
\[
\begin{aligned}
 \dot{\mathfrak x}_{n,\lambda}(t)
 &=b_{n,\lambda}(\mathfrak x_{n,\lambda}(t)),
 &\mathfrak x_{n,\lambda}(0)&=0,\\
 b_{n,\lambda}&:\R\to\R,
 &b_{n,\lambda}(r)&=(1-\lambda)r-r^3+(n-1)^{-1}.
\end{aligned}
\]
The polynomial $b_{n,\lambda}$ is locally Lipschitz, so the initial-value problem has a
unique maximal solution.  For all sufficiently large $n$,
\[
 b_{n,\lambda}(0)=(n-1)^{-1}>0,
 \qquad
 b_{n,\lambda}(1)=-\lambda+(n-1)^{-1}<0.
\]
Thus $b_{n,\lambda}(0)>0$ points the trajectory into $[0,1]$ at the lower boundary,
while $b_{n,\lambda}(1)<0$ points the trajectory into $[0,1]$ at the upper boundary.
A first-exit argument using uniqueness gives
$\mathfrak x_{n,\lambda}(t)\in[0,1]$ throughout the
maximal interval of existence.  If the maximal interval had a finite right
endpoint, boundedness of $b_{n,\lambda}$ on $[0,1]$ would make
$\mathfrak x_{n,\lambda}$ uniformly Lipschitz
up to the endpoint and hence give a limiting value in $[0,1]$.  Local
existence from the limiting value would extend the solution past the finite
endpoint, a contradiction.  Hence the maximal interval is $[0,\infty)$.

Fix $\vartheta\in(0,\sqrt{1-\lambda})$ and define the finite deterministic
positive lower-growth rate
$c_{\vartheta,\lambda}:=1-\lambda-\vartheta^2>0$.  While
$\mathfrak x_{n,\lambda}(t)\leq\vartheta$,
$\dot{\mathfrak x}_{n,\lambda}(t)
\geq c_{\vartheta,\lambda}\mathfrak x_{n,\lambda}(t)+(n-1)^{-1}$.
Until $\mathfrak x_{n,\lambda}(t)=\vartheta$, the integrating-factor
calculation gives
\begin{align*}
 \frac{\dd}{\dd t}\{e^{-c_{\vartheta,\lambda}t}\mathfrak x_{n,\lambda}(t)\}
 &\geq\frac{e^{-c_{\vartheta,\lambda}t}}{n-1},\\
 \mathfrak x_{n,\lambda}(t)
 &\geq\frac{e^{c_{\vartheta,\lambda}t}-1}
 {(n-1)c_{\vartheta,\lambda}}.
\end{align*}
The lower bound equals $\vartheta$ at
\begin{equation}\label{eq:pitchfork-time}
 T_{n,\vartheta}^\lambda
 =\frac1{c_{\vartheta,\lambda}}
   \log\{1+(n-1)c_{\vartheta,\lambda}\vartheta\}
 =O_{\vartheta,\lambda}(\log n).
\end{equation}
Thus $\mathfrak x_{n,\lambda}$ reaches $\vartheta$ no later than
$T_{n,\vartheta}^\lambda$.
Substituting $c_{\vartheta,\lambda}=1-\lambda-\vartheta^2$ into
\eqref{eq:pitchfork-time} gives the
formula for $T_{n,\vartheta}^\lambda$ in \eqref{eq:pitchfork-risk-gap}.
At $\mathfrak x_{n,\lambda}(t)=\vartheta$,
$b_{n,\lambda}(\vartheta)
=c_{\vartheta,\lambda}\vartheta+(n-1)^{-1}>0$, so
$\mathfrak x_{n,\lambda}$ cannot subsequently cross $\vartheta$
downwards.

Under the balanced population law, for $\vartheta\leq |x|\leq1$,
\[
 P_0\ell_Z(x)-P_0\ell_Z(0)
 =-\frac{x^2}{4}(2-x^2)\leq-\frac{\vartheta^2}{4}.
\]
At $T_{n,\vartheta}^\lambda$, every deleted trajectory has magnitude in
$[\vartheta,1]$,
whereas the full trajectory is zero.  Therefore
\[
 \frac1n\sum_{i=1}^n
 \left\{P_0\ell_Z(\theta_{-i,T_{n,\vartheta}^\lambda}^\lambda)
             -P_0\ell_Z(0)\right\}
 \leq-\frac{\vartheta^2}{4}.
\]
The contrast
$\Risk_S^-(T_{n,\vartheta}^\lambda,\lambda)
 -\Risk_S(T_{n,\vartheta}^\lambda,\lambda)$ in
\eqref{eq:pitchfork-risk-gap} is therefore at most $-\vartheta^2/4$, so
\eqref{eq:pitchfork-risk-gap} follows.  At time zero
all trajectories coincide, completing the verification.
\end{proof}

\section{Proofs for the output-space and two-layer mean-field results}
\label{app:proofs-output-space}

\begin{proof}[Proof of \cref{thm:kernel-mismatch}]
The exact deletion difference satisfies
\[
 \dot\Delta_t^\lambda
 =-\overline B_t^\lambda\Delta_t^\lambda
  -\eps g_h^\lambda(\theta_t^{P,\lambda}),
\]
and hence
\[
 \Delta_t^\lambda
 =-\eps\int_0^t\overline U^\lambda(t,s)
   g_h^\lambda(\theta_s^{P,\lambda})\,\dd s.
\]
For every $s\in[0,T]$ and for $|h|$-almost every $z$, the chain rule gives
\[
 \nabla\ell_z(\theta_s^{P,\lambda})
 =J_s^\lambda(x_z)^*a_s^\lambda(z).
\]
The $z$-independent difference between $\ell^{\rm tr}_{\lambda,z}$ and
$\ell_z$ integrates to zero because $h(\mathsf Z)=0$.  Integrating the
pointwise chain-rule identity with respect to $h$ therefore gives
\[
 g_h^\lambda(\theta_s^{P,\lambda})
 =\int J_s^\lambda(x_z)^*a_s^\lambda(z)\,\dd h(z).
\]
The chord fundamental theorem of calculus gives
\(
 f_{\theta_t^{P^\eps,\lambda}}(x)-f_{\theta_t^{P,\lambda}}(x)
 =\overline J_t^{\Delta,\lambda}(x)\Delta_t^\lambda
\), proving \cref{eq:true-prediction-response}.  Applying the chord argument
to the response $d_t^\lambda$ and the propagator representation of
$d_t^\lambda$ in
\cref{eq:generic-resummed-response} gives
\cref{eq:alo-prediction-response}.  Subtracting
\eqref{eq:true-prediction-response} and \eqref{eq:alo-prediction-response} gives
\cref{eq:kernel-mismatch-identity}.

\end{proof}

\begin{proof}[Proof of \cref{thm:output-space-deletion}]
The deletion normalization gives
\[
 (n-1)^{-1}\|h_i\|_{\TV}\leq\frac2n.
\]
By \cref{eq:kernel-mismatch-identity}, the kernel-mismatch and training-gradient envelope then give, for every
$x\in\mathcal X_{\rm ev}$,
\[
 \|f_{\theta_{-i,t}^\lambda}(x)
     -f_{\theta_t^\lambda+d_{i,t}^\lambda}(x)\|
 \leq(n-1)^{-1}\int_0^t
       D_{\rm dyn}A_{\rm tr}\|h_i\|_{\TV}\,\dd s
 \leq\frac{2T}{n}A_{\rm tr}D_{\rm dyn}.
\]
Thus, \cref{eq:output-deletion-bound} holds.  The scalar chord fundamental
theorem of calculus for the omitted-point loss gives
\[
 |\ell_{Z_i}(\theta_{-i,t}^\lambda)
   -\ell_{Z_i}(\theta_t^\lambda+d_{i,t}^\lambda)|
 \leq A_{\rm ev}
       \|f_{\theta_{-i,t}^\lambda}(x_i)
          -f_{\theta_t^\lambda+d_{i,t}^\lambda}(x_i)\|.
\]
Average over $i$ and take the supremum over $(t,\lambda)\in\cH$ to obtain
\cref{eq:output-score-bound}.
\end{proof}

\begin{proof}[Proof of \cref{thm:two-layer-mean-field}]
Fix $m\in\{1,2,\ldots\}$.
For a block vector $u=(u_1,\ldots,u_m)\in\Theta_m$, the normalized Hilbert
norm $\|u\|_m=(m^{-1}\sum_{r=1}^m\|u_r\|^2)^{1/2}$ is the root-mean-square particle
norm, whereas
$\|u\|_{{\rm b},\infty}=\max_{1\leq r\leq m}\|u_r\|$ is the largest
particle-block norm.
For a linear map $H:\Theta_m\to\Theta_m$, we use the induced operator norms
\[
 \|H\|_{m\to m}=\sup_{u\in\Theta_m\setminus\{0\}}\frac{\|Hu\|_m}{\|u\|_m},
 \qquad
 \|H\|_{{\rm b},\infty\to{\rm b},\infty}
 =\sup_{u\in\Theta_m\setminus\{0\}}\frac{\|Hu\|_{{\rm b},\infty}}
                         {\|u\|_{{\rm b},\infty}}.
\]
For $A\geq0$, the \emph{output-weight strip with bound $A$} is
\[
 \{q=(q_r)_{r=1}^m\in\Theta_m:\max_{1\leq r\leq m}|a_r|\leq A\},
 \qquad q_r=(a_r,w_r).
\]
The output-weight strip bounds every scalar output weight $a_r$ and leaves
the hidden weights $w_r$ unrestricted.

For $x\in\mathcal X$ with $\|x\|\leq X$ and one particle (one hidden unit)
$p=(a,w)\in\R\times\R^d$, the scalar
network contribution of $p$ is $a\sigma(w^\top x)$.  The Euclidean gradient
with respect to $p$ is
\[
 \nabla_p\{a\sigma(w^\top x)\}
 =\bigl(\sigma(w^\top x),a\sigma'(w^\top x)x\bigr).
\]
The Euclidean Hessian with respect to $p$, viewed as a linear map on
$\R^{d+1}$, is
\[
 \nabla_p^2\{a\sigma(w^\top x)\}
 =\begin{pmatrix}
    0&\sigma'(w^\top x)x^\top\\
    \sigma'(w^\top x)x&a\sigma''(w^\top x)xx^\top
  \end{pmatrix}.
\]
The deterministic map
$A\mapsto(B_A,D_A,E_A):[0,\infty)\to[0,\infty)^3$ is defined by
\begin{equation}\label{eq:mean-field-proof-constants}
 B_A=\mathsf S_0+A\mathsf S_1X,\qquad
 D_A=2\mathsf S_1X+A\mathsf S_2X^2,\qquad
 E_A=3\mathsf S_2X^2+A\mathsf S_3X^3.
\end{equation}
The quantity $B_A$ bounds the one-particle gradient norm, $D_A$ bounds the
one-particle Hessian norm, and $E_A$ bounds the one-particle Hessian Lipschitz
modulus.  Indeed, for another particle
$p'=(a',w')$, if $\|x\|\leq X$ and the
output-weight coordinate of every point $(1-u)p+up'$, $0\leq u\leq1$, has
absolute value at most $A$, then
\begin{equation}\label{eq:mean-field-particle-derivatives}
 \begin{aligned}
 \|\nabla_p\{a\sigma(w^\top x)\}\|&\leq B_A,\\
 \|\nabla_p^2\{a\sigma(w^\top x)\}\|_{\op}&\leq D_A,\\
 \|\nabla_p^2\{a\sigma(w^\top x)\}
    -\nabla_{p'}^2\{a'\sigma({w'}^\top x)\}\|_{\op}
 &\leq E_A\|p-p'\|.
 \end{aligned}
\end{equation}
The gradient, Hessian, and Hessian-Lipschitz bounds in
\cref{eq:mean-field-particle-derivatives} depend on the
output-weight bound $A$ and are uniform over all hidden weights $w,w'\in\R^d$.
For a full particle array
$q=(q_1,\ldots,q_m)\in\Theta_m$, write
$J_q(x)=D_qf_{m,q}(x):\Theta_m\to\R$ for the derivative of the network output
with respect to all particles.  The adjoint
$J_q(x)^*:\R\to\Theta_m$ is characterized, for $u\in\Theta_m$ and $c\in\R$,
by
$\langle J_q(x)^*c,u\rangle_m=cJ_q(x)u$ and is therefore taken in the
normalized Hilbert inner product.  The second derivative
$D^2f_{m,q}(x):\Theta_m\times\Theta_m\to\R$ is a continuous bilinear form.
For
$u,v\in\Theta_m$ and $c\in\R$, we have
 \begin{align}
 J_q(x)u
 &=\frac1m\sum_{r=1}^m
   \nabla_{q_r}\{a_r\sigma(w_r^\top x)\}^\top u_r,
& J_q(x)^*c
&=\bigl(\nabla_{q_r}\{a_r\sigma(w_r^\top x)\}c\bigr)_{r=1}^m,
 \label{eq:mean-field-jacobian-formula}\\
 D^2f_{m,q}(x)[u,v]
 &=\frac1m\sum_{r=1}^m u_r^\top
   \nabla_{q_r}^2\{a_r\sigma(w_r^\top x)\}v_r.
 \label{eq:mean-field-output-hessian}
\end{align}
For full arrays $q,q'$ in the output-weight strip with bound $A$ and
$x\in\mathcal X$ with $\|x\|\leq X$,
\cref{eq:mean-field-jacobian-formula,eq:mean-field-output-hessian} give
\begin{align}
 \|J_q(x)\|_{\op}&\leq B_A,&
 \|J_q(x)-J_{q'}(x)\|_{\op}
 &\leq D_A\|q-q'\|_{{\rm b},\infty},
 \label{eq:mean-field-jacobian-bounds}\\
 \|D^2f_{m,q}(x)\|_{\op}&\leq D_A,&
 \|D^2f_{m,q}(x)-D^2f_{m,q'}(x)\|_{\op}
 &\leq E_A\|q-q'\|_{{\rm b},\infty}.
 \label{eq:mean-field-output-hessian-bounds}
\end{align}

For $j\in\{1,2\}$, $\partial_1^j\varphi$ denotes the $j$th derivative of
$\varphi$ with respect to the scalar first argument.  For $z=(x,y)\in\mathsf Z$ with $\|x\|\leq X$,
$q\in\Theta_m$, and $v\in\Theta_m$, the gradient and Hessian of one
observation's loss satisfy
\begin{equation}\label{eq:mean-field-full-hessian}
 \begin{aligned}
 \nabla\ell_z(q)
 &=\bigl(\partial_1\varphi(f_{m,q}(x),y)
   \nabla_{q_r}\{a_r\sigma(w_r^\top x)\}\bigr)_{r=1}^m,\\
 \nabla^2\ell_z(q)v
 &=\bigl(\partial_1^2\varphi(f_{m,q}(x),y)\{J_q(x)v\}
   \nabla_{q_r}\{a_r\sigma(w_r^\top x)\}\\
 &\qquad+\partial_1\varphi(f_{m,q}(x),y)
   \nabla_{q_r}^2\{a_r\sigma(w_r^\top x)\}v_r\bigr)_{r=1}^m.
 \end{aligned}
\end{equation}
The Hessian of the weight-decay term is $\lambda I$.
For $A\geq0$, $q\in\Theta_m$ in the output-weight strip with bound $A$,
$\lambda\in\Lambda$, and every probability measure $P$ on
$(\mathsf Z,\mathcal Z)$ satisfying
$P\{(x,y)\in\mathsf Z:\|x\|\leq X\}=1$,
Equations~\eqref{eq:mean-field-particle-derivatives}--
\eqref{eq:mean-field-full-hessian} imply
\begin{equation}\label{eq:mean-field-hessian-bound}
 \max\left\{
 \|\nabla^2F_{P,\lambda}(q)\|_{m\to m},
 \|\nabla^2F_{P,\lambda}(q)\|_{{\rm b},\infty\to{\rm b},\infty}
 \right\}
 \leq\mathsf L_2B_A^2+\mathsf L_1D_A+\lambda_{\max}.
\end{equation}
In \cref{eq:mean-field-hessian-bound}, the $m\to m$ norm uses the
root-mean-square particle norm, whereas the
${\rm b},\infty\to{\rm b},\infty$ norm uses the largest particle-block norm.
The $m\to m$ and ${\rm b},\infty\to{\rm b},\infty$ Hessian bounds in
\cref{eq:mean-field-hessian-bound} are uniform over
$q\in\Theta_m$ in the output-weight strip with bound $A$,
$\lambda\in\Lambda$, and probability measures $P$ satisfying
$P\{(x,y)\in\mathsf Z:\|x\|\leq X\}=1$.

Fix $i\in[n]$ and $\lambda\in\Lambda$; the coordinate identities in
\eqref{eq:mean-field-coordinate-flows} hold for $1\leq r\leq m$ and $0\leq t\leq T$.  Write the particle blocks of
the full and deleted paths as
$\theta_t^\lambda=((a_{r,t}^\lambda,w_{r,t}^\lambda))_{r=1}^m$ and
$\theta_{-i,t}^\lambda=((a_{-i,r,t}^\lambda,w_{-i,r,t}^\lambda))_{r=1}^m$.
For a locally Lipschitz function $g:[0,T]\to\R$ and $0\leq t<T$, write
$D^+g(t)=\limsup_{\varepsilon\downarrow0}
\{g(t+\varepsilon)-g(t)\}/\varepsilon$ for the upper-right Dini
derivative of $g$.  Here $P_n$ is the full empirical measure, $P_{-i}$ is the
deletion-$i$ empirical measure.  For every measurable Bochner-integrable map
$\eta_{\rm obs}:(\mathsf Z,\mathcal Z)\to E$ into a finite-dimensional real vector space
$E$,
$P_n[\eta_{\rm obs}]=n^{-1}\sum_{j=1}^n\eta_{\rm obs}(Z_j)$ and
$P_{-i}[\eta_{\rm obs}]=(n-1)^{-1}\sum_{j\in[n]\setminus\{i\}}
\eta_{\rm obs}(Z_j)$.  The
full-path and deletion-$i$ coordinate equations are
\begin{equation*}
\tag{\ref*{eq:mean-field-particle-radius}a}
\label{eq:mean-field-coordinate-flows}
\begin{aligned}
 \dot a_{r,t}^\lambda
 &=-P_n\!\left[
   \partial_1\varphi(f_{m,\theta_t^\lambda}(x),y)
   \sigma((w_{r,t}^\lambda)^\top x)\right]-\lambda a_{r,t}^\lambda,\\
 \dot w_{r,t}^\lambda
 &=-P_n\!\left[
   \partial_1\varphi(f_{m,\theta_t^\lambda}(x),y)
   a_{r,t}^\lambda\sigma'((w_{r,t}^\lambda)^\top x)x\right]
   -\lambda w_{r,t}^\lambda,\\
 \dot a_{-i,r,t}^\lambda
 &=-P_{-i}\!\left[
   \partial_1\varphi(f_{m,\theta_{-i,t}^\lambda}(x),y)
   \sigma((w_{-i,r,t}^\lambda)^\top x)\right]
   -\lambda a_{-i,r,t}^\lambda,\\
 \dot w_{-i,r,t}^\lambda
 &=-P_{-i}\!\left[
   \partial_1\varphi(f_{m,\theta_{-i,t}^\lambda}(x),y)
   a_{-i,r,t}^\lambda\sigma'((w_{-i,r,t}^\lambda)^\top x)x\right]
   -\lambda w_{-i,r,t}^\lambda.
\end{aligned}
\end{equation*}
Applying $D^+$ to the maxima over $1\leq r\leq m$ of
$|a_{r,t}^\lambda|$, $\|w_{r,t}^\lambda\|$,
$|a_{-i,r,t}^\lambda|$, and $\|w_{-i,r,t}^\lambda\|$, and using
\cref{eq:mean-field-coordinate-flows}, \cref{eq:mean-field-smoothness}, and
nonnegative weight decay gives
\begin{equation*}
\tag{\ref*{eq:mean-field-particle-radius}b}
\label{eq:mean-field-coordinate-dini-bounds}
\begin{aligned}
 D^+\max_{1\leq r\leq m}|a_{r,t}^\lambda|
 &\leq\mathsf L_1\mathsf S_0,&
 D^+\max_{1\leq r\leq m}\|w_{r,t}^\lambda\|
 &\leq\mathsf L_1\mathsf S_1X\max_{1\leq r\leq m}|a_{r,t}^\lambda|,\\
 D^+\max_{1\leq r\leq m}|a_{-i,r,t}^\lambda|
 &\leq\mathsf L_1\mathsf S_0,&
 D^+\max_{1\leq r\leq m}\|w_{-i,r,t}^\lambda\|
 &\leq\mathsf L_1\mathsf S_1X\max_{1\leq r\leq m}|a_{-i,r,t}^\lambda|.
\end{aligned}
\end{equation*}
For every $1\leq i\leq n$, $\lambda\in\Lambda$, and $0\leq t\leq T$,
integrating \cref{eq:mean-field-coordinate-dini-bounds} gives
\begin{equation}\label{eq:mean-field-particle-radius}
 \begin{split}
 \max\!\left\{\max_{1\leq r\leq m}|a_{r,t}^\lambda|,
                 \max_{1\leq r\leq m}|a_{-i,r,t}^\lambda|\right\}
 &\leq A_T:=Q_0+\mathsf L_1\mathsf S_0T,\\
 \max\!\left\{\max_{1\leq r\leq m}\|w_{r,t}^\lambda\|,
                 \max_{1\leq r\leq m}\|w_{-i,r,t}^\lambda\|\right\}
 &\leq Q_0+\mathsf L_1\mathsf S_1X
       \left(Q_0T+\frac{\mathsf L_1\mathsf S_0}{2}T^2\right).
 \end{split}
\end{equation}
$A_T\in[0,\infty)$ is the finite deterministic output-weight radius for all
full and deleted paths through time $T$; $A_T$ depends only on the fixed
mean-field primitives and $T$, not on $m$, $n$, or the realized sample.
The maps $q\mapsto-\nabla F_{P_n,\lambda}(q)$ and
$q\mapsto-\nabla F_{P_{-i},\lambda}(q)$ in
\cref{eq:mean-field-coordinate-flows} are locally Lipschitz.
Equation~\eqref{eq:mean-field-particle-radius} keeps each maximal full-data or
deletion-$i$ trajectory in a fixed compact ball before time $T$.  If either
$\theta_\cdot^\lambda$ or $\theta_{-i,\cdot}^\lambda$ had a finite maximal time
$t_{\max}\leq T$, boundedness of
$q\mapsto-\nabla F_{P_n,\lambda}(q)$ or
$q\mapsto-\nabla F_{P_{-i},\lambda}(q)$ would make
$\theta_\cdot^\lambda$ or $\theta_{-i,\cdot}^\lambda$ Cauchy as
$t\uparrow t_{\max}$.  Local existence from the limiting particle array would
then extend the maximal full-data or deletion-$i$ trajectory beyond
$t_{\max}$, contradicting maximality.  Hence $\theta_\cdot^\lambda$ and
$\theta_{-i,\cdot}^\lambda$ extend uniquely through $[0,T]$.  In
\cref{eq:flow-alo-response},
$B_{i,t}^\lambda:\Theta_m\to\Theta_m$ is the Hessian coefficient and
$c_{i,t}^\lambda\in\Theta_m$ is the centered sample-gradient forcing: one
training-loss gradient minus the empirical average of the training-loss
gradients.  Since $\sigma\in C^3$ and $\varphi(\,\cdot\,,y)\in C^3$, the
definitions in \cref{eq:resummed-B,eq:centered-gradient} imply that
$t\mapsto B_{i,t}^\lambda$ and $t\mapsto c_{i,t}^\lambda$ are continuous on
$[0,T]$ along the full path.
Consequently, the affine response vector field
$v\mapsto-B_{i,t}^\lambda v+c_{i,t}^\lambda/(n-1)$ on $\Theta_m$ from
\cref{eq:flow-alo-response} is globally Lipschitz in $v$, uniformly for
$0\leq t\leq T$.  Khalil's global-Lipschitz existence-and-uniqueness theorem
\cite[Theorem~3.2]{khalil2002nonlinear} therefore gives a unique deletion
response $d_{i,\cdot}^\lambda$ on $[0,T]$.
Define
\begin{equation}\label{eq:mean-field-base-strip-constants}
\begin{aligned}
 B_{{\rm jac},T}&=B_{A_T},\\
 \Gamma&=\mathsf L_2B_{{\rm jac},T}^2+\mathsf L_1D_{A_T}+\lambda_{\max},\\
 \Xi&=\mathsf L_3B_{{\rm jac},T}^3
      +3\mathsf L_2B_{{\rm jac},T}D_{A_T}+\mathsf L_1E_{A_T}.
\end{aligned}
\end{equation}
$B_{{\rm jac},T}$ bounds the network Jacobian, $\Gamma$ bounds the objective Hessian, and
$\Xi$ bounds the Hessian Lipschitz modulus throughout the output-weight strip
with bound $A_T$, which contains the full path and every deletion-$i$ path.
The constants $B_{{\rm jac},T}$, $\Gamma$, and $\Xi$ are finite,
nonnegative, and deterministic functions only of the fixed mean-field
primitives and $T$; none of $B_{{\rm jac},T}$, $\Gamma$, and $\Xi$ depends on
$m$, $n$, or the realized sample.
Differentiating the full Hessian in
\cref{eq:mean-field-full-hessian} and comparing $q$ with $q'$ produces four
contributions.  The change in $\partial_1^2\varphi$ contributes
$\mathsf L_3B_{{\rm jac},T}^3$.  The differences in the two network-Jacobian factors in the
$\partial_1^2\varphi$ term
contribute $2\mathsf L_2B_{{\rm jac},T}D_{A_T}$.  The change in
$\partial_1\varphi$ contributes $\mathsf L_2B_{{\rm jac},T}D_{A_T}$.  The network-output
Hessian contributes $\mathsf L_1E_{A_T}$.  Fix $q,q'\in\Theta_m$,
$\lambda\in\Lambda$, and a probability measure $P$ on
$(\mathsf Z,\mathcal Z)$ satisfying
$P\{(x,y)\in\mathsf Z:\|x\|\leq X\}=1$.  If every
array $(1-u)q+uq'$, $0\leq u\leq1$, lies in the output-weight strip with bound
$A_T$, the $m\to m$ and ${\rm b},\infty\to{\rm b},\infty$ induced operator
norms satisfy
\begin{equation*}
\tag{\ref*{eq:mean-field-base-strip-constants}a}
\label{eq:mean-field-base-strip-hessian-lipschitz}
\begin{aligned}
 \|\nabla^2F_{P,\lambda}(q)-\nabla^2F_{P,\lambda}(q')\|_{m\to m}
 &\leq\Xi\|q-q'\|_{{\rm b},\infty},\\
 \|\nabla^2F_{P,\lambda}(q)-\nabla^2F_{P,\lambda}(q')\|
 _{{\rm b},\infty\to{\rm b},\infty}
 &\leq\Xi\|q-q'\|_{{\rm b},\infty}.
\end{aligned}
\end{equation*}
The propagation gain
$\phi_\Gamma(T)=\int_0^T e^{\Gamma s}\,\dd s$ is defined in
\cref{eq:phi-J-def}.  Set the deletion-displacement scale
\begin{equation}\label{eq:mean-field-deletion-constant}
 A_\Delta=2\mathsf L_1B_{{\rm jac},T}\phi_\Gamma(T).
\end{equation}
$A_\Delta\in[0,\infty)$ is a finite deterministic deletion-displacement
scale.  $A_\Delta$ depends only on the fixed mean-field primitives and $T$, not on
$m$, $n$, or the realized sample.
For the centered sample gradient $c_{i,t}^\lambda$ in
\cref{eq:centered-gradient}, the
weight-decay gradient $\lambda\theta_t^\lambda$ cancels.  The data-loss
gradient bound in \cref{eq:mean-field-full-hessian} therefore gives
$\|c_{i,t}^\lambda\|_{{\rm b},\infty}\leq2\mathsf L_1B_{{\rm jac},T}$.  Write
$\Delta_{i,t}^\lambda=\theta_{-i,t}^\lambda-\theta_t^\lambda$ for the exact
deletion displacement.  \Cref{eq:exact-deletion-ode} has the factor
$1/(n-1)$ established in \cref{eq:deleted-gradient-hessian}.  Convexity of
the output-weight strip with
bound $A_T$ places
$\{\theta_t^\lambda+u\Delta_{i,t}^\lambda:0\leq u\leq1\}$ in the
output-weight strip with bound $A_T$.  The Hessian bound $\Gamma$ from
\cref{eq:mean-field-hessian-bound,eq:mean-field-base-strip-constants},
together with $\Delta_{i,0}^\lambda=0$, gives
\begin{equation*}
\tag{\ref*{eq:mean-field-particle-deletion}a}
\label{eq:mean-field-deletion-dini-bound}
 D^+\|\Delta_{i,t}^\lambda\|_{{\rm b},\infty}
 \leq\Gamma\|\Delta_{i,t}^\lambda\|_{{\rm b},\infty}
       +\frac{2\mathsf L_1B_{{\rm jac},T}}{n-1}.
\end{equation*}
Khalil's scalar comparison lemma
\cite[Lemma~3.4]{khalil2002nonlinear}, applied to
\cref{eq:mean-field-deletion-dini-bound} and the scalar path
$\eta:[0,T]\to[0,\infty)$ satisfying
$\dot\eta=\Gamma\eta+2\mathsf L_1B_{{\rm jac},T}/(n-1)$ with
$\eta(0)=0$, then gives
\begin{equation}\label{eq:mean-field-particle-deletion}
 \sup_{\substack{1\leq i\leq n,\ 0\leq t\leq T\\ \lambda\in\Lambda}}
 \|\Delta_{i,t}^\lambda\|_{{\rm b},\infty}
 \leq(n-1)^{-1}A_\Delta.
\end{equation}
Fix $1\leq i\leq n$ and $\lambda\in\Lambda$.
In \cref{eq:flow-alo-response}, the Hessian coefficient
$B_{i,t}^\lambda$ has induced block-maximum operator norm at most $\Gamma$,
and the forcing $c_{i,t}^\lambda/(n-1)$ has norm at most
 $2\mathsf L_1B_{{\rm jac},T}/(n-1)$.  For $0\leq s\leq t\leq T$, the homogeneous
response propagator $U_i^\lambda(t,s):\Theta_m\to\Theta_m$ solves
$\partial_tU_i^\lambda(t,s)=-B_{i,t}^\lambda U_i^\lambda(t,s)$ with
$U_i^\lambda(s,s)=I$.  For every $v\in\Theta_m$, the response-propagator
generator equation gives
\begin{equation*}
\tag{\ref*{eq:mean-field-particle-response}a}
\label{eq:mean-field-response-propagator-dini}
 D^+\|U_i^\lambda(t,s)v\|_{{\rm b},\infty}
 \leq\Gamma\|U_i^\lambda(t,s)v\|_{{\rm b},\infty}.
\end{equation*}
The variation-of-constants formula in
\cref{eq:flow-alo-duhamel} reads
\[
 d_{i,t}^\lambda
 =\frac1{n-1}\int_0^t
    U_i^\lambda(t,s)c_{i,s}^\lambda\,\dd s.
\]
Applying Khalil's comparison lemma
\cite[Lemma~3.4]{khalil2002nonlinear} to
\cref{eq:mean-field-response-propagator-dini}, followed by substitution in
\cref{eq:flow-alo-duhamel} using
$\|c_{i,s}^\lambda\|_{{\rm b},\infty}\leq2\mathsf L_1B_{{\rm jac},T}$, gives
\begin{equation}\label{eq:mean-field-particle-response}
 \begin{aligned}
 \|U_i^\lambda(t,s)\|_{{\rm b},\infty\to{\rm b},\infty}
 &\leq e^{\Gamma(t-s)},\qquad 0\leq s\leq t\leq T,\\
 \sup_{\substack{1\leq j\leq n,\ 0\leq t\leq T\\ \lambda'\in\Lambda}}
 \|d_{j,t}^{\lambda'}\|_{{\rm b},\infty}
 &\leq(n-1)^{-1}A_\Delta.
 \end{aligned}
\end{equation}

For $\lambda\in\Lambda$ and $0\leq s\leq t\leq T$, the propagators
$\overline U^\lambda(t,s),U^{0,\lambda}(t,s):\Theta_m\to\Theta_m$ in
\cref{eq:generic-secant-propagator} are generated by
$-\overline B_t^\lambda$ and $-B_t^{0,\lambda}$.  The secant Hessian
$\overline B_t^\lambda$ averages the deleted-objective Hessian along the exact
deletion segment, whereas $B_t^{0,\lambda}$ is the deleted-objective Hessian on
the full path.  Define the enlarged-strip constants $A_+$, $B_+$, and $D_+$ by
\begin{equation}\label{eq:mean-field-response-strip-constants}
 A_+=A_T+A_\Delta,\qquad
 B_+=B_{A_+},\qquad D_+=D_{A_+}.
\end{equation}
$A_+$ is the response-strip output-weight radius, while $B_+$ and $D_+$ are
the network-Jacobian and output-Hessian envelopes on the $A_+$ strip.  The
constants $A_+$, $B_+$, and $D_+$ are finite
nonnegative deterministic functions only of the fixed mean-field primitives
and $T$, independent of $m$, $n$, and the sample.
For $q,q'\in\Theta_m$ and $\lambda\in\Lambda$, whenever
$(1-u)q+uq'$ lies in the output-weight strip with bound $A_+$ for every
$0\leq u\leq1$, the chain rule and
Equations~\eqref{eq:mean-field-jacobian-bounds}--
\eqref{eq:mean-field-output-hessian-bounds} give, uniformly in every
probability measure $P$ on
$(\mathsf Z,\mathcal Z)$ satisfying
$P\{(x,y)\in\mathsf Z:\|x\|\leq X\}=1$,
\begin{equation}\label{eq:mean-field-hessian-path-lipschitz}
 \|\nabla^2F_{P,\lambda}(q)-\nabla^2F_{P,\lambda}(q')\|_{m\to m}
 \leq\{\mathsf L_3B_+^3+3\mathsf L_2B_+D_+
          +\mathsf L_1E_{A_+}\}\|q-q'\|_{{\rm b},\infty}.
\end{equation}
The block-maximum induced operator norm satisfies
\begin{equation}\label{eq:mean-field-mixed-lipschitz}
 \|\nabla^2F_{P,\lambda}(q)-\nabla^2F_{P,\lambda}(q')\|
 _{{\rm b},\infty\to{\rm b},\infty}
 \leq\{\mathsf L_3B_+^3+3\mathsf L_2B_+D_+
          +\mathsf L_1E_{A_+}\}\|q-q'\|_{{\rm b},\infty}.
\end{equation}
The change in $\partial_1^2\varphi$ contributes $\mathsf L_3B_+^3$, the two
Jacobian differences contribute $2\mathsf L_2B_+D_+$, the change in
$\partial_1\varphi$ contributes $\mathsf L_2B_+D_+$, and the network-output
Hessian difference contributes $\mathsf L_1E_{A_+}$ to the constants in
\cref{eq:mean-field-hessian-path-lipschitz,eq:mean-field-mixed-lipschitz}.

For the remaining chord and kernel estimates, fix $i\in[n]$,
$\lambda\in\Lambda$, $0\leq s\leq t\leq T$, an evaluation input
$x\in\mathcal X$ with $\|x\|\leq X$, and
$z=(x_z,y_z)\in\mathsf Z$ with $\|x_z\|\leq X$.
The chord Jacobians
$\overline J_t^{\Delta,\lambda}(x),
\overline J_t^{d,\lambda}(x):\Theta_m\to\R$ in
\cref{eq:chord-jacobians} are network
derivatives averaged along two explicit segments:
$\overline J_t^{\Delta,\lambda}(x)$ uses
$\{\theta_t^\lambda+u\Delta_{i,t}^\lambda:0\leq u\leq1\}$, whereas
$\overline J_t^{d,\lambda}(x)$ uses
$\{\theta_t^\lambda+u d_{i,t}^\lambda:0\leq u\leq1\}$.
The full and deleted paths lie in the output-weight strip with bound $A_T$, so
convexity places the exact-deletion segment in the output-weight strip with
bound $A_T$.
\Cref{eq:mean-field-particle-response} places the response segment in the
output-weight strip with bound $A_+$.  The exact-deletion and response segments
therefore lie in the output-weight strip with bound $A_+$; by convexity, for
each $0\leq u\leq1$, the segment joining
$\theta_t^\lambda+u\Delta_{i,t}^\lambda$ to
$\theta_t^\lambda+u d_{i,t}^\lambda$ also lies in the output-weight strip with
bound $A_+$.
\Cref{eq:mean-field-particle-deletion,eq:mean-field-particle-response} imply
$\|\Delta_{i,t}^\lambda-d_{i,t}^\lambda\|_{{\rm b},\infty}
 \leq2(n-1)^{-1}A_\Delta$, and the integral of the chord parameter $u$ over
$[0,1]$ is $1/2$; hence
\begin{equation}\label{eq:mean-field-chord-jacobian-difference}
 \|\overline J_t^{\Delta,\lambda}(x)
     -\overline J_t^{d,\lambda}(x)\|_{\op}
 \leq(n-1)^{-1}A_\Delta D_+.
\end{equation}
The exact-deletion segment lies in the output-weight strip with bound $A_T$.
The definitions of
$\overline B_t^\lambda,B_t^{0,\lambda}$ in
\cref{eq:generic-secant-propagator} give
\begin{equation}\label{eq:mean-field-secant-hessian-difference}
 \begin{aligned}
 \overline B_t^\lambda-B_t^{0,\lambda}
 &=\int_0^1\!\left\{
   \nabla^2F_{P_{-i},\lambda}
      (\theta_t^\lambda+u\Delta_{i,t}^\lambda)
   -\nabla^2F_{P_{-i},\lambda}(\theta_t^\lambda)
  \right\}\,\dd u,\\
 \|\overline B_t^\lambda-B_t^{0,\lambda}\|_{\op}
 &\leq\Xi\int_0^1u\|\Delta_{i,t}^\lambda\|_{{\rm b},\infty}\,\dd u\\
 &\leq\tfrac12(n-1)^{-1}A_\Delta\Xi.
 \end{aligned}
\end{equation}
The operator identity in \eqref{eq:mean-field-secant-hessian-difference}
follows from \eqref{eq:mean-field-particle-deletion}.  The operator-norm
inequality in \eqref{eq:mean-field-secant-hessian-difference} follows from
\eqref{eq:mean-field-base-strip-hessian-lipschitz}.
The exact-deletion segment and full path lie in the output-weight strip with
bound $A_T$, so \cref{eq:mean-field-hessian-bound} gives
$\|\overline B_t^\lambda\|_{\op},
  \|B_t^{0,\lambda}\|_{\op}\leq\Gamma$.
For every $v\in\Theta_m$, the generator equations in
\cref{eq:generic-secant-propagator} give
\begin{equation*}
\tag{\ref*{eq:mean-field-propagator-difference}a}
\label{eq:mean-field-secant-propagator-dini}
 \begin{aligned}
 D^+\|\overline U^\lambda(t,s)v\|_m
 &\leq\Gamma\|\overline U^\lambda(t,s)v\|_m,\\
 D^+\|U^{0,\lambda}(t,s)v\|_m
 &\leq\Gamma\|U^{0,\lambda}(t,s)v\|_m.
 \end{aligned}
\end{equation*}
Khalil's comparison lemma
\cite[Lemma~3.4]{khalil2002nonlinear}, applied to
\cref{eq:mean-field-secant-propagator-dini}, gives
\begin{equation*}
\tag{\ref*{eq:mean-field-propagator-difference}b}
\label{eq:mean-field-secant-propagator-envelope}
 \max\!\left\{\|\overline U^\lambda(t,s)\|_{\op},
                \|U^{0,\lambda}(t,s)\|_{\op}\right\}
 \leq e^{\Gamma(t-s)}.
\end{equation*}
The propagator-difference variation-of-constants identity and norm bound are
\begin{equation}\label{eq:mean-field-propagator-difference}
 \begin{aligned}
 \overline U^\lambda(t,s)-U^{0,\lambda}(t,s)
 &=-\int_s^t\overline U^\lambda(t,r)
    (\overline B_r^\lambda-B_r^{0,\lambda})
    U^{0,\lambda}(r,s)\,\dd r,\\
 \|\overline U^\lambda(t,s)-U^{0,\lambda}(t,s)\|_{\op}
 &\leq\tfrac12(n-1)^{-1}A_\Delta\Xi(t-s)e^{\Gamma(t-s)}.
 \end{aligned}
\end{equation}
The inequality in \cref{eq:mean-field-propagator-difference} follows by
substituting the secant-Hessian bound in
\cref{eq:mean-field-secant-hessian-difference} into the
variation-of-constants identity and applying the $\overline U^\lambda$ and
$U^{0,\lambda}$ propagator bounds in
\cref{eq:mean-field-secant-propagator-envelope}.
Only the response parameter segment can leave the output-weight strip with
bound $A_T$, and \cref{eq:mean-field-particle-response} keeps
$\{\theta_t^\lambda+u d_{i,t}^\lambda:0\leq u\leq1\}$ in the output-weight
strip with bound $A_+$.  The kernel formulas in
\cref{eq:chord-kernel,eq:resummed-kernel}, the Jacobian bounds in
\cref{eq:mean-field-jacobian-bounds}, and the propagator envelope in
\cref{eq:mean-field-secant-propagator-envelope} therefore give the
width-independent uniform bound in \cref{eq:mean-field-kernel-envelope} with
\begin{equation}\label{eq:mean-field-explicit-kernel-envelope}
 K_{{\rm ker},T}=B_+B_{{\rm jac},T} e^{\Gamma T}.
\end{equation}
The product difference in
\cref{eq:chord-kernel,eq:resummed-kernel} is
\begin{equation*}
\tag{\ref*{eq:mean-field-explicit-kernel-envelope}a}
\label{eq:mean-field-kernel-product-difference}
\begin{aligned}
 \mathcal K_{m,i}^{\rm ch,\lambda}(t,s;x,z)
 -\mathcal K_{m,i}^{\rm res,\lambda}(t,s;x,z)
 &=\bigl\{\overline J_t^{\Delta,\lambda}(x)
          -\overline J_t^{d,\lambda}(x)\bigr\}
      \overline U^\lambda(t,s)J_{\theta_s^\lambda}(x_z)^*\\
 &\quad+\overline J_t^{d,\lambda}(x)
      \bigl\{\overline U^\lambda(t,s)-U^{0,\lambda}(t,s)\bigr\}
      J_{\theta_s^\lambda}(x_z)^*.
\end{aligned}
\end{equation*}
The three operator-factor bounds are
\begin{equation*}
\tag{\ref*{eq:mean-field-explicit-kernel-envelope}b}
\label{eq:mean-field-kernel-factor-bounds}
\begin{aligned}
 \|\overline U^\lambda(t,s)\|_{\op}&\leq e^{\Gamma(t-s)},&
 \|\overline J_t^{d,\lambda}(x)\|_{\op}&\leq B_+,\\
 \|J_{\theta_s^\lambda}(x_z)\|_{\op}&\leq B_{{\rm jac},T}.
\end{aligned}
\end{equation*}
The $\overline U^\lambda(t,s)$ bound in
\cref{eq:mean-field-kernel-factor-bounds} uses
\cref{eq:mean-field-secant-propagator-envelope}.  The
$\overline J_t^{d,\lambda}(x)$ and $J_{\theta_s^\lambda}(x_z)$ bounds use
\cref{eq:mean-field-jacobian-bounds} together with the response- and full-path
strip bounds in
\cref{eq:mean-field-particle-response,eq:mean-field-particle-radius}.
Combining \cref{eq:mean-field-kernel-factor-bounds} with
\cref{eq:mean-field-kernel-product-difference,eq:mean-field-chord-jacobian-difference,eq:mean-field-propagator-difference}
gives
\[
 \|\mathcal K_{m,i}^{\rm ch,\lambda}(t,s;x,z)
    -\mathcal K_{m,i}^{\rm res,\lambda}(t,s;x,z)\|_{\op}
 \leq(n-1)^{-1}A_\Delta B_{{\rm jac},T} e^{\Gamma T}
       \left(D_++\frac T2B_+\Xi\right).
\]
Thus the kernel-mismatch constant in \cref{eq:mean-field-kernel-mismatch} can
be taken as
\begin{equation}\label{eq:mean-field-explicit-kernel-constant}
 C_{{\rm ker},T}
 =A_\Delta B_{{\rm jac},T} e^{\Gamma T}
       \left(D_++\frac T2B_+\Xi\right).
\end{equation}
Equation \eqref{eq:mean-field-smoothness} also gives the uniform training- and
evaluation-loss output-gradient bounds
$A_{\rm tr}=A_{\rm ev}=\mathsf L_1$.  On the evaluation domain
$\mathcal X_{\rm ev}=\{x\in\mathcal X:\|x\|\leq X\}$,
\cref{eq:mean-field-kernel-mismatch} and
Equation~\eqref{eq:mean-field-explicit-kernel-constant}
verify the kernel-difference premise of \cref{thm:output-space-deletion} with
$D_{\rm dyn}=C_{{\rm ker},T}/(n-1)$.  Substituting
$A_{\rm tr}=A_{\rm ev}=\mathsf L_1$ and
$D_{\rm dyn}=C_{{\rm ker},T}/(n-1)$ into
\cref{eq:output-deletion-bound,eq:output-score-bound} gives
\cref{eq:mean-field-prediction-rate,eq:mean-field-score-rate}.
\end{proof}

\section{Proofs for population-risk-curve recovery and selection}
\label{app:proofs-statistical-selection}

\subsection{Deletion-to-full risk and exact-LOO concentration}

\begin{proof}[Proof of \cref{thm:jackknife-cancellation}]
Let $h_i=P_n-\delta_{Z_i}$.  By \cref{eq:delete-measure},
$P_{-i}=P_n+(n-1)^{-1}h_i$ and $n^{-1}\sum_{i=1}^n h_i=0$.  One-dimensional Taylor's
formula on the deletion chord gives
\begin{equation}\label{eq:chord-taylor-proof}
 \psi_{i,t,\lambda}((n-1)^{-1})
 =\psi_{i,t,\lambda}(0)+(n-1)^{-1}\psi_{i,t,\lambda}'(0)
  +\int_0^{(n-1)^{-1}}((n-1)^{-1}-u)\psi_{i,t,\lambda}''(u)\,\dd u.
\end{equation}
The average of the linear terms vanishes by
\eqref{eq:average-first-order-cancellation}.  Averaging
\eqref{eq:chord-taylor-proof}, using \eqref{eq:jackknife-chord-average}, and
then using $\|h_i\|_{\TV}\leq2$ gives
\[
 |\Risk_S^-(t,\lambda)-\Risk_S(t,\lambda)|
 \leq\frac{C_{PP}}{2n(n-1)^2}\sum_{i=1}^n\|h_i\|_{\TV}^2
 \leq\frac{2C_{PP}}{(n-1)^2}.
\]
The averaged remainder bound proves \cref{eq:cross-size-absolute}.  Subtracting
$[\Risk_S-\Risk_S^-](t_0,\lambda_0)$ from
$[\Risk_S-\Risk_S^-](t,\lambda)$ and applying the
$2C_{PP}(n-1)^{-2}$ bound at $(t,\lambda)$ and at $(t_0,\lambda_0)$ proves
\cref{eq:cross-size-relative}.
\end{proof}

\begin{proof}[Calculations for \cref{ex:quadratic-mean-flow}]
For $i\in[n]$, let $h_i=P_n-\delta_{Z_i}$.  Then
$P_{-i}=P_n+(n-1)^{-1}h_i$ and $h_iZ=\overline Z_n-Z_i$.
The exact displacement $\theta_{-i,t}-\theta_t$ and the response $d_{i,t}$
both start from zero and solve
\[
 \dot x_t=-x_t+\frac{h_iZ}{n-1},
\]
because $P_{-i}Z-P_nZ=h_iZ/(n-1)$ and, in the response equation,
$B_{i,t}=1$, $c_{i,t}=h_iZ$, and $C_{i,t}=0$.  Thus
$x_t=a_th_iZ/(n-1)$.  This proves the displacement identity in
\eqref{eq:quadratic-mean-deletion-response}; because
$\theta_t+d_{i,t}=\theta_{-i,t}$, the held-out losses also agree term by term,
which proves the score equality.

For every $t\in[0,T]$, $h\in\mathcal V_S$, and $u\in\R$,
\[
 \rho_t(P_n+uh)
 =\frac12\left\{m_2-2a_t\mu(\overline Z_n+uhZ)
                  +a_t^2(\overline Z_n+uhZ)^2\right\}.
\]
Differentiating the map $u\mapsto\rho_t(P_n+uh)$ at $u=0$ gives
$D\rho_t(P_n)[h]=a_t(a_t\overline Z_n-\mu)hZ$ and
$D^2\rho_t(P_n)[h,h]=a_t^2(hZ)^2$, as stated in
\cref{ex:quadratic-mean-flow}.  Fix $i\in[n]$, $t\in[0,T]$, and
$u\in[0,(n-1)^{-1}]$.  Along the $i$th deletion chord,
$|\psi_{i,t,0}''(u)|=a_t^2(h_iZ)^2$.  If $h_i\ne0$, then
\[
 |\psi_{i,t,0}''(u)|
 \leq (1-e^{-T})^2\frac{(h_iZ)^2}{\|h_i\|_{\TV}^2}
       \|h_i\|_{\TV}^2
 \leq C_{PP}\|h_i\|_{\TV}^2;
\]
if $h_i=0$, then $P_n+u h_i=P_n$ and
$|\psi_{i,t,0}''(u)|=0=C_{PP}\|h_i\|_{\TV}^2$.
The $h_i\ne0$ and $h_i=0$ bounds verify
\eqref{eq:measure-curvature-bound} with $C_{PP}$ defined in
\eqref{eq:quadratic-cpp-bound}.

The quadratic Taylor identity is exact:
\[
 \rho_t(P_{-i})-\rho_t(P_n)
 =\frac{a_t(a_t\overline Z_n-\mu)}{n-1}h_iZ
  +\frac{a_t^2}{2(n-1)^2}(h_iZ)^2.
\]
Averaging, using $n^{-1}\sum_i h_iZ=0$ and
$h_iZ=\overline Z_n-Z_i$, gives
\eqref{eq:quadratic-jackknife-gap}.  Finally,
\[
 P_{-i}Z=\overline Z_n-\frac{Z_i-\overline Z_n}{n-1},
 \qquad
 a_tP_{-i}Z-Z_i
 =(a_t-1)\overline Z_n
  -\left(1+\frac{a_t}{n-1}\right)(Z_i-\overline Z_n).
\]
Since $n^{-1}\sum_i(Z_i-\overline Z_n)=0$,
\begin{align*}
 2\widetilde\Score_n(t)
 &=(1-a_t)^2\overline Z_n^2
   +\left(1+\frac{a_t}{n-1}\right)^2
      (\widehat m_{2,n}-\overline Z_n^2),\\
 2\Risk_S(t)
 &=m_2-2a_t\mu\overline Z_n+a_t^2\overline Z_n^2.
\end{align*}
Subtracting these expressions and collecting powers of $a_t$ gives
\eqref{eq:quadratic-risk-identity}.
\end{proof}

\begin{proof}[Proof of \cref{cor:quadratic-unbounded-risk}]
By Cauchy--Schwarz, $\mu^2\leq P_0Z^2\leq M_2$.  Moreover,
\[
 \operatorname{Var}(\overline Z_n)
 \leq\frac{M_2}{n},
 \qquad
 \operatorname{Var}(\widehat m_{2,n})
 \leq\frac{M_4}{n}.
\]
Define the Chebyshev event
\[
 \mathcal Q_{n,\delta}
 :=\{|\overline Z_n-\mu|\leq u_{n,\delta}\}
   \cap\{|\widehat m_{2,n}-m_2|\leq v_{n,\delta}\}.
\]
Chebyshev's inequality and a union bound give
$\Pp(\mathcal Q_{n,\delta})\geq1-\delta$.  On
$\mathcal Q_{n,\delta}$,
$|\overline Z_n|\leq\sqrt{M_2}+u_{n,\delta}$ and the nonnegative sample
variance satisfies
$\widehat m_{2,n}-\overline Z_n^2\leq\widehat m_{2,n}
 \leq M_2+v_{n,\delta}$.  Since $0\leq a_t\leq1$, substituting
$|\overline Z_n|\leq\sqrt{M_2}+u_{n,\delta}$ and
$\widehat m_{2,n}-\overline Z_n^2\leq M_2+v_{n,\delta}$
into \eqref{eq:quadratic-risk-identity} proves
\eqref{eq:quadratic-unbounded-band}, uniformly over $0\leq t\leq T$.
\end{proof}

\subsubsection{Continuous measure responses}
\label{app:continuous-measure-responses}

The curvature bound in \cref{cor:cpp-bound} uses the first two derivatives of
the flow with respect to an affine perturbation of the training measure.
The real normed space $\mathcal M_0$ from
\cref{sec:second-order-deletion-risk} consists of zero-mass finite signed measures on
$(\mathsf Z,\mathcal Z)$, equipped with
$\|h\|_{\TV}=|h|(\mathsf Z)$.
Fix $\lambda\in\Lambda$, a probability measure $P$ on
$(\mathsf Z,\mathcal Z)$, and
$h\in\mathcal M_0$.  Let $\mathcal I_{P,h}\subset\R$ be an interval
containing zero.  For every $\alpha\in\mathcal I_{P,h}$, assume that
$F_{P+\alpha h,\lambda}$ is defined and that \eqref{eq:full-flow} with
training measure $P+\alpha h$ has a solution on $[0,T]$.
Write $\theta_t^{\alpha,\lambda}=\theta_t^{P+\alpha h,\lambda}$.

For every $t\in[0,T]$, assume that the map
$\alpha\mapsto\theta_t^{\alpha,\lambda}:\mathcal I_{P,h}\to\Theta$ is twice continuously differentiable
on $\mathcal I_{P,h}$, with one-sided derivatives at the endpoints.  Assume
also that the first two $\alpha$-derivatives commute with the time derivative:
\[
 \partial_t\partial_\alpha^k\theta_t^{\alpha,\lambda}
 =\partial_\alpha^k\partial_t\theta_t^{\alpha,\lambda},
 \qquad 0\leq t\leq T,\quad \alpha\in\mathcal I_{P,h},\quad k\in\{1,2\}.
\]
Let $\mathcal O\subset\Theta$ be an open set containing
$\{\theta_s^{\alpha,\lambda}:0\leq s\leq T,\
\alpha\in\mathcal I_{P,h}\}$.  On $\mathcal O$, suppose that, for each
$z\in\mathsf Z$, the map $\theta\mapsto\ell^{\rm tr}_{\lambda,z}(\theta)$ is
$C^3$.  For every
$\theta\in\mathcal O$, $\alpha\in\mathcal I_{P,h}$, and $k=1,2,3$, assume
that the integrals against $P+\alpha h$ and $h$ both exist and that
$\nabla_\theta^kF_{P+\alpha h,\lambda}(\theta)
=\int\nabla_\theta^k\ell^{\rm tr}_{\lambda,z}(\theta)
\,\dd(P+\alpha h)(z)$.  For the evaluation loss, suppose that the map
\[
 \mathcal O\longrightarrow\R,\qquad
 \theta\longmapsto P_0\ell^{\rm ev}_{\lambda,Z}(\theta)
 =\int_{\mathsf Z}\ell^{\rm ev}_{\lambda,z}(\theta)\,\dd P_0(z)
\]
is $C^2$.  Assume that the first two derivatives of
$P_0\ell^{\rm ev}_{\lambda,Z}$ are obtained by differentiating under $P_0$.

At $\alpha=0$, define the response paths
$u_\cdot^\lambda[h],u_\cdot^{(2),\lambda}[h,h]:[0,T]\to\Theta$ by
\[
 u_t^\lambda[h]
 =\left.\partial_\alpha\theta_t^{\alpha,\lambda}\right|_{\alpha=0},
 \qquad
 u_t^{(2),\lambda}[h,h]
 =\left.\partial_\alpha^2\theta_t^{\alpha,\lambda}\right|_{\alpha=0}.
\]
Along the base path, for every $0\leq t\leq T$, define linear operators
$A_t^\lambda,H_{h,t}^\lambda:\Theta\to\Theta$, a vector
$g_{h,t}^\lambda\in\Theta$, and a bilinear map
$\mathsf T_t^\lambda:\Theta\times\Theta\to\Theta$ by
\begin{align*}
 A_t^\lambda
 &=\nabla^2F_{P,\lambda}(\theta_t^{P,\lambda}),\\
 g_{h,t}^\lambda
 &=\int \nabla\ell^{\rm tr}_{\lambda,z}(\theta_t^{P,\lambda})
       \,\dd h(z),\\
 H_{h,t}^\lambda
 &=\int \nabla^2\ell^{\rm tr}_{\lambda,z}(\theta_t^{P,\lambda})
       \,\dd h(z).
\end{align*}
Define $\mathsf T_t^\lambda$ by the Riesz identity
\[
 \langle v,\mathsf T_t^\lambda[u,w]\rangle
 =\nabla_\theta^3F_{P,\lambda}(\theta_t^{P,\lambda})[v,u,w],
 \qquad u,v,w\in\Theta.
\]

For every $\theta\in\mathcal O$ and $\alpha\in\mathcal I_{P,h}$, the affine
dependence of \eqref{eq:objective-vector-field} on the training measure gives
\[
 \nabla F_{P+\alpha h,\lambda}(\theta)
 =\nabla F_{P,\lambda}(\theta)
  +\alpha\int\nabla\ell^{\rm tr}_{\lambda,z}(\theta)\,\dd h(z).
\]
Differentiating \eqref{eq:full-flow} once with respect to $\alpha$ at
$\alpha=0$ gives
\begin{equation}\label{eq:first-data-response}
 \dot u_t^\lambda[h]
 =-A_t^\lambda u_t^\lambda[h]-g_{h,t}^\lambda,
 \qquad u_0^\lambda[h]=0.
\end{equation}

Differentiating
$\nabla F_{P,\lambda}(\theta_t^{\alpha,\lambda})$ twice with respect to
$\alpha$ and evaluating at $\alpha=0$ gives the two chain-rule terms
$A_t^\lambda u_t^{(2),\lambda}[h,h]$ and
$\mathsf T_t^\lambda[u_t^\lambda[h],u_t^\lambda[h]]$.
Differentiating the affine perturbation twice at $\alpha=0$ gives
\[
 \left.\partial_\alpha^2\left\{
  \alpha\int\nabla\ell^{\rm tr}_{\lambda,z}
       (\theta_t^{\alpha,\lambda})\,\dd h(z)
 \right\}\right|_{\alpha=0}
 =2H_{h,t}^\lambda u_t^\lambda[h].
\]
The factor two is the coefficient of the cross term in the second-order
product rule.  Because the common initialization does not depend on $\alpha$,
the second response satisfies
\begin{equation}\label{eq:second-data-response}
\begin{aligned}
 \dot u_t^{(2),\lambda}[h,h]
 &=-A_t^\lambda u_t^{(2),\lambda}[h,h]
   -2H_{h,t}^\lambda u_t^\lambda[h]\\
 &\quad-\mathsf T_t^\lambda[u_t^\lambda[h],u_t^\lambda[h]],
 \qquad u_0^{(2),\lambda}[h,h]=0.
\end{aligned}
\end{equation}

The ordinary chain rule for the $C^2$ population-risk map gives
\begin{align}
 \left.\frac{\dd}{\dd\alpha}
 P_0\ell^{\rm ev}_{\lambda,Z}
   (\theta_t^{P+\alpha h,\lambda})\right|_{\alpha=0}
 &=\langle P_0\nabla\ell^{\rm ev}_{\lambda,Z}
      (\theta_t^{P,\lambda}),u_t^\lambda[h]\rangle,
 \label{eq:risk-first-response}\\
 \left.\frac{\dd^2}{\dd\alpha^2}
 P_0\ell^{\rm ev}_{\lambda,Z}
   (\theta_t^{P+\alpha h,\lambda})\right|_{\alpha=0}
 &=P_0\nabla^2\ell^{\rm ev}_{\lambda,Z}(\theta_t^{P,\lambda})
   [u_t^\lambda[h],u_t^\lambda[h]]\notag\\
 &\quad+\langle P_0\nabla\ell^{\rm ev}_{\lambda,Z}
      (\theta_t^{P,\lambda}),
   u_t^{(2),\lambda}[h,h]\rangle.
 \label{eq:risk-second-response}
\end{align}

At $P=P_n$, the sample-dependent real linear deletion-direction subspace from
\cref{sec:second-order-deletion-risk} is
$\mathcal V_S=\operatorname{span}\{h_1,\ldots,h_n\}\subseteq\mathcal M_0$, where
$h_i=P_n-\delta_{Z_i}$ as in \cref{eq:delete-measure}.
Then, for every $g\in\mathcal V_S$, let $u_\cdot^\lambda[g]:[0,T]\to\Theta$
denote the unique solution of \eqref{eq:first-data-response} with $h=g$.
The forcing $g_{g,t}^\lambda$ is linear in $g$, while
$A_t^\lambda$ is independent of $g$; uniqueness therefore makes
$g\mapsto u_t^\lambda[g]:\mathcal V_S\to\Theta$ linear for each fixed
$(t,\lambda)$.  For each deletion direction $h_i$, the
derivative of the actual chord
$\alpha\mapsto\theta_t^{P_n+\alpha h_i,\lambda}$ solves
\eqref{eq:first-data-response} with $h=h_i$, and hence equals
$u_t^\lambda[h_i]$.  By
\eqref{eq:risk-first-response},
\[
 \psi_{i,t,\lambda}'(0)
 =\left\langle
   P_0\nabla\ell^{\rm ev}_{\lambda,Z}(\theta_t^\lambda),
   u_t^\lambda[h_i]\right\rangle.
\]
The map from $\mathcal V_S$ to $\R$ given by
$g\mapsto\langle P_0\nabla\ell^{\rm ev}_{\lambda,Z}(\theta_t^\lambda),
u_t^\lambda[g]\rangle$ is linear, agrees with
$\psi_{i,t,\lambda}'(0)$ at every $h_i$, and supplies the linear-functional
sufficient condition for \eqref{eq:average-first-order-cancellation}.

For $\kappa\geq0$, define the nonnegative response-gain function
$\chi_\kappa:[0,\infty)\to[0,\infty)$ by
\begin{equation}\label{eq:chi-definition}
 \chi_\kappa(t)=
 \begin{cases}
 (\kappa t e^{\kappa t}-e^{\kappa t}+1)/\kappa^2,&\kappa>0,\\
 t^2/2,&\kappa=0.
 \end{cases}
\end{equation}
\begin{corollary}[A sufficient deletion-chord curvature bound]\label[corollary]{cor:cpp-bound}
Fix $\kappa\in[0,\infty)$.
For every $i\in[n]$, $0\leq\alpha_0\leq(n-1)^{-1}$, and
$P=P_n+\alpha_0 h_i\in\mathfrak P_S$ with $h_i\ne0$, define the
sample-dependent realized normalized deletion direction
$\widetilde h_i:=h_i/\|h_i\|_{\TV}\in\mathcal M_0$.  For every $\lambda\in\Lambda$ and
$t\in[0,T]$, suppose there is an interval containing zero on which the flows
trained under $P+\alpha\widetilde h_i$ exist, the map
$\alpha\mapsto\theta_t^{P+\alpha\widetilde h_i,\lambda}$ is twice
continuously differentiable, and the first and second $\alpha$-derivatives of
$\theta_t^{P+\alpha\widetilde h_i,\lambda}$ commute with the time derivative
and may be passed through the training and population-risk integrals.  Suppose
there are
$G_h,H_h,J_3,G_0,H_0\in[0,\infty)$, possibly depending on the realized sample
$S$, such that, uniformly over
$i\in[n]$ with $h_i\ne0$, $\alpha_0\in[0,(n-1)^{-1}]$, $\lambda\in\Lambda$, and
$t\in[0,T]$,
\[
\begin{gathered}
 \left\|\int\nabla\ell^{\rm tr}_{\lambda,z}(\theta_t^{P,\lambda})
          \,\dd\widetilde h_i(z)\right\|\leq G_h,\quad
 \left\|\int\nabla^2\ell^{\rm tr}_{\lambda,z}(\theta_t^{P,\lambda})
          \,\dd\widetilde h_i(z)\right\|_{\op}\leq H_h,\\
 \|\nabla_\theta^3F_{P,\lambda}(\theta_t^{P,\lambda})\|_{\op}\leq J_3,
 \quad \nabla^2F_{P,\lambda}(\theta_t^{P,\lambda})\succeq-\kappa I,\\
 \|P_0\nabla\ell^{\rm ev}_{\lambda,Z}(\theta_t^{P,\lambda})\|
 \leq G_0,\quad
 \|P_0\nabla^2\ell^{\rm ev}_{\lambda,Z}(\theta_t^{P,\lambda})\|_{\op}
 \leq H_0.
\end{gathered}
\]
The averaged first-order cancellation condition
\eqref{eq:average-first-order-cancellation} holds, and the curvature condition
\eqref{eq:measure-curvature-bound} holds with
\begin{equation}\label{eq:cpp-explicit}
 C_{PP}=H_0G_h^2\phi_\kappa(T)^2
 +G_0\{2H_hG_h\chi_\kappa(T)+J_3G_h^2\mathcal J_\kappa(T)\}.
\end{equation}
If $h_i=0$, then $P_n+\alpha_0 h_i=P_n$ for every
$0\leq\alpha_0\leq(n-1)^{-1}$.  Consequently, the scalar chord map
$\alpha_0\mapsto\psi_{i,t,\lambda}(\alpha_0)$ in
\eqref{eq:deletion-chord-functional} is constant, so
$\psi_{i,t,\lambda}''(\alpha_0)=0$.
\end{corollary}

\begin{proof}[Proof of \cref{cor:cpp-bound}]
For every $(t,\lambda)\in\cH$, the forcing
$g\mapsto g_{g,t}^\lambda$ is linear on $\mathcal V_S$, so uniqueness in
\eqref{eq:first-data-response} makes
$g\mapsto u_t^\lambda[g]$ linear on $\mathcal V_S$.  Equations
\eqref{eq:risk-first-response} and \eqref{eq:delete-measure} then give
\[
 \frac1n\sum_{i=1}^n\psi_{i,t,\lambda}'(0)
 =\left\langle
   P_0\nabla\ell^{\rm ev}_{\lambda,Z}(\theta_t^\lambda),
   u_t^\lambda\!\left[\frac1n\sum_{i=1}^nh_i\right]
  \right\rangle
 =0,
\]
which proves \eqref{eq:average-first-order-cancellation}.

Fix $i\in[n]$, $\alpha_0\in[0,(n-1)^{-1}]$, $\lambda\in\Lambda$, and
$0\leq t\leq T$.  If $h_i=0$, then the map $\psi_{i,t,\lambda}$ in
\eqref{eq:deletion-chord-functional} is constant on
$[0,(n-1)^{-1}]$, so $\psi_{i,t,\lambda}''(\alpha_0)=0$.  Suppose $h_i\ne0$
and set
$P=P_n+\alpha_0h_i$ and $\widetilde h_i=h_i/\|h_i\|_{\TV}$.
For $0\leq r\leq s\leq t$, let
$U^\lambda(s,r):\Theta\to\Theta$ solve
\[
 \partial_sU^\lambda(s,r)=-A_s^\lambda U^\lambda(s,r),
 \qquad U^\lambda(r,r)=I.
\]
For every $v\in\Theta$, the lower Hessian bound
$A_s^\lambda\succeq-\kappa I$ gives
\[
 \frac12\frac{\dd}{\dd s}\|U^\lambda(s,r)v\|^2
 =-\langle U^\lambda(s,r)v,
          A_s^\lambda U^\lambda(s,r)v\rangle
 \leq\kappa\|U^\lambda(s,r)v\|^2.
\]
Gronwall's inequality and $U^\lambda(r,r)=I$ yield
$\|U^\lambda(s,r)\|_{\op}\leq e^{\kappa(s-r)}$.

At the base measure $P$ and in the direction $\widetilde h_i$, the common
flow initialization is independent of the perturbation parameter, so the
initial values in \cref{eq:first-data-response,eq:second-data-response} are
zero.  Integrating \cref{eq:first-data-response,eq:second-data-response}
against $U^\lambda(s,r)$ from time $0$ to $s$ gives
\[
\begin{aligned}
 u_s^\lambda[\widetilde h_i]
 &=-\int_0^sU^\lambda(s,r)g_{\widetilde h_i,r}^\lambda\,\dd r,\\
 u_s^{(2),\lambda}[\widetilde h_i,\widetilde h_i]
 &=-\int_0^sU^\lambda(s,r)
   \Bigl\{2H_{\widetilde h_i,r}^\lambda
              u_r^\lambda[\widetilde h_i]
   +\mathsf T_r^\lambda
      [u_r^\lambda[\widetilde h_i],u_r^\lambda[\widetilde h_i]]
   \Bigr\}\,\dd r.
\end{aligned}
\]
For $0\leq r\leq t$, the normalized direction satisfies
\[
 \|g_{\widetilde h_i,r}^\lambda\|\leq G_h,
 \qquad
 \|H_{\widetilde h_i,r}^\lambda\|_{\op}\leq H_h,
 \qquad
 \|\mathsf T_r^\lambda[v,v]\|\leq J_3\|v\|^2,
 \qquad v\in\Theta.
\]
Consequently, for $0\leq s\leq t$,
\[
 \|u_s^\lambda[\widetilde h_i]\|
 \leq G_h\int_0^s e^{\kappa(s-r)}\,\dd r
 =G_h\phi_\kappa(s).
\]
The integral formula for
$u_t^{(2),\lambda}[\widetilde h_i,\widetilde h_i]$,
$\|u_r^\lambda[\widetilde h_i]\|\leq G_h\phi_\kappa(r)$, and
\eqref{eq:chi-definition} and \eqref{eq:phi-J-def} give
\[
\begin{aligned}
 \|u_t^{(2),\lambda}[\widetilde h_i,\widetilde h_i]\|
 &\leq2H_hG_h\int_0^t
       e^{\kappa(t-r)}\phi_\kappa(r)\,\dd r\\
 &\quad+J_3G_h^2\int_0^t
       e^{\kappa(t-r)}\phi_\kappa(r)^2\,\dd r\\
 &=2H_hG_h\chi_\kappa(t)+J_3G_h^2\mathcal J_\kappa(t).
\end{aligned}
\]
For $t\geq0$,
$\phi_\kappa'(t)=e^{\kappa t}$,
$\chi_\kappa'(t)=\phi_\kappa(t)+\kappa\chi_\kappa(t)$, and
$\mathcal J_\kappa'(t)=\phi_\kappa(t)^2+\kappa\mathcal J_\kappa(t)$ are
nonnegative, so the gains are nondecreasing.  Applying
\eqref{eq:risk-second-response} gives
\[
\begin{aligned}
 &\left|\left.\frac{\dd^2}{\dd\alpha^2}
 P_0\ell^{\rm ev}_{\lambda,Z}
   (\theta_t^{P+\alpha\widetilde h_i,\lambda})
 \right|_{\alpha=0}\right|\\
 &\quad\leq
 H_0\|u_t^\lambda[\widetilde h_i]\|^2
 +G_0\|u_t^{(2),\lambda}[\widetilde h_i,\widetilde h_i]\|\\
 &\quad\leq H_0G_h^2\phi_\kappa(T)^2
 +G_0\{2H_hG_h\chi_\kappa(T)
          +J_3G_h^2\mathcal J_\kappa(T)\}
 =C_{PP}.
\end{aligned}
\]

For every $\eta$ such that
$\alpha_0+\eta\in[0,(n-1)^{-1}]$,
\[
 \psi_{i,t,\lambda}(\alpha_0+\eta)
 =P_0\ell^{\rm ev}_{\lambda,Z}
   \bigl(\theta_t^{P+\eta\|h_i\|_{\TV}\widetilde h_i,\lambda}\bigr).
\]
Twice differentiating at $\eta=0$, with one-sided derivatives when
$\alpha_0$ is a chord endpoint, gives
\[
\begin{aligned}
 |\psi_{i,t,\lambda}''(\alpha_0)|
 &=\|h_i\|_{\TV}^2
 \left|\left.\frac{\dd^2}{\dd\alpha^2}
 P_0\ell^{\rm ev}_{\lambda,Z}
   (\theta_t^{P+\alpha\widetilde h_i,\lambda})
 \right|_{\alpha=0}\right|\\
 &\leq C_{PP}\|h_i\|_{\TV}^2.
\end{aligned}
\]
Taking the supremum over $i$, $\alpha_0$, $t$, and $\lambda$ proves
\eqref{eq:measure-curvature-bound}.
\end{proof}

\subsubsection{Exact-LOO concentration}

\begin{proof}[Proof of \cref{thm:loo-concentration}]
Fix $r>0$ with $\mathcal N(\mathcal U,\rho,r)<\infty$ and
$\delta\in(0,1)$.  Define the sample-dependent real-valued process
$D_{S,\mathcal U}:\mathcal U\to\R$ by
$D_{S,\mathcal U}(u)=\Score_{n,\mathcal U}^{\LOO}(u)
-\Risk_{S,\mathcal U}^-(u)$, and fix
$u\in\mathcal U$.
Fix $j\in[n]$ and an arbitrary deterministic $z_j'\in\mathsf Z$.  Replace
$Z_j$ by $z_j'$ and let $S':\Omega\to\mathsf Z^n$ be the resulting
measurable random sample.
Because $S_{-j}'=S_{-j}$, the difference between the $i=j$ LOO contributions
for $S$ and $S'$ satisfies
$n^{-1}|L_{n-1,u}(S_{-j},Z_j)-L_{n-1,u}(S_{-j},z_j')|
\leq B_{n-1}^{\rm rng}/n$.
Each of the other $n-1$ LOO-score summands changes by at most
$\beta_{n-1}/n$.  The corresponding population-risk term
$n^{-1}P_0L_{n-1,u}(S_{-j},Z)$ is unchanged, and each of the other
population-risk summands changes by at most $\beta_{n-1}/n$.
Thus, since $j\in[n]$ and $z_j'\in\mathsf Z$ were arbitrary,
\[
 \left|D_{S,\mathcal U}(u)-D_{S',\mathcal U}(u)\right|
 \leq \frac{B_{n-1}^{\rm rng}}n
      +2(n-1)\frac{\beta_{n-1}}n
 =c_n^-.
\]

Moreover, $\E D_{S,\mathcal U}(u)=0$: conditional on $S_{-i}$, the omitted observation
$Z_i$ has law $P_0$.  If $c_n^-=0$, every coordinate oscillation is zero, so
$D_{S,\mathcal U}(u)$ is almost surely constant and the centering gives
$D_{S,\mathcal U}(u)=0$ almost
surely.  For $c_n^->0$ and $x>0$, McDiarmid's bounded-differences inequality
\cite[Lemma~1.2]{mcdiarmid1989bounded} gives
\begin{equation}\label{eq:loo-pointwise-tail}
 \Pp\{|D_{S,\mathcal U}(u)|>x\}
 \leq2\exp\left\{-\frac{2x^2}{n(c_n^-)^2}\right\}.
\end{equation}
Let $\mathcal T_r\subset\mathcal U$ be a deterministic finite $r$-net with
$|\mathcal T_r|\leq\mathcal N(\mathcal U,\rho,r)$.  For any
$u\in\mathcal U$, choose $v\in\mathcal T_r$ with $\rho(u,v)\leq r$.
The modulus in \eqref{eq:hyper-modulus} gives
\begin{equation}\label{eq:loo-net-score-risk-modulus}
\begin{aligned}
 &\left|\Score_{n,\mathcal U}^{\LOO}(u)
             -\Score_{n,\mathcal U}^{\LOO}(v)\right|
 \leq\frac1n\sum_{i=1}^n
 \left|L_{n-1,u}(S_{-i},Z_i)-L_{n-1,v}(S_{-i},Z_i)\right|
 \leq\omega_{n-1}(r),\\
 &\left|\Risk_{S,\mathcal U}^-(u)-\Risk_{S,\mathcal U}^-(v)\right|
 \leq\frac1n\sum_{i=1}^nP_0
 \left|L_{n-1,u}(S_{-i},Z)-L_{n-1,v}(S_{-i},Z)\right|
 \leq\omega_{n-1}(r).
\end{aligned}
\end{equation}
The triangle inequality and \eqref{eq:loo-net-score-risk-modulus} imply
\begin{equation}\label{eq:loo-net-process-modulus}
 |D_{S,\mathcal U}(u)-D_{S,\mathcal U}(v)|
 \leq
 \left|\Score_{n,\mathcal U}^{\LOO}(u)
             -\Score_{n,\mathcal U}^{\LOO}(v)\right|
 +\left|\Risk_{S,\mathcal U}^-(u)-\Risk_{S,\mathcal U}^-(v)\right|
 \leq2\omega_{n-1}(r).
\end{equation}
Set
\begin{equation}\label{eq:loo-net-threshold}
 q_{n,r,\delta}
 =c_n^-\left\{\frac n2
   \log\frac{2\mathcal N(\mathcal U,\rho,r)}\delta\right\}^{1/2}.
\end{equation}
Thus $q_{n,r,\delta}\in[0,\infty)$ is a finite deterministic net threshold.
If $c_n^->0$, apply \eqref{eq:loo-pointwise-tail} with
$x=q_{n,r,\delta}$, sum over $v\in\mathcal T_r$, and use
$|\mathcal T_r|\leq\mathcal N(\mathcal U,\rho,r)$ to obtain
\begin{equation}\label{eq:loo-net-event-probability}
 \Pp\left\{\max_{v\in\mathcal T_r}|D_{S,\mathcal U}(v)|
   >q_{n,r,\delta}\right\}
 \leq2\mathcal N(\mathcal U,\rho,r)
 \exp\left\{-\frac{2q_{n,r,\delta}^2}{n(c_n^-)^2}\right\}
 =\delta.
\end{equation}
If $c_n^-=0$, every $D_{S,\mathcal U}(v)$ for $v\in\mathcal T_r$ is zero almost surely,
and $q_{n,r,\delta}=0$ by \eqref{eq:loo-net-threshold}, so the probability in
\eqref{eq:loo-net-event-probability} is zero.  Thus
the event
$\{\max_{v\in\mathcal T_r}|D_{S,\mathcal U}(v)|
\leq q_{n,r,\delta}\}$ has probability at
least $1-\delta$ in both cases.  On the event
$\{\max_{v\in\mathcal T_r}|D_{S,\mathcal U}(v)|
\leq q_{n,r,\delta}\}$,
\eqref{eq:loo-net-process-modulus} gives
\begin{equation}\label{eq:loo-net-supremum-bound}
 \sup_{u\in\mathcal U}|D_{S,\mathcal U}(u)|
 \leq q_{n,r,\delta}+2\omega_{n-1}(r)=s_n(r,\delta),
\end{equation}
where the equality in \eqref{eq:loo-net-supremum-bound} uses
\eqref{eq:sn-definition}.
Moreover, \eqref{eq:hyper-modulus} and
$\omega_{n-1}(s)\downarrow0$ as $s\downarrow0$
make both score and deletion-risk sample paths continuous in $u$.  Fix a
deterministic countable dense subset $\mathcal U_0\subset\mathcal U$.  The joint
measurability assumed in \cref{thm:loo-concentration} makes
$D_{S,\mathcal U}(v)$ a
measurable function of $S$ for every $v\in\mathcal U_0$, while continuity
gives
$\sup_{u\in\mathcal U}|D_{S,\mathcal U}(u)|
=\sup_{v\in\mathcal U_0}|D_{S,\mathcal U}(v)|$.  The map
$S\mapsto\sup_{u\in\mathcal U}|D_{S,\mathcal U}(u)|$ is therefore measurable,
so the event in \eqref{eq:loo-concentration-bound} is measurable.  The
probability bound in \eqref{eq:loo-net-event-probability} for $c_n^->0$, the
almost-sure equalities $D_{S,\mathcal U}(v)=0$ for $v\in\mathcal T_r$ when
$c_n^-=0$, and
\eqref{eq:loo-net-supremum-bound} give
\[
 1-\delta
 \leq\Pp\left\{\max_{v\in\mathcal T_r}|D_{S,\mathcal U}(v)|
          \leq q_{n,r,\delta}\right\}
 \leq\Pp\left\{\sup_{u\in\mathcal U}|D_{S,\mathcal U}(u)|
          \leq s_n(r,\delta)\right\},
\]
which proves \eqref{eq:loo-concentration-bound}.
\end{proof}

\begin{proof}[Proof of \cref{thm:chained-loo-concentration}]
For $p=(t,\lambda)\in\cH$, write
$L_{n-1,p}=L_{n-1,(t,\lambda)}$.  For $\mathsf s_n\in\mathsf Z^n$, let
$\mathfrak D_p(\mathsf s_n)$ be the expression in
\cref{eq:relative-loo-process} with $S=\mathsf s_n$.  The resulting
deterministic measurable map is
$\mathfrak D_p:(\mathsf Z^n,\mathcal Z^{\otimes n})
\to(\R,\mathcal B(\R))$.  Put
$X_p:=\mathfrak D_p(S):\Omega\to\R$ for the resulting measurable random
variable.  For $p,q\in\cH$, define the deterministic
evaluated-loss increment map
\[
\begin{aligned}
 L_{p,q}^{\rm inc}:(\mathsf Z^{n-1}\times\mathsf Z,
 \mathcal Z^{\otimes(n-1)}\otimes\mathcal Z)
 &\longrightarrow(\R,\mathcal B(\R)),\\
 L_{p,q}^{\rm inc}(\mathsf a,z)
 &=L_{n-1,p}(\mathsf a,z)-L_{n-1,q}(\mathsf a,z),
 \quad (\mathsf a,z)\in\mathsf Z^{n-1}\times\mathsf Z.
\end{aligned}
\]
$L_{p,q}^{\rm inc}$ is measurable.
The LOO contrast between $p$ and $q$ consists of an evaluation term and a
population-centering term:
\[
 X_p-X_q
 =\frac1n\sum_{i=1}^n\left[
   \underbrace{L_{p,q}^{\rm inc}(S_{-i},Z_i)}_{\text{LOO evaluation term}}
   -\underbrace{P_0L_{p,q}^{\rm inc}(S_{-i},Z)}_
        {\text{population-centering term}}
 \right].
\]
For every $\mathsf a\in\mathsf Z^{n-1}$ and $z\in\mathsf Z$,
\cref{eq:hyper-lipschitz-constant} gives
\[
 |L_{p,q}^{\rm inc}(\mathsf a,z)|\leq L_{n-1}^{\rm hyp}d_{\cH}(p,q).
\]
Conditional independence of $Z_i$ and $S_{-i}$ then gives
\[
 \E\{L_{p,q}^{\rm inc}(S_{-i},Z_i)\mid S_{-i}\}
 =P_0L_{p,q}^{\rm inc}(S_{-i},Z),
 \qquad \E(X_p-X_q)=0.
\]

Fix $j\in[n]$ and an arbitrary deterministic replacement $z_j'\in\mathsf Z$,
and let
\[
 S^{(j)}=(Z_1,\ldots,Z_{j-1},z_j',Z_{j+1},\ldots,Z_n):
 \Omega\to\mathsf Z^n.
\]
$S^{(j)}$ is a measurable random element under
$\mathcal Z^{\otimes n}$.  Let
$S_{-i}^{(j)}:\Omega\to\mathsf Z^{n-1}$ be the measurable random element
obtained by deleting observation $i$ from $S^{(j)}$.  Since
$S_{-j}^{(j)}=S_{-j}$, the changes of the $i=j$ LOO
evaluation term and the $i=j$ population-centering term are, respectively,
\begin{align*}
 &\frac1n\left|
 L_{p,q}^{\rm inc}(S_{-j},Z_j)
 -L_{p,q}^{\rm inc}(S_{-j}^{(j)},z_j')\right|
 =\frac1n\left|
 L_{p,q}^{\rm inc}(S_{-j},Z_j)-L_{p,q}^{\rm inc}(S_{-j},z_j')\right|
 \leq \frac{2L_{n-1}^{\rm hyp}}n d_{\cH}(p,q),\\
 &\frac1n\left|P_0\left\{
 L_{p,q}^{\rm inc}(S_{-j},Z)-L_{p,q}^{\rm inc}(S_{-j}^{(j)},Z)
 \right\}\right|=0.
\end{align*}
For $i,j\in[n]$ with $i\ne j$, the samples $S_{-i}$ and
$S_{-i}^{(j)}$ differ in exactly one
coordinate.  Hence \cref{eq:mixed-stability} gives the two separate bounds
\begin{align*}
 &\frac1n\left|
 L_{p,q}^{\rm inc}(S_{-i},Z_i)
 -L_{p,q}^{\rm inc}(S_{-i}^{(j)},Z_i)\right|
 \leq\frac{\beta_{n-1}^{\rm mix}}n d_{\cH}(p,q),\\
 &\frac1n\left|P_0\left\{
 L_{p,q}^{\rm inc}(S_{-i},Z)
 -L_{p,q}^{\rm inc}(S_{-i}^{(j)},Z)\right\}\right|
 \leq\frac{\beta_{n-1}^{\rm mix}}n d_{\cH}(p,q).
\end{align*}
For every $j\in[n]$ and every deterministic replacement $z_j'\in\mathsf Z$,
the coordinate oscillation is
\begin{align}
 &\left|\{\mathfrak D_p(S)-\mathfrak D_q(S)\}
 -\{\mathfrak D_p(S^{(j)})-\mathfrak D_q(S^{(j)})\}\right|
 \label{eq:chained-loo-coordinate-oscillation}\\
 &\quad\leq\frac1n\left|
 L_{p,q}^{\rm inc}(S_{-j},Z_j)
 -L_{p,q}^{\rm inc}(S_{-j},z_j')\right|\notag\\
 &\qquad
 +\sum_{i\in[n]\setminus\{j\}}\frac1n\left|
 L_{p,q}^{\rm inc}(S_{-i},Z_i)
 -L_{p,q}^{\rm inc}(S_{-i}^{(j)},Z_i)\right|\notag\\
 &\qquad
 +\sum_{i\in[n]\setminus\{j\}}\frac1n\left|P_0\left\{
 L_{p,q}^{\rm inc}(S_{-i},Z)
 -L_{p,q}^{\rm inc}(S_{-i}^{(j)},Z)\right\}\right|\notag\\
 &\quad\leq
 \left\{\frac{2L_{n-1}^{\rm hyp}}n
       +\sum_{i\in[n]\setminus\{j\}}\frac{2\beta_{n-1}^{\rm mix}}n\right\}d_{\cH}(p,q)
 \notag\\
 &\quad=\frac{2\{L_{n-1}^{\rm hyp}+(n-1)\beta_{n-1}^{\rm mix}\}}n
 d_{\cH}(p,q)
 =\frac{2a_n^{\rm mix}}n d_{\cH}(p,q).\notag
\end{align}
McDiarmid's bounded-differences moment-generating-function argument
\cite[Lemma~5.8 and the proof of Theorem~6.7]{mcdiarmid1989bounded}, applied
also to the negative increment, now gives, for every $\xi\in\R$,
\begin{equation}\label{eq:chained-loo-subgaussian-increments}
\begin{aligned}
 \log\E\exp\{\xi(X_p-X_q)\}
 &\leq\frac{\xi^2}{8}\sum_{j=1}^n
 \left\{\frac{2a_n^{\rm mix}}n d_{\cH}(p,q)\right\}^2\\
 &=\frac{\xi^2(a_n^{\rm mix})^2
 d_{\cH}(p,q)^2}{2n}.
\end{aligned}
\end{equation}

We now give the finite-level chaining calculation.  If $D_{\cH}=0$, then
$\cH=\{(t_0,\lambda_0)\}$ and
$\sup_{p\in\cH}|X_p|=|X_{(t_0,\lambda_0)}|=0$.  If
$a_n^{\rm mix}=0$, then
\cref{eq:mixed-increment-constant} gives $L_{n-1}^{\rm hyp}=0$, so
$L_{p,q}^{\rm inc}(\mathsf a,z)=0$ by \cref{eq:hyper-lipschitz-constant} and hence
$\sup_{p\in\cH}|X_p|=0$.  Hence assume $D_{\cH}>0$ and
$a_n^{\rm mix}>0$, and set
$\epsilon_k=2^{-k}D_{\cH}$ for $k\in\{0,1,2,\ldots\}$.  Each $\epsilon_k$ is
a finite positive deterministic mesh radius.  Put the deterministic set
$\mathcal T_0=\{(t_0,\lambda_0)\}$ and define
$\Pi_0:\cH\to\mathcal T_0$ by $\Pi_0p=(t_0,\lambda_0)$.  The diameter gives
\[
 d_{\cH}(p,\Pi_0p)\leq D_{\cH}=\epsilon_0,\qquad p\in\cH.
\]
For $k\geq1$, choose a deterministic finite $\epsilon_k$-net $\mathcal T_k$ with
$|\mathcal T_k|\leq \mathcal N(\cH,d_{\cH},\epsilon_k)$ and a deterministic map
$\Pi_k:\cH\to\mathcal T_k$ satisfying
\[
 d_{\cH}(p,\Pi_kp)=\min_{v\in\mathcal T_k}d_{\cH}(p,v)\leq\epsilon_k,
 \qquad p\in\cH.
\]
The level-$k$ projection-link set is
\[
 \mathcal L_k^{\rm proj}
 :=\{(\Pi_kp,\Pi_{k-1}p):p\in\cH\}
 \subseteq\mathcal T_k\times\mathcal T_{k-1}.
\]
Thus $\mathcal L_k^{\rm proj}$ is a finite deterministic set for every
$k\geq1$.
Consequently, for every $p\in\cH$, the link lengths and the number of
distinct links satisfy
\begin{align*}
 d_{\cH}(\Pi_kp,\Pi_{k-1}p)
 &\leq d_{\cH}(\Pi_kp,p)+d_{\cH}(p,\Pi_{k-1}p)
 \leq\epsilon_k+\epsilon_{k-1}
 =\frac32\epsilon_{k-1},\\
 |\mathcal L_k^{\rm proj}|&\leq|\mathcal T_k||\mathcal T_{k-1}|.
\end{align*}

For $u,v\in\cH$ with $d_{\cH}(u,v)>0$, $r>0$, and $\xi>0$, Markov's inequality
applied to $\exp\{\xi(X_u-X_v)\}$ and
\cref{eq:chained-loo-subgaussian-increments} gives the exponential, or
Chernoff, bound
\[
 \Pp\{X_u-X_v>r\}
 \leq \exp\left\{-\xi r+
 \frac{\xi^2(a_n^{\rm mix})^2d_{\cH}(u,v)^2}{2n}\right\}.
\]
The exponent is minimized at
$\xi=nr/\{(a_n^{\rm mix})^2d_{\cH}(u,v)^2\}$.  Interchanging $u$ and $v$ and using
$d_{\cH}(v,u)=d_{\cH}(u,v)$ therefore yields
\[
 \max\bigl\{\Pp\{X_u-X_v>r\},\Pp\{X_v-X_u>r\}\bigr\}
 \leq\exp\left\{-\frac{nr^2}
 {2(a_n^{\rm mix})^2d_{\cH}(u,v)^2}\right\}.
\]
Since
$\{|X_u-X_v|>r\}=\{X_u-X_v>r\}\cup\{X_v-X_u>r\}$, a union bound gives
\[
 \Pp\{|X_u-X_v|>r\}
 \leq2\exp\left\{-\frac{nr^2}
 {2(a_n^{\rm mix})^2d_{\cH}(u,v)^2}\right\}.
\]
For $d_{\cH}(u,v)=0$, the metric property instead gives
\[
 d_{\cH}(u,v)=0\quad\Longrightarrow\quad u=v
 \quad\Longrightarrow\quad X_u-X_v=0.
\]
For $k\geq1$, set the finite positive deterministic scalars
$\delta_k=2^{-k}\delta$ and
\[
 b_k:=\frac{3a_n^{\rm mix}\epsilon_{k-1}}{2\sqrt n}
 \sqrt{2\log\frac{2|\mathcal T_k||\mathcal T_{k-1}|}{\delta_k}}.
\]
The threshold $b_k$ depends only on the deterministic mesh, covering numbers,
$n$, and $\delta$.
For each $(u,v)\in\mathcal L_k^{\rm proj}$ with $d_{\cH}(u,v)>0$, the inequality
$d_{\cH}(u,v)\leq3\epsilon_{k-1}/2$ yields
\begin{align*}
 \Pp\{|X_u-X_v|>b_k\}
 &\leq2\exp\left\{-\frac{nb_k^2}
 {2(a_n^{\rm mix})^2d_{\cH}(u,v)^2}\right\}\\
 &\leq2\exp\left\{-\frac{nb_k^2}
 {2(a_n^{\rm mix})^2(3\epsilon_{k-1}/2)^2}\right\}
 =\frac{\delta_k}{|\mathcal T_k||\mathcal T_{k-1}|}.
\end{align*}
Zero-length links have zero increment almost surely.  Thus the union bound at
level $k$ gives
\[
 \Pp\left\{\max_{(u,v)\in\mathcal L_k^{\rm proj}}
 |X_u-X_v|>b_k\right\}
 \leq\sum_{(u,v)\in\mathcal L_k^{\rm proj}}
 \Pp\{|X_u-X_v|>b_k\}
 \leq\delta_k.
\]
Define the sample-dependent simultaneous link event
\[
 \Omega_\delta^{\rm ch}
 :=\bigcap_{k\geq1}
 \left\{\max_{(u,v)\in\mathcal L_k^{\rm proj}}|X_u-X_v|\leq b_k\right\}.
\]
$\Omega_\delta^{\rm ch}$ belongs to $\mathcal F$ because the nets are deterministic and the
finite-level increments are measurable.
A second union bound gives
\[
 \Pp(\Omega_\delta^{\rm ch})
 \geq1-\sum_{k\geq1}\delta_k=1-\delta.
\]
On $\Omega_\delta^{\rm ch}$, the projection increments satisfy
\[
 |X_{\Pi_kp}-X_{\Pi_{k-1}p}|\leq b_k,
 \qquad p\in\cH,\quad k\geq1.
\]

Since the covering numbers increase as the mesh shrinks,
\[
 \log(|\mathcal T_k||\mathcal T_{k-1}|)
 \leq2\log \mathcal N(\cH,d_{\cH},\epsilon_k).
\]
Also,
\[
 \sum_{k\geq1}\epsilon_{k-1}
       \sqrt{\log \mathcal N(\cH,d_{\cH},\epsilon_k)}
 \leq4\int_0^{D_{\cH}}\sqrt{\log \mathcal N(\cH,d_{\cH},u)}\,\dd u,
\]
because the integral on $[\epsilon_{k+1},\epsilon_k]$ is at least
$(\epsilon_k/2)\sqrt{\log \mathcal N(\cH,d_{\cH},\epsilon_k)}$.  The remaining
dyadic term satisfies
\[
 \sum_{k\geq1}\epsilon_{k-1}
 \sqrt{\log(2^{k+1}/\delta)}
 \leq C_{\rm ch}D_{\cH}\sqrt{\log(4/\delta)}
\]
after enlarging, if necessary, the universal constant $C_{\rm ch}$ in
\cref{thm:chained-loo-concentration}.  The definition of $b_k$ and
$\sqrt{x+y}\leq\sqrt{x}+\sqrt{y}$ therefore imply
\[
 \sum_{k\geq1}b_k
 \leq\frac{C_{\rm ch}a_n^{\rm mix}}{\sqrt n}
 \left\{\mathfrak E(\cH,d_{\cH})
 +D_{\cH}\sqrt{\log\frac4\delta}\right\}.
\]
For every deterministic integer $K_{\rm ch}\geq1$, on
$\Omega_\delta^{\rm ch}$,
$X_{\Pi_0p}=X_{(t_0,\lambda_0)}=0$ and the telescoping identity gives,
simultaneously for $p\in\cH$,
\[
 |X_{\Pi_{K_{\rm ch}}p}|
 =\left|\sum_{k=1}^{K_{\rm ch}}
 \{X_{\Pi_kp}-X_{\Pi_{k-1}p}\}\right|
 \leq\sum_{k=1}^{K_{\rm ch}} b_k.
\]

For every $p,q\in\cH$,
\cref{eq:hyper-lipschitz-constant} gives the sample-path bound
\[
 |X_p-X_q|
 \leq\frac1n\sum_{i=1}^n\left\{
 |L_{p,q}^{\rm inc}(S_{-i},Z_i)|
 +P_0|L_{p,q}^{\rm inc}(S_{-i},Z)|\right\}
 \leq2L_{n-1}^{\rm hyp}d_{\cH}(p,q).
\]
Thus $p\mapsto X_p$ is $2L_{n-1}^{\rm hyp}$-Lipschitz and hence continuous
for every sample.  The Lipschitz bound and
$d_{\cH}(p,\Pi_{K_{\rm ch}}p)\leq\epsilon_{K_{\rm ch}}\to0$ give
$X_{\Pi_{K_{\rm ch}}p}\to X_p$ for every $p\in\cH$.
On $\Omega_\delta^{\rm ch}$, letting $K_{\rm ch}\to\infty$ in
$|X_{\Pi_{K_{\rm ch}}p}|\leq\sum_{k=1}^{K_{\rm ch}}b_k$, and then taking the
supremum over $p\in\cH$, gives
\[
 \sup_{p\in\cH}|X_p|
 \leq\sum_{k\geq1}b_k
 \leq s_n^{\rm ch}(\delta).
\]
Since every $\mathcal T_k$ is finite and
$d_{\cH}(p,\Pi_kp)\leq\epsilon_k\downarrow0$, the deterministic set
$\bigcup_{k\geq0}\mathcal T_k$ is countable and dense in $\cH$.  Continuity
therefore gives, samplewise,
$\sup_{p\in\cH}|X_p|
=\sup_{p\in\bigcup_{k\geq0}\mathcal T_k}|X_p|$.
For every fixed $p\in\bigcup_{k\geq0}\mathcal T_k$,
$X_p=\mathfrak D_p(S)$ is measurable, so
$\sup_{p\in\bigcup_{k\geq0}\mathcal T_k}|X_p|$ is a countable supremum of
measurable random variables.  Hence
$\sup_{p\in\cH}|X_p|$ is measurable.  Thus the measurable event
$\{\sup_{p\in\cH}|X_p|\leq s_n^{\rm ch}(\delta)\}$ contains
$\Omega_\delta^{\rm ch}$, and \eqref{eq:chained-loo-relative-bound} follows:
\[
 \Pp\left\{\sup_{p\in\cH}|X_p|\leq s_n^{\rm ch}(\delta)\right\}
 \geq\Pp(\Omega_\delta^{\rm ch})
 \geq1-\delta.
\]

\end{proof}

\subsection{Risk-curve recovery}

\begin{proof}[Proof of \cref{thm:oracle}]
For the absolute route, let $\mathcal C_{n,\delta}$ be a measurable event
supplied by \cref{thm:loo-concentration}, with
$\Pp(\mathcal C_{n,\delta})\geq1-\delta$.  On $\mathcal C_{n,\delta}$,
\[
 \sup_{(t,\lambda)\in\cH}
 \big|\Score_n^{\LOO}(t,\lambda)-\Risk_S^-(t,\lambda)-C_S\big|
 \leq s_n^\star(\delta),
 \qquad C_S=0.
\]
Under the chained route, the integrable anchor and
\cref{thm:chained-loo-concentration} give a measurable event
$\mathcal C_{n,\delta}$ with probability at least $1-\delta$ on which
\[
 \sup_{(t,\lambda)\in\cH}
 \big|\Score_n^{\LOO}(t,\lambda)-\Risk_S^-(t,\lambda)-C_S\big|
 \leq s_n^\star(\delta),
 \quad
 C_S=D_S(t_0,\lambda_0),\quad
 s_n^\star(\delta)=s_n^{\rm ch}(\delta).
\]

On $\mathcal E_n^{\rm reg}$, \cref{eq:score-approximation-bound} and
\eqref{eq:regularity-event-envelopes} give
\[
 \sup_{(t,\lambda)\in\cH}
 |\widetilde\Score_n(t,\lambda)-\Score_n^{\LOO}(t,\lambda)|
 \leq
 \frac{\overline G_{{\rm out},n}\overline L_{3,n}\overline M_n^2}
      {2(n-1)^2}\mathcal J_{\overline\kappa_n}(T).
\]
Here
\(
 \mathcal J_\kappa(T)=\int_0^T
 e^{\kappa(T-s)}\phi_\kappa(s)^2\,\dd s
\)
is nondecreasing in $\kappa\geq0$, because both factors in the integrand are
nondecreasing in $\kappa$.  On $\mathcal E_n^{\rm reg}$,
\cref{eq:cross-size-absolute,eq:regularity-event-envelopes} give
\[
 \sup_{(t,\lambda)\in\cH}
 |\Risk_S^-(t,\lambda)-\Risk_S(t,\lambda)|
 \leq\frac{2\overline C_{PP,n}}{(n-1)^2}.
\]
For every $(t,\lambda)\in\cH$, the triangle inequality gives, on
$\mathcal C_{n,\delta}\cap\mathcal E_n^{\rm reg}$,
\begin{align*}
 &|\widetilde\Score_n(t,\lambda)-\Risk_S(t,\lambda)-C_S|\\
 &\quad\leq
 |\widetilde\Score_n(t,\lambda)-\Score_n^{\LOO}(t,\lambda)|
 +|\Score_n^{\LOO}(t,\lambda)-\Risk_S^-(t,\lambda)-C_S|
 +|\Risk_S^-(t,\lambda)-\Risk_S(t,\lambda)|\\
 &\quad\leq\mathfrak r_n(\delta).
\end{align*}
Taking $\mathcal O_{n,\delta}
=\mathcal C_{n,\delta}\cap\mathcal E_n^{\rm reg}$ proves
\cref{eq:risk-curve-recovery}, since
$\Pp(\mathcal O_{n,\delta})\geq1-\delta-\tau_n$.
\end{proof}

\begin{proof}[Proof of \cref{cor:regularity-event-oracle}]
For the time-zero chained anchor,
\eqref{eq:flow-alo-response} gives $d_{i,0}^{\lambda_0}=0$, and the common
initialization in \eqref{eq:full-flow} gives
\[
 \Score_n^{\LOO}(0,\lambda_0)
 =\widetilde\Score_n(0,\lambda_0),
 \qquad
 \Risk_S^-(0,\lambda_0)=\Risk_S(0,\lambda_0).
\]
Hence
\[
 C_S=D_S(0,\lambda_0)
 =\widetilde\Score_n(0,\lambda_0)-\Risk_S(0,\lambda_0).
\]
Substituting
$C_S=\widetilde\Score_n(0,\lambda_0)-\Risk_S(0,\lambda_0)$ into
\eqref{eq:risk-curve-recovery} proves
\cref{eq:time-zero-centered-risk-curve}.

If \eqref{eq:approximate-score-minimizer} holds, then on
$\mathcal O_{n,\delta}$, for every $(t,\lambda)\in\cH$,
\[
 \Risk_S(\widehat t,\widehat\lambda)
 \leq\widetilde\Score_n(\widehat t,\widehat\lambda)-C_S+\mathfrak r_n(\delta)
 \leq\widetilde\Score_n(t,\lambda)-C_S+\mathfrak r_n(\delta)+\zeta_n
 \leq\Risk_S(t,\lambda)+2\mathfrak r_n(\delta)+\zeta_n.
\]
Taking the infimum over $(t,\lambda)\in\cH$ proves
\cref{eq:main-oracle-bound}.
\end{proof}

If $\zeta_n\equiv\zeta>0$ is deterministic and the map
$Q:\Omega\times\cH\to\mathbb R$ defined by
$Q(\xi,t,\lambda)=[\widetilde\Score_n(t,\lambda)](\xi)$ is
$\mathcal F\otimes\mathcal B(\cH)$-measurable and continuous in
$(t,\lambda)$ for every $\xi\in\Omega$, then an
$\mathcal F/\mathcal B(\cH)$-measurable selector satisfying
\eqref{eq:approximate-score-minimizer} exists.  The compact metric space
$\cH$ is separable, so fix a deterministic dense sequence
$p_k=(t_k,\lambda_k)$, $k\geq1$.  Samplewise
continuity gives, for every $\xi\in\Omega$,
\[
 m(\xi):=\inf_{(t,\lambda)\in\cH}Q(\xi,t,\lambda)
 =\inf_{k\geq1}Q(\xi,p_k).
\]
Each $Q(\mathord\cdot,p_k)$ is $\mathcal F$-measurable, so
$m=\inf_{k\geq1}Q(\mathord\cdot,p_k):\Omega\to\mathbb R$ is measurable.  For
deterministic $\zeta>0$, the definition of the infimum gives, for every
$\xi\in\Omega$, an index $k$ satisfying
$Q(\xi,p_k)<m(\xi)+\zeta$.  Thus, for every $\xi\in\Omega$, the set
$\{k\geq1:Q(\xi,p_k)\leq m(\xi)+\zeta\}$ is nonempty.  Define
\[
 J_\zeta:\Omega\to\{1,2,\ldots\},\qquad
 J_\zeta(\xi)
 =\min\{k\geq1:Q(\xi,p_k)\leq m(\xi)+\zeta\}.
\]
For every $k\geq1$,
\[
 \{J_\zeta=k\}
 =\{Q(\mathord\cdot,p_k)\leq m+\zeta\}
  \cap\bigcap_{j<k}\{Q(\mathord\cdot,p_j)>m+\zeta\}
 \in\mathcal F.
\]
Thus $J_\zeta$ is measurable for the discrete sigma-field.  Define
$\widehat p_\zeta=(\widehat t_\zeta,\widehat\lambda_\zeta)=p_{J_\zeta}$.
For every $A\in\mathcal B(\cH)$,
\[
 \{\widehat p_\zeta\in A\}
 =\bigcup_{\{k:p_k\in A\}}\{J_\zeta=k\}\in\mathcal F.
\]
Therefore $\widehat p_\zeta:\Omega\to\cH$ is
$\mathcal F/\mathcal B(\cH)$-measurable and satisfies
\eqref{eq:approximate-score-minimizer} with $\zeta_n\equiv\zeta$.

\subsubsection*{Tanh--logistic example}
For labels $y\in\{-1,1\}$, activation $\sigma(v)=\tanh(v)$, and outer loss
$\varphi(u,y)=\log(1+e^{-yu})$, \eqref{eq:mean-field-smoothness} holds with
\[
 (\mathsf S_0,\mathsf S_1,\mathsf S_2,\mathsf S_3)=(1,1,2,6),\qquad
 (\mathsf L_1,\mathsf L_2,\mathsf L_3)=(1,1/4,1/4),
 \qquad c_\varphi=\log 2,\quad L_0=0.
\]
Hence the tanh--logistic specification meets the smoothness premise of
\cref{cor:mean-field-width-oracle}.  Under the support, initialization, and
measurability hypotheses of \cref{cor:mean-field-width-oracle}, for width and
candidate families fixed in $n$, \eqref{eq:mean-field-width-risk-curve} is
$O(n^{-1/2})+O(n^{-2})$ at fixed horizon and confidence.

\begin{proof}[Proof of \cref{cor:mean-field-width-oracle}]
Set
\[
 A_T=Q_0+\mathsf L_1\mathsf S_0T,
 \quad
 W_T=Q_0+\mathsf L_1\mathsf S_1X
 \left(Q_0T+\frac{\mathsf L_1\mathsf S_0}{2}T^2\right),
 \quad Q_T=(A_T^2+W_T^2)^{1/2},
 \]
 \[
 B_{{\rm jac},T}=\mathsf S_0+A_T\mathsf S_1X,\quad
 D_{A_T}=2\mathsf S_1X+A_T\mathsf S_2X^2,\quad
 E_{A_T}=3\mathsf S_2X^2+A_T\mathsf S_3X^3,
\]
\[
 \Gamma_0=\mathsf L_2B_{{\rm jac},T}^2+\mathsf L_1D_{A_T},\quad
 \Gamma=\Gamma_0+\lambda_{\max},\quad
 \Xi=\mathsf L_3B_{{\rm jac},T}^3
     +3\mathsf L_2B_{{\rm jac},T}D_{A_T}+\mathsf L_1E_{A_T}.
\]
The constants
\[
 \begin{gathered}
 A_T,\ W_T,\ Q_T,\ B_{{\rm jac},T},\ D_{A_T},\ E_{A_T},\\
 \Gamma_0,\ \Gamma,\ \Xi
 \end{gathered}
\]
are finite nonnegative deterministic functions only of the fixed mean-field
primitives and $T$, independent of $m$, $n$, and the realized sample.  For a
particle array $q=(q_1,\ldots,q_m)\in\Theta_m$ and a linear map
$\mathsf R:\Theta_m\to\Theta_m$, define
\begin{equation}\label{eq:mean-field-array-norms}
 \begin{aligned}
 \|q\|_{{\rm b},\infty}
 &=\max_{1\leq r\leq m}\|q_r\|,\\
 \|\mathsf R\|_{m\to m}
 &=\sup_{\substack{u\in\Theta_m\\\|u\|_m\leq1}}\|\mathsf Ru\|_m,
 &\|\mathsf R\|_{{\rm b},\infty\to{\rm b},\infty}
 &=\sup_{\substack{u\in\Theta_m\\\|u\|_{{\rm b},\infty}\leq1}}
       \|\mathsf Ru\|_{{\rm b},\infty}.
 \end{aligned}
\end{equation}
The proof of \cref{thm:two-layer-mean-field} gives the particle-radius bound
\eqref{eq:mean-field-particle-radius}, the network-derivative bounds
\eqref{eq:mean-field-jacobian-bounds}--\eqref{eq:mean-field-output-hessian-bounds}, and the objective-Hessian and
Hessian-Lipschitz bounds \eqref{eq:mean-field-hessian-bound} and
\eqref{eq:mean-field-base-strip-hessian-lipschitz}.  Thus,
$A_T,W_T,Q_T,B_{{\rm jac},T},D_{A_T},E_{A_T},\Gamma_0,\Gamma$, and $\Xi$
provide the particle-radius,
Jacobian, Hessian, and Hessian-Lipschitz envelopes.  Work on the
full-$P_0$-measure set $\mathsf Z_0$ from the setup of
\cref{cor:mean-field-width-oracle}.  For $N\geq1$, $m\in\mathcal W_n$,
$\mathsf s\in\mathsf Z_0^N$, $\lambda\in\Lambda$, and $0\leq t\leq T$, write
  $q_t^{\mathsf s,\lambda}
  =((a_{r,t}^{\mathsf s,\lambda},w_{r,t}^{\mathsf s,\lambda}))_{r=1}^m$
  for the common-initialization width-$m$ flow trained on $\mathsf s$.  The
  argument establishing \eqref{eq:mean-field-particle-radius} applies to every
  empirical law $P_{\mathsf s}$ and gives
  \[
   \sup_{\substack{N\geq1,\ m\in\mathcal W_n,\
                    \mathsf s\in\mathsf Z_0^N\\
                    0\leq\tau\leq T,\ \lambda\in\Lambda}}
   \max_{1\leq r\leq m}|a_{r,\tau}^{\mathsf s,\lambda}|
   \leq A_T.
  \]
  Fix $m\in\mathcal W_n$ and a particle array
  $q=((a_r,w_r))_{r=1}^m\in\Theta_m$ satisfying
  $\max_{1\leq r\leq m}|a_r|\leq A_T$.  For every $(x,y)\in\mathsf Z_0$,
  Equations~\eqref{eq:mean-field-network}--\eqref{eq:mean-field-smoothness}
  imply
\[
\begin{aligned}
 |f_{m,q}(x)|
 &\leq\frac1m\sum_{r=1}^m|a_r|\,|\sigma(w_r^\top x)|
 \leq \mathsf S_0A_T,\\
 c_\varphi-\mathsf L_1\mathsf S_0A_T
 &\leq\varphi(f_{m,q}(x),y)
 \leq c_\varphi+L_0+\mathsf L_1\mathsf S_0A_T.
\end{aligned}
\]
Set
\begin{equation}\label{eq:mean-field-statistical-constants}
 B_T=L_0+2\mathsf L_1\mathsf S_0A_T,\qquad
 C_{\beta,T}=2\mathsf L_1^2B_{{\rm jac},T}^2\phi_\Gamma(T).
\end{equation}
$B_T$ is the loss-range length and $C_{\beta,T}/N$ is the
width-uniform replace-one stability bound at sample size $N$.  Both constants
are finite, nonnegative, and deterministic functions only of the fixed
mean-field primitives and $T$, independent of $m$, $n$, and the sample.
The bounds in \eqref{eq:mean-field-statistical-constants} place every evaluated
loss directly in the common deterministic interval
$[c_\varphi-\mathsf L_1\mathsf S_0A_T,
  c_\varphi+L_0+\mathsf L_1\mathsf S_0A_T]$ of length $B_T$.

Fix $N\geq1$, $m\in\mathcal W_n$, $\lambda\in\Lambda$, and
$\mathsf s_N,\mathsf s_N'\in\mathsf Z_0^N$ with $\mathsf s_N\sim \mathsf s_N'$.
Define the neighboring-flow difference path
$e^{\mathsf s_N,\mathsf s_N',\lambda}:[0,T]\to\Theta_m$ by
$e_s^{\mathsf s_N,\mathsf s_N',\lambda}
=q_s^{\mathsf s_N,\lambda}-q_s^{\mathsf s_N',\lambda}$, and write $e_s$ for
$e_s^{\mathsf s_N,\mathsf s_N',\lambda}$ in the present neighboring-flow
comparison.
For $0\leq s<T$, define
$D^+\|e_s\|_{{\rm b},\infty}:=\limsup_{\varepsilon\downarrow0}
\{\|e_{s+\varepsilon}\|_{{\rm b},\infty}
-\|e_s\|_{{\rm b},\infty}\}/\varepsilon$
as the upper-right Dini derivative of the map
$s\mapsto\|e_s\|_{{\rm b},\infty}$.
The neighboring vector fields and paths satisfy
\[
\begin{aligned}
 e_0&=0,\\
 \sup_{\substack{q=((a_r,w_r))_{r=1}^m\in\Theta_m\\
                   \max_{1\leq r\leq m}|a_r|\leq A_T}}
 \|\nabla F_{P_{\mathsf s_N},\lambda}(q)
   -\nabla F_{P_{\mathsf s_N'},\lambda}(q)\|_{{\rm b},\infty}
 &\leq\frac{2\mathsf L_1B_{{\rm jac},T}}{N},\\
 D^+\|e_s\|_{{\rm b},\infty}
 &\leq\Gamma\|e_s\|_{{\rm b},\infty}
       +\frac{2\mathsf L_1B_{{\rm jac},T}}{N},\\
 \|e_s\|_{{\rm b},\infty}
 &\leq\frac{2\mathsf L_1B_{{\rm jac},T}}{N}\phi_\Gamma(s)
 \leq\frac{2\mathsf L_1B_{{\rm jac},T}}{N}\phi_\Gamma(T).
\end{aligned}
\]
At every $t$ with $(t,\lambda)\in\cH_m$ and for every $z\in\mathsf Z_0$,
\[
\begin{aligned}
 |L_{N,(m,t,\lambda)}(\mathsf s_N,z)
   -L_{N,(m,t,\lambda)}(\mathsf s_N',z)|
 &\leq\mathsf L_1B_{{\rm jac},T}\|e_t\|_{{\rm b},\infty}
 \leq\frac{C_{\beta,T}}{N}.
\end{aligned}
\]
Taking the suprema in \eqref{eq:uniform-stability} proves
$\beta_{N,m}\leq C_{\beta,T}/N$, uniformly in $m$ and $N$.

To control the hyperparameter increments, let
\[
 G=\mathsf L_1B_{{\rm jac},T},\qquad V_T=G+\lambda_{\max}Q_T,\qquad
 C_{{\rm hyp},T}=G\{V_T+Q_T\phi_\Gamma(T)\}.
\]
$G$, $V_T$, and $C_{{\rm hyp},T}$ are, respectively, finite deterministic
nonnegative loss-gradient, particle-speed, and hyperparameter-modulus
envelopes.  None of $G$, $V_T$, and $C_{{\rm hyp},T}$ depends on anything
other than the fixed mean-field primitives and $T$; in particular, none
depends on $m$, $n$, or the sample.
For every $N\geq1$, $m\in\mathcal W_n$, $\mathsf s_N\in\mathsf Z_0^N$, and
$\lambda,\lambda'\in\Lambda$, the flow bounds give
\[
 \sup_{0\leq t\leq T}\|\dot q_t^{\mathsf s_N,\lambda}\|_{{\rm b},\infty}
 \leq V_T,
 \qquad
 \sup_{0\leq t\leq T}
 \|q_t^{\mathsf s_N,\lambda}-q_t^{\mathsf s_N,\lambda'}\|_{{\rm b},\infty}
 \leq Q_T\phi_\Gamma(T)|\lambda-\lambda'|.
\]
The inequality
\[
 \sup_{0\leq t\leq T}
 \|q_t^{\mathsf s_N,\lambda}-q_t^{\mathsf s_N,\lambda'}\|_{{\rm b},\infty}
 \leq Q_T\phi_\Gamma(T)|\lambda-\lambda'|
\]
follows by coupling the two flows and bounding the additional forcing by
$Q_T|\lambda-\lambda'|$.  The
$B_{{\rm jac},T}$-Lipschitz output map and
$\mathsf L_1$-Lipschitz outer loss prove, with the normalized
candidate-Lipschitz supremum defined to be zero when every
$\cH_m$, $m\in\mathcal W_n$, is a singleton,
\[
 \sup_{\substack{N\geq1,\ m\in\mathcal W_n,\ \mathsf s_N\in\mathsf Z_0^N\\
                 z\in\mathsf Z_0,\
                 (t,\lambda),(t',\lambda')\in\cH_m,\
                 (t,\lambda)\ne(t',\lambda')}}
 \frac{|L_{N,(m,t,\lambda)}(\mathsf s_N,z)
   -L_{N,(m,t',\lambda')}(\mathsf s_N,z)|}
 {d_1((t,\lambda),(t',\lambda'))}
 \leq C_{{\rm hyp},T}.
\]
We next verify mixed stability.  Fix $N\geq1$, $m\in\mathcal W_n$,
neighboring $\mathsf s_N,\mathsf s_N'\in\mathsf Z_0^N$, and
$\lambda\in\Lambda$.  For $\mathsf s\in\{\mathsf s_N,\mathsf s_N'\}$ and
$0\leq t\leq T$, put
$\upsilon_t^{\mathsf s,\lambda}=\partial_\lambda q_t^{\mathsf s,\lambda}\in\Theta_m$.
Set
\[
 C_{q,T}=2G\phi_\Gamma(T),\qquad
 C_{\dot q,T}=\Gamma C_{q,T}+2G,\qquad
 U_{\lambda,T}=Q_T\phi_\Gamma(T).
\]
$C_{q,T}$, $C_{\dot q,T}$, and $U_{\lambda,T}$ are finite deterministic
nonnegative envelopes for neighboring-path displacement, neighboring-speed
difference, and decay sensitivity.  The constants $C_{q,T}$,
$C_{\dot q,T}$, and $U_{\lambda,T}$ depend only on the fixed mean-field
primitives and $T$.
The vector-field comparison with $C_{q,T}=2G\phi_\Gamma(T)$ gives
$\|q_t^{\mathsf s_N,\lambda}-q_t^{\mathsf s_N',\lambda}\|_{{\rm b},\infty}
\leq C_{q,T}/N$.  Comparing the two vector fields once more yields
\begin{equation}\label{eq:mean-field-neighbor-time-derivative}
 \|\dot q_t^{\mathsf s_N,\lambda}-\dot q_t^{\mathsf s_N',\lambda}\|_{{\rm b},\infty}
 \leq\frac{C_{\dot q,T}}N.
\end{equation}
Because the initialization is common across samples and independent of
$\lambda$, each decay sensitivity starts from zero and, for
$\mathsf s\in\{\mathsf s_N,\mathsf s_N'\}$, satisfies
\[
 \dot \upsilon_t^{\mathsf s,\lambda}
 =-\nabla^2F_{P_{\mathsf s},\lambda}(q_t^{\mathsf s,\lambda})
   \upsilon_t^{\mathsf s,\lambda}-q_t^{\mathsf s,\lambda},
 \qquad
 \|\upsilon_t^{\mathsf s,\lambda}\|_{{\rm b},\infty}\leq U_{\lambda,T}.
\]
At a common parameter array, replacing one observation changes the empirical
Hessian by at most $2\Gamma_0/N$ in block-maximum operator norm.  The pathwise
Hessian Lipschitz bound is $\Xi$.  The upper-right Dini derivative of
$\|\upsilon_t^{\mathsf s_N,\lambda}
-\upsilon_t^{\mathsf s_N',\lambda}\|_{{\rm b},\infty}$ gives
\[
 D^+\|\upsilon_t^{\mathsf s_N,\lambda}-\upsilon_t^{\mathsf s_N',\lambda}\|_{{\rm b},\infty}
 \leq\Gamma\|\upsilon_t^{\mathsf s_N,\lambda}-\upsilon_t^{\mathsf s_N',\lambda}\|_{{\rm b},\infty}
 +\frac{(\Xi C_{q,T}+2\Gamma_0)U_{\lambda,T}+C_{q,T}}N.
\]
Define
\[
 C_{u,T}=\{(\Xi C_{q,T}+2\Gamma_0)U_{\lambda,T}+C_{q,T}\}
           \phi_\Gamma(T).
\]
$C_{u,T}\in[0,\infty)$ is the finite deterministic neighboring
decay-sensitivity envelope; $C_{u,T}$ depends only on the fixed mean-field
primitives and $T$.
Then
$\sup_{0\leq t\leq T}\|\upsilon_t^{\mathsf s_N,\lambda}-\upsilon_t^{\mathsf s_N',\lambda}\|_{{\rm b},\infty}
\leq C_{u,T}/N$.

For every $z\in\mathsf Z_0$, every
$q=((a_r,w_r))_{r=1}^m\in\Theta_m$ with
$\max_{1\leq r\leq m}|a_r|\leq A_T$, and every $v,v'\in\Theta_m$, the first
and second parameter derivatives of the evaluated loss obey
\[
 |D\ell_z(q)[v]|\leq G\|v\|_{{\rm b},\infty},\qquad
 |D^2\ell_z(q)[v,v']|
 \leq\Gamma_0\|v\|_{{\rm b},\infty}
                   \|v'\|_{{\rm b},\infty}.
\]
For each $z\in\mathsf Z_0$, define the neighboring-sample loss gap
$g_{\mathsf s_N,\mathsf s_N',z}:[0,T]\times\Lambda\to\R$ by
\[
 g_{\mathsf s_N,\mathsf s_N',z}(t,\lambda)
 =\ell_z(q_t^{\mathsf s_N,\lambda})-\ell_z(q_t^{\mathsf s_N',\lambda}).
\]
For every $(t,\lambda)\in[0,T]\times\Lambda$, the derivative bounds give
\begin{align*}
 |\partial_tg_{\mathsf s_N,\mathsf s_N',z}(t,\lambda)|
 &\leq\frac{C_{t,T}}N,
 &C_{t,T}&=GC_{\dot q,T}+\Gamma_0V_TC_{q,T},\\
 |\partial_\lambda g_{\mathsf s_N,\mathsf s_N',z}(t,\lambda)|
 &\leq\frac{C_{\lambda,T}}N,
 &C_{\lambda,T}&=GC_{u,T}+\Gamma_0U_{\lambda,T}C_{q,T}.
\end{align*}
Endpoint decay derivatives are interpreted one-sided.  The fundamental
theorem of calculus in $t$ and $\lambda$ proves
\begin{equation}\label{eq:mean-field-mixed-stability}
 \beta_{N,m}^{\rm mix}\leq\frac{C_{\Delta,T}}N,\qquad
 C_{\Delta,T}=\max\{C_{t,T},C_{\lambda,T}\}.
\end{equation}
$C_{t,T}$ and $C_{\lambda,T}$ are the finite deterministic nonnegative time-
and decay-increment envelopes, and $C_{\Delta,T}$ is the mixed-stability
maximum of $C_{t,T}$ and $C_{\lambda,T}$.  None of
$C_{t,T},C_{\lambda,T}$, and $C_{\Delta,T}$ depends on anything other than
the fixed mean-field primitives and $T$; in particular, none depends on $m$,
$n$, or the sample.
Only the existing $C^3$ bounds enter through $\Xi$.

We apply exact-LOO concentration to the width-$m$ evaluated-loss processes on
the full-$P_0$-measure space $\mathsf Z_0$.  Membership in the common finite
interval of length $B_T$ makes every deterministic anchor integrable.  For
each $m\in\mathcal W_n$, fix
a deterministic anchor $(t_{0,m},\lambda_{0,m})\in\cH_m$ and let
$D_{S,m}$ and $D_{S,m}^\circ$
denote the width-$m$ processes in
\cref{eq:absolute-loo-process,eq:relative-loo-process}.  Let
$L_{n-1,m}^{\rm hyp}$ and $\beta_{n-1,m}^{\rm mix}$ denote the corresponding
width-$m$ quantities in
\cref{eq:hyper-lipschitz-constant,eq:mixed-stability}, and put
\[
 \delta_\ast=\frac{\delta}{2|\mathcal W_n|},\qquad
 a_{n,m}^{\rm mix}
 =L_{n-1,m}^{\rm hyp}+(n-1)\beta_{n-1,m}^{\rm mix}
 \leq C_{{\rm hyp},T}+C_{\Delta,T}.
\]
Here $\delta_\ast\in(0,1)$ and $a_{n,m}^{\rm mix}\in[0,\infty)$ are finite
deterministic scalars.
Define
\[
\begin{aligned}
 r_{n,\delta}^{\rm anc}
 &=\frac{B_T+2C_{\beta,T}}{\sqrt{2n}}
   \sqrt{\log\frac{4|\mathcal W_n|}{\delta}},\\
 r_{n,m,\delta}^{\rm ch}
 &=\frac{C_{\rm ch}a_{n,m}^{\rm mix}}{\sqrt n}
   \left\{\mathfrak E_m+D_m
   \sqrt{\mathfrak L_{n,\mathcal W,\delta}}\right\}\\
 &\leq\frac{C_{\rm ch}(C_{{\rm hyp},T}+C_{\Delta,T})}{\sqrt n}
   \left\{\mathfrak E_m+D_m
   \sqrt{\mathfrak L_{n,\mathcal W,\delta}}\right\}.
\end{aligned}
\]
The radii $r_{n,\delta}^{\rm anc}$ and $r_{n,m,\delta}^{\rm ch}$ are finite
nonnegative deterministic scalars.  The quantity $\delta_\ast$ depends only
on $n$, $\delta$, and $\mathcal W_n$; $a_{n,m}^{\rm mix}$ depends only on
$n$, $L_{n-1,m}^{\rm hyp}$, and $\beta_{n-1,m}^{\rm mix}$;
$r_{n,\delta}^{\rm anc}$ depends only on $n$, $\delta$, $\mathcal W_n$,
$B_T$, and $C_{\beta,T}$; and $r_{n,m,\delta}^{\rm ch}$ depends only on
$n$, $m$, $\delta$, $\mathcal W_n$, $a_{n,m}^{\rm mix}$, $D_m$, and
$\mathfrak E_m$.  All four quantities are independent of the realized
sample.
Define the measurable map
$\mathfrak D_m^{\rm anc}:(\mathsf Z^n,\mathcal Z^{\otimes n})
\to(\R,\mathcal B(\R))$ as follows.  For
$\mathsf s_n\in\mathsf Z_0^n$, let
$\mathfrak D_m^{\rm anc}(\mathsf s_n)$ equal the expression in
\cref{eq:absolute-loo-process} at $(t_{0,m},\lambda_{0,m})$ with
$S=\mathsf s_n$; set $\mathfrak D_m^{\rm anc}(\mathsf s_n)=0$ for
$\mathsf s_n\in\mathsf Z^n\setminus\mathsf Z_0^n$.  Thus
$D_{S,m}(t_{0,m},\lambda_{0,m})=\mathfrak D_m^{\rm anc}(S)$ almost surely.
For every $i\in[n]$, conditional centering, the common range length $B_T$, and
$\beta_{n-1,m}\leq C_{\beta,T}/(n-1)$ give
\[
\begin{aligned}
 &\E\!\left[
   L_{n-1,(m,t_{0,m},\lambda_{0,m})}(S_{-i},Z_i)
   \mid S_{-i}\right]\\
 &\qquad=P_0L_{n-1,(m,t_{0,m},\lambda_{0,m})}(S_{-i},Z),\\
 \E\mathfrak D_m^{\rm anc}(S)&=0,\\
 \sup_{\substack{\mathsf s_n,\mathsf s_n'\in\mathsf Z_0^n\\
                   \mathsf s_n\sim\mathsf s_n'}}
 |\mathfrak D_m^{\rm anc}(\mathsf s_n)
   -\mathfrak D_m^{\rm anc}(\mathsf s_n')|
 &\leq\frac{B_T}{n}
   +\frac{2(n-1)}n\frac{C_{\beta,T}}{n-1}
 =\frac{B_T+2C_{\beta,T}}n.
\end{aligned}
\]
Apply \cref{thm:chained-loo-concentration} and McDiarmid's inequality
\cite[Lemma~1.2]{mcdiarmid1989bounded}.  For every $m\in\mathcal W_n$,
\[
\begin{aligned}
 \Pp\!\left\{
  \sup_{(t,\lambda)\in\cH_m}
  |D_{S,m}^\circ(t,\lambda;t_{0,m},\lambda_{0,m})|
  >r_{n,m,\delta}^{\rm ch}\right\}
 &\leq\delta_\ast,\\
 \Pp\!\left\{
  |D_{S,m}(t_{0,m},\lambda_{0,m})|
  >r_{n,\delta}^{\rm anc}\right\}
 &\leq\delta_\ast.
\end{aligned}
\]
The failure allocation $\delta_\ast$ satisfies
\[
 \log\frac4{\delta_\ast}
 =\mathfrak L_{n,\mathcal W,\delta},\qquad
 \log\frac2{\delta_\ast}
 =\log\frac{4|\mathcal W_n|}{\delta}
 \leq\mathfrak L_{n,\mathcal W,\delta}.
\]
Define the simultaneous event $\mathcal C_{n,\delta}\in\mathcal F$ by
\begin{equation}\label{eq:mean-field-simultaneous-concentration-event}
 \mathcal C_{n,\delta}
 =\bigcap_{m\in\mathcal W_n}
 \left\{
  |D_{S,m}(t_{0,m},\lambda_{0,m})|
     \leq r_{n,\delta}^{\rm anc},\quad
  \sup_{(t,\lambda)\in\cH_m}
  |D_{S,m}^\circ(t,\lambda;t_{0,m},\lambda_{0,m})|
     \leq r_{n,m,\delta}^{\rm ch}
 \right\}.
\end{equation}
A union bound over the chaining and anchor events gives
\[
 \Pp(\mathcal C_{n,\delta}^c)
 \leq\sum_{m\in\mathcal W_n}(\delta_\ast+\delta_\ast)
 =\delta.
\]
The identity $D_{S,m}^\circ(t,\lambda;t_{0,m},\lambda_{0,m})
=D_{S,m}(t,\lambda)-D_{S,m}(t_{0,m},\lambda_{0,m})$ gives, on
$\mathcal C_{n,\delta}$,
\[
\begin{aligned}
 \max_{m\in\mathcal W_n}\sup_{(t,\lambda)\in\cH_m}
 |D_{S,m}(t,\lambda)|
 &\leq r_{n,\delta}^{\rm anc}
      +\max_{m\in\mathcal W_n}r_{n,m,\delta}^{\rm ch}\\
 &\leq\frac{B_T+2C_{\beta,T}}{\sqrt{2n}}
       \sqrt{\mathfrak L_{n,\mathcal W,\delta}}\\
 &\quad+\frac{C_{\rm ch}(C_{{\rm hyp},T}+C_{\Delta,T})}{\sqrt n}
  \left\{\mathfrak E_{\mathcal W}+D_{\mathcal W}
       \sqrt{\mathfrak L_{n,\mathcal W,\delta}}\right\}\\
 &=s_{n,\mathcal W}^{\rm mf,ch}(\delta).
\end{aligned}
\]

For deletion-chord curvature, fix $m\in\mathcal W_n$ and a realization
$\xi\in\Omega$ such that $S(\xi)=(Z_1(\xi),\ldots,Z_n(\xi))\in\mathsf Z_0^n$;
suppress $\xi$ from the sample coordinates, $P_n$, and $h_i$.  Fix
$i\in[n]$, $\lambda\in\Lambda$, $0\leq\alpha_0\leq(n-1)^{-1}$, and
$P=P_n+\alpha_0 h_i$.  Since
$h_i=P_n-\delta_{Z_i}$,
\[
 P=\frac{1-(n-1)\alpha_0}{n}\delta_{Z_i}
   +\frac{1+\alpha_0}{n}\sum_{j\in[n]\setminus\{i\}}\delta_{Z_j},\qquad
 P(\mathsf Z)=1,\qquad P(\mathsf Z_0)=1.
\]
Write
$\theta_t^{P,\lambda}=((a_{r,t}^{P,\lambda},w_{r,t}^{P,\lambda}))_{r=1}^m$.
The coordinate inequalities in
\eqref{eq:mean-field-coordinate-dini-bounds}, with $P$ in place of $P_n$,
give
\[
 \sup_{0\leq t\leq T}
 \max_{1\leq r\leq m}|a_{r,t}^{P,\lambda}|\leq A_T,\qquad
 \sup_{0\leq t\leq T}
 \max_{1\leq r\leq m}\|w_{r,t}^{P,\lambda}\|\leq W_T,
 \qquad
 \sup_{0\leq t\leq T}
 \|\theta_t^{P,\lambda}\|_{{\rm b},\infty}\leq Q_T.
\]

Suppose $h_i\ne0$ and define the sample-dependent realized normalized
deletion direction
$\widetilde h_i:=h_i/\|h_i\|_{\TV}\in\mathcal M_0$.  Put
\[
 \mathcal I_{i,\alpha_0}
 =[-\alpha_0\|h_i\|_{\TV},
   \{(n-1)^{-1}-\alpha_0\}\|h_i\|_{\TV}].
\]
$\mathcal I_{i,\alpha_0}\subset\R$ is the realized-sample- and chord-base-dependent
interval of admissible perturbations that keeps
$P+\alpha\widetilde h_i$ on the original deletion chord.
For every $\alpha\in\mathcal I_{i,\alpha_0}$ and $q\in\Theta_m$,
\[
 -\nabla F_{P+\alpha\widetilde h_i,\lambda}(q)
 =-\nabla F_{P,\lambda}(q)
  -\alpha\int\nabla\ell^{\rm tr}_{\lambda,z}(q)
                 \,\dd\widetilde h_i(z).
\]
The $C^3$ assumptions in \cref{thm:two-layer-mean-field} give the vector-field
map
\[
 \mathcal I_{i,\alpha_0}\times\Theta_m\to\Theta_m,
 \qquad
 (\alpha,q)\mapsto-\nabla F_{P+\alpha\widetilde h_i,\lambda}(q),
\]
with a jointly $C^2$ extension to an open neighborhood of the
deletion-chord trajectories.  For $0\leq\beta_0\leq(n-1)^{-1}$, the
jointly $C^2$ neighborhood contains
$\theta_t^{P_n+\beta_0 h_i,\lambda}$ for every
$0\leq t\leq T$.  Differentiating in $\alpha$
gives \cref{eq:first-data-response,eq:second-data-response}, with one-sided
derivatives at chord endpoints:
\[
\begin{aligned}
 \dot u_t^\lambda[\widetilde h_i]
 &=-A_t^\lambda u_t^\lambda[\widetilde h_i]
   -g_{\widetilde h_i,t}^\lambda,
 &u_0^\lambda[\widetilde h_i]&=0,\\
 \dot u_t^{(2),\lambda}[\widetilde h_i,\widetilde h_i]
 &=-A_t^\lambda u_t^{(2),\lambda}[\widetilde h_i,\widetilde h_i]
   -2H_{\widetilde h_i,t}^\lambda u_t^\lambda[\widetilde h_i]
   -\mathsf T_t^\lambda[u_t^\lambda[\widetilde h_i],
                         u_t^\lambda[\widetilde h_i]],
 &u_0^{(2),\lambda}[\widetilde h_i,\widetilde h_i]&=0.
\end{aligned}
\]
Because $\widetilde h_i(\mathsf Z)=0$,
$\|\widetilde h_i\|_{\TV}=1$, and
$\ell^{\rm tr}_{\lambda,z}=\ell_z+\lambda R_m$, the penalty terms cancel:
\[
\begin{aligned}
 g_{\widetilde h_i,t}^\lambda
 &=\int\nabla\ell_z(\theta_t^{P,\lambda})
          \,\dd\widetilde h_i(z)
   +\lambda\nabla R_m(\theta_t^{P,\lambda})
          \widetilde h_i(\mathsf Z)
 =\int\nabla\ell_z(\theta_t^{P,\lambda})
          \,\dd\widetilde h_i(z),\\
 H_{\widetilde h_i,t}^\lambda
 &=\int\nabla^2\ell_z(\theta_t^{P,\lambda})
          \,\dd\widetilde h_i(z)
   +\lambda I\widetilde h_i(\mathsf Z)
 =\int\nabla^2\ell_z(\theta_t^{P,\lambda})
          \,\dd\widetilde h_i(z).
\end{aligned}
\]
For $0\leq t\leq T$ and every $v\in\Theta_m$,
Equations~\eqref{eq:mean-field-full-hessian}--\eqref{eq:mean-field-hessian-bound} and
\eqref{eq:mean-field-base-strip-hessian-lipschitz} give
\[
\begin{aligned}
 \|g_{\widetilde h_i,t}^\lambda\|_{{\rm b},\infty}
 &\leq\mathsf L_1B_{{\rm jac},T}=G,\\
 \|H_{\widetilde h_i,t}^\lambda\|_{{\rm b},\infty\to{\rm b},\infty}
 &\leq\mathsf L_2B_{{\rm jac},T}^2+\mathsf L_1D_{A_T}=\Gamma_0,\\
 \|A_t^\lambda\|_{{\rm b},\infty\to{\rm b},\infty}
 &\leq\Gamma_0+\lambda_{\max}=\Gamma,\\
 \|\mathsf T_t^\lambda[v,v]\|_{{\rm b},\infty}
 &\leq(\mathsf L_3B_{{\rm jac},T}^3
        +3\mathsf L_2B_{{\rm jac},T}D_{A_T}+\mathsf L_1E_{A_T})
          \|v\|_{{\rm b},\infty}^2
 =\Xi\|v\|_{{\rm b},\infty}^2.
\end{aligned}
\]
Variation of constants in
\cref{eq:first-data-response,eq:second-data-response} yields
\[
\begin{aligned}
 \|u_t^\lambda[\widetilde h_i]\|_{{\rm b},\infty}
 &\leq\int_0^t e^{\Gamma(t-s)}G\,\dd s
 \leq G\phi_\Gamma(T)=:U_T,\\
 \|u_t^{(2),\lambda}[\widetilde h_i,\widetilde h_i]\|_{{\rm b},\infty}
 &\leq\int_0^t e^{\Gamma(t-s)}
       \{2\Gamma_0U_T+\Xi U_T^2\}\,\dd s\\
 &\leq\{2\Gamma_0U_T+\Xi U_T^2\}\phi_\Gamma(T)
 =:W_T^{\rm resp}.
\end{aligned}
\]
For every $q=((a_r,w_r))_{r=1}^m\in\Theta_m$ with
$\max_{1\leq r\leq m}|a_r|\leq A_T$ and every $v\in\Theta_m$,
integrating the pointwise derivative estimates over $P_0$, using
$P_0(\mathsf Z_0)=1$, gives
\[
 |P_0D\ell_Z(q)[v]|\leq G\|v\|_{{\rm b},\infty},\qquad
 |P_0D^2\ell_Z(q)[v,v]|
 \leq\Gamma_0\|v\|_{{\rm b},\infty}^2.
\]
The $P_0D\ell_Z$ inequality uses
\eqref{eq:mean-field-particle-derivatives} and
\eqref{eq:mean-field-jacobian-bounds}; the $P_0D^2\ell_Z$ inequality uses
\eqref{eq:mean-field-particle-derivatives} and
\eqref{eq:mean-field-full-hessian}.
Define the width-uniform curvature constant
\begin{equation}\label{eq:mean-field-cpp-constant}
 C_{PP,T}=\Gamma_0U_T^2+GW_T^{\rm resp}.
\end{equation}
The constants $U_T$, $W_T^{\rm resp}$, and $C_{PP,T}$ are finite,
nonnegative, and deterministic functions only of the fixed mean-field
primitives and $T$, independent of $m$, $n$, and the sample.
The homogeneity relations and \eqref{eq:risk-second-response} give
\[
\begin{gathered}
 u_t^\lambda[h_i]=\|h_i\|_{\TV}u_t^\lambda[\widetilde h_i],\qquad
 u_t^{(2),\lambda}[h_i,h_i]
 =\|h_i\|_{\TV}^2
   u_t^{(2),\lambda}[\widetilde h_i,\widetilde h_i],\\
 |\psi_{i,t,\lambda}''(\alpha_0)|
 \leq\Gamma_0\|u_t^\lambda[h_i]\|_{{\rm b},\infty}^2
      +G\|u_t^{(2),\lambda}[h_i,h_i]\|_{{\rm b},\infty}
 \leq C_{PP,T}\|h_i\|_{\TV}^2.
\end{gathered}
\]
At $P_n$, uniqueness in \eqref{eq:first-data-response} gives
\[
 u_t^\lambda[\alpha h+\beta h']
 =\alpha u_t^\lambda[h]+\beta u_t^\lambda[h'],
 \qquad \alpha,\beta\in\R,\quad h,h'\in\mathcal V_S,
\]
so \eqref{eq:risk-first-response} supplies the linear-functional sufficient
condition, and hence the direct cancellation in
\eqref{eq:average-first-order-cancellation}.  If $h_i=0$, then
$P_n+\alpha_0 h_i=P_n$ and
$\psi_{i,t,\lambda}''(\alpha_0)=0$.

Equations~\eqref{eq:mean-field-score-rate} and
\eqref{eq:cross-size-absolute}, with $C_{PP}=C_{PP,T}$, give the deterministic
envelopes, for $m\in\mathcal W_n$ and $(t,\lambda)\in\cH_m$,
\[
 \bar a_n(m,t,\lambda)=\frac{2T\mathsf L_1^2C_{{\rm ker},T}}{n(n-1)},
 \qquad
 \bar b_n(m,t,\lambda)=\frac{2C_{PP,T}}{(n-1)^2}.
\]
Put
\[
 \mathfrak c_{n,T}^{\rm mf}
 =\frac{2T\mathsf L_1^2C_{{\rm ker},T}}{n(n-1)}
  +\frac{2C_{PP,T}}{(n-1)^2},\qquad
 \varepsilon_{n,\mathcal W,T}^{\rm mf}(\delta)
 =\mathfrak c_{n,T}^{\rm mf}+s_{n,\mathcal W}^{\rm mf,ch}(\delta).
\]
$\mathfrak c_{n,T}^{\rm mf}\in[0,\infty)$ is the finite deterministic common
modelwise correction, and
$\varepsilon_{n,\mathcal W,T}^{\rm mf}:(0,1)\to[0,\infty)$ is the finite
deterministic simultaneous curve radius.  The correction
$\mathfrak c_{n,T}^{\rm mf}$ may depend on $n$, $T$, and the fixed mean-field
primitives.  The radius
$\varepsilon_{n,\mathcal W,T}^{\rm mf}(\delta)$ may additionally depend on
$\delta$, $\mathcal W_n$, and $(\cH_m)_{m\in\mathcal W_n}$.  Neither
quantity depends on the realized sample.
On $\mathcal C_{n,\delta}$, for every $m\in\mathcal W_n$ and
$(t,\lambda)\in\cH_m$, the \FlowALO--LOO, LOO--deletion-risk, and
deletion--full-risk differences satisfy
\[
\begin{aligned}
 &|\widetilde\Score_{n,m}(t,\lambda)-\Risk_{S,m}(t,\lambda)|\\
 &\quad\leq
 |\widetilde\Score_{n,m}(t,\lambda)-\Score_{n,m}^{\LOO}(t,\lambda)|
 +|\Score_{n,m}^{\LOO}(t,\lambda)-\Risk_{S,m}^-(t,\lambda)|\\
 &\qquad
 +|\Risk_{S,m}^-(t,\lambda)-\Risk_{S,m}(t,\lambda)|
 \leq\bar a_n(m,t,\lambda)
      +s_{n,\mathcal W}^{\rm mf,ch}(\delta)
      +\bar b_n(m,t,\lambda)
 =\varepsilon_{n,\mathcal W,T}^{\rm mf}(\delta).
\end{aligned}
\]
Taking the maximum over $m$ and supremum over $\cH_m$ proves
\eqref{eq:mean-field-width-risk-curve}.
\end{proof}

\subsection{Uniformly contractive risk curves}

Positive curvature of every empirical and deletion-chord objective gives a
model-independent sufficient condition for a risk-curve radius that does not
grow with the time horizon.  The curvature requirement is global over the
sample class because the exact-LOO concentration step uses the uniform
stability in \eqref{eq:uniform-stability}.

For $N\geq1$ and
$\mathsf s_N=(z_1,\ldots,z_N)\in\mathsf Z^N$, define the empirical probability
measure $P_{\mathsf s_N}=N^{-1}\sum_{j=1}^N\delta_{z_j}$.  For every $n\geq2$,
let $\mathfrak P_n^{\rm ctr}$ be the deterministic class of probability
measures
\begin{equation}\label{eq:contractive-measure-class}
\begin{aligned}
 \mathfrak P_n^{\rm ctr}
 &=\{P_{\mathsf s_N}:N\in\{n-1,n\},\ \mathsf s_N\in\mathsf Z^N\}\\
 &\quad\cup\left\{P_{\mathsf s_n}
   +\alpha(P_{\mathsf s_n}-\delta_{z_i}):
   \mathsf s_n=(z_1,\ldots,z_n)\in\mathsf Z^n,\ i\in[n],\
   0\leq\alpha\leq(n-1)^{-1}\right\}.
\end{aligned}
\end{equation}
The second set in \eqref{eq:contractive-measure-class} contains every
normalized deletion chord from a size-$n$ empirical measure to a size-$(n-1)$
empirical measure.

\subsubsection{Conditions for uniformly contractive risk curves}
\label{app:contractive-risk-curve-conditions}

Fix deterministic constants
\[
 \mu>0,\qquad
 G_{\rm tr},H_{\rm tr},J_3,G_{\rm ev},H_{\rm ev},
 L_{\rm tr}^{\lambda},L_{\rm ev}^{\lambda},B_{\rm rng}\in[0,\infty),
\]
and a deterministic convex set $\mathcal K\subseteq\Theta$ containing
$\theta_0$.  Define the deterministic response enlargement
\[
 \mathcal K_\mu^+
 =\{\theta+v:\theta\in\mathcal K,\ v\in\Theta,\
                   \|v\|\leq2G_{\rm tr}/\mu\}.
\]
Assume the following conditions hold for every $n\geq2$, every
$P\in\mathfrak P_n^{\rm ctr}$, and every $\lambda\in\Lambda$.
The flow in \eqref{eq:full-flow} exists on $[0,\infty)$ and remains in
$\mathcal K$.  The objective $F_{P,\lambda}$ is $C^3$ on an open neighborhood
of $\mathcal K$, and, uniformly over $\theta\in\mathcal K$,
\begin{equation}\label{eq:contractive-primitive-bounds}
\begin{gathered}
 \nabla^2F_{P,\lambda}(\theta)\succeq\mu I,
 \qquad
 \|\nabla^3F_{P,\lambda}(\theta)\|_{\op}\leq J_3,\\
 \sup_{z\in\mathsf Z}
 \|\nabla\ell^{\rm tr}_{\lambda,z}(\theta)\|\leq G_{\rm tr},
 \qquad
 \sup_{z\in\mathsf Z}
 \|\nabla^2\ell^{\rm tr}_{\lambda,z}(\theta)\|_{\op}\leq H_{\rm tr}.
\end{gathered}
\end{equation}
Every evaluation loss is $C^2$ in $\theta$ on an open neighborhood of
$\mathcal K_\mu^+$ and satisfies, uniformly over
$(z,\lambda,\theta)\in\mathsf Z\times\Lambda\times\mathcal K_\mu^+$,
\begin{equation}\label{eq:contractive-evaluation-bounds}
 \|\nabla\ell^{\rm ev}_{\lambda,z}(\theta)\|\leq G_{\rm ev},
 \qquad
 \|\nabla^2\ell^{\rm ev}_{\lambda,z}(\theta)\|_{\op}\leq H_{\rm ev}.
\end{equation}
There is one deterministic interval of length $B_{\rm rng}$ containing
$\ell^{\rm ev}_{\lambda,z}(\theta)$ for every
$(z,\lambda,\theta)\in\mathsf Z\times\Lambda\times\mathcal K$.
For every $P\in\mathfrak P_n^{\rm ctr}$,
$\theta\in\mathcal K$, and $\lambda,\lambda'\in\Lambda$, assume
\begin{equation}\label{eq:contractive-lambda-moduli}
\begin{aligned}
 \|\nabla F_{P,\lambda}(\theta)-\nabla F_{P,\lambda'}(\theta)\|
 &\leq L_{\rm tr}^{\lambda}\|\lambda-\lambda'\|,\\
 \sup_{z\in\mathsf Z}
 |\ell^{\rm ev}_{\lambda,z}(\theta)-
   \ell^{\rm ev}_{\lambda',z}(\theta)|
 &\leq L_{\rm ev}^{\lambda}\|\lambda-\lambda'\|.
\end{aligned}
\end{equation}

For every realized deletion chord, assume the twice-continuous
measure-response and differentiation-under-the-integral conditions in
\cref{app:continuous-measure-responses} hold on each finite time interval, including at
every base measure on the chord and in the normalized deletion direction.
For every $n\geq2$, let $T_n\in[0,\infty)$ be deterministic and define
\begin{equation}\label{eq:contractive-intrinsic-metric}
 \mathfrak t_\mu(t)=\int_0^te^{-\mu s}\,\dd s,
 \qquad
 d_\mu((t,\lambda),(t',\lambda'))
 =|\mathfrak t_\mu(t)-\mathfrak t_\mu(t')|
   +\|\lambda-\lambda'\|.
\end{equation}
Let $\cH_n\subseteq[0,T_n]\times\Lambda$ be a predetermined nonempty
$d_\mu$-compact set.  For $0\leq t\leq T_n$, let
$\theta_t^\lambda=\theta_t^{P_n,\lambda}$ and
$\theta_{-i,t}^\lambda=\theta_t^{P_{-i},\lambda}$, define
$\Delta_{i,t}^\lambda=\theta_{-i,t}^\lambda-\theta_t^\lambda$, and let
$d_{i,t}^\lambda$ solve \eqref{eq:flow-alo-response}.  On $\cH_n$, use the
definitions of
$\widetilde\Score_n,\Score_n^{\LOO},\Risk_S^-,\Risk_S$, and
$\psi_{i,t,\lambda}$ from
\cref{eq:flow-alo-score,eq:loo-score,eq:minus-risk,eq:conditional-risk,eq:deletion-chord-functional}.
\[
 L_{N,(t,\lambda)}(\mathsf s_N,z)
 =\ell^{\rm ev}_{\lambda,z}
   (\theta_t^{P_{\mathsf s_N},\lambda}),
 \qquad N\in\{n-1,n\},
\]
defines the evaluated-loss process on
$\mathsf Z^N\times\mathsf Z\times\cH_n$.  Assume the map
\[
 (\mathsf s_{n-1},z,t,\lambda)\longmapsto
 L_{n-1,(t,\lambda)}(\mathsf s_{n-1},z)
\]
is jointly measurable on
$\mathsf Z^{n-1}\times\mathsf Z\times\cH_n$.

\begin{theorem}[Time-uniform risk curves under contraction]
\label[theorem]{thm:contractive-risk-curve}
Suppose the conditions in
\cref{app:contractive-risk-curve-conditions} hold.  Let
$D_\Lambda=\operatorname{diam}(\Lambda,\|\cdot\|)$ and define the finite
deterministic constants
\begin{equation}\label{eq:contractive-risk-constants}
\begin{aligned}
 C_{{\rm ALO},\mu}
 &=\frac{2G_{\rm ev}J_3G_{\rm tr}^2}{\mu^3},\\
 C_{PP,\mu}
 &=\frac{H_{\rm ev}G_{\rm tr}^2}{\mu^2}
   +G_{\rm ev}\left\{
      \frac{2H_{\rm tr}G_{\rm tr}}{\mu^2}
      +\frac{J_3G_{\rm tr}^2}{\mu^3}\right\},\\
 C_{{\rm hyp},\mu}
 &=\max\left\{G_{\rm ev}G_{\rm tr},
       L_{\rm ev}^{\lambda}
       +\frac{G_{\rm ev}L_{\rm tr}^{\lambda}}{\mu}\right\},
 &D_\mu&=\mu^{-1}+D_\Lambda.
\end{aligned}
\end{equation}
For $\delta\in(0,1)$, define
\begin{equation}\label{eq:contractive-risk-radius}
\begin{aligned}
 \mathfrak r_{n,\mu}^{\rm ctr}(\delta)
 &=\frac{C_{{\rm ALO},\mu}+2C_{PP,\mu}}{(n-1)^2}
   +\frac{2C_{{\rm hyp},\mu}}{\sqrt n}\\
 &\quad+\frac{B_{\rm rng}+4G_{\rm ev}G_{\rm tr}/\mu}{\sqrt{2n}}
 \left[(d_\lambda+1)\log\{1+4D_\mu\sqrt n\}
       +\log\frac2\delta\right]^{1/2}.
\end{aligned}
\end{equation}
Then there is a measurable event with probability at least $1-\delta$ on
which
\begin{equation}\label{eq:contractive-risk-curve}
 \sup_{(t,\lambda)\in\cH_n}
 |\widetilde\Score_n(t,\lambda)-\Risk_S(t,\lambda)|
 \leq\mathfrak r_{n,\mu}^{\rm ctr}(\delta).
\end{equation}
The radius in \eqref{eq:contractive-risk-radius} is independent of $T_n$.
\end{theorem}

\begin{proof}[Proof of \cref{thm:contractive-risk-curve}]
Fix $n\geq2$.  Every full, deleted, or neighboring training flow in this proof
remains in the convex set $\mathcal K$ by hypothesis.  For any continuous
Hessian path $H_t\succeq\mu I$ on
$\mathcal K$, let $U_H(t,s)$ be the fundamental solution generated by $-H$.
For $v\in\Theta$ and $0\leq s\leq t$,
\[
 \frac{\dd}{\dd t}\|U_H(t,s)v\|^2
 =-2\langle U_H(t,s)v,H_tU_H(t,s)v\rangle
 \leq-2\mu\|U_H(t,s)v\|^2.
\]
Gronwall's inequality gives the uniform propagator bound
\begin{equation}\label{eq:contractive-propagator}
 \|U_H(t,s)\|_{\op}\leq e^{-\mu(t-s)},
 \qquad 0\leq s\leq t<\infty.
\end{equation}
The convexity of $\mathcal K$ and the first inequality in
\eqref{eq:contractive-primitive-bounds} make
\eqref{eq:contractive-propagator} applicable to every full-path,
exact-deletion secant, response, measure-response, and neighboring-flow
segment propagator used in the proof.

For $i\in[n]$, the centered gradient in \eqref{eq:centered-gradient} satisfies
$\|c_{i,t}^\lambda\|\leq2G_{\rm tr}$.  Variation of constants in
\cref{eq:exact-deletion-ode,eq:flow-alo-response} and
\eqref{eq:contractive-propagator} therefore gives, for every $t\geq0$ and
$\lambda\in\Lambda$,
\begin{equation}\label{eq:contractive-deletion-response-bounds}
 \|\Delta_{i,t}^\lambda\|\vee\|d_{i,t}^\lambda\|
 \leq\frac{2G_{\rm tr}}{n-1}\int_0^te^{-\mu(t-s)}\,\dd s
 \leq\frac{2G_{\rm tr}}{\mu(n-1)}.
\end{equation}
In particular, $\theta_t^\lambda+d_{i,t}^\lambda\in\mathcal K_\mu^+$.
The third-derivative bound in \eqref{eq:contractive-primitive-bounds} and
\cref{eq:secant-hessian,eq:resummed-B} give
\[
 \|\overline B_{i,t}^\lambda-B_{i,t}^\lambda\|_{\op}
 \leq\int_0^1J_3u\|\Delta_{i,t}^\lambda\|\,\dd u
 =\frac{J_3}{2}\|\Delta_{i,t}^\lambda\|.
\]
For $r_{i,t}^\lambda=\Delta_{i,t}^\lambda-d_{i,t}^\lambda$, subtraction of
\eqref{eq:flow-alo-response} from \eqref{eq:exact-deletion-ode} yields
\[
 \dot r_{i,t}^\lambda
 =-B_{i,t}^\lambda r_{i,t}^\lambda
   -(\overline B_{i,t}^\lambda-B_{i,t}^\lambda)
      \Delta_{i,t}^\lambda,
 \qquad r_{i,0}^\lambda=0.
\]
Equations~\eqref{eq:contractive-propagator} and
\eqref{eq:contractive-deletion-response-bounds} imply
\begin{equation}\label{eq:contractive-response-remainder}
 \sup_{t\geq0}\|\Delta_{i,t}^\lambda-d_{i,t}^\lambda\|
 \leq\frac{J_3}{2}\int_0^\infty e^{-\mu s}
       \left\{\frac{2G_{\rm tr}}{\mu(n-1)}\right\}^2\,\dd s
 =\frac{2J_3G_{\rm tr}^2}{\mu^3(n-1)^2}.
\end{equation}
The set $\mathcal K_\mu^+$ is convex.  The evaluation-gradient bound in
\eqref{eq:contractive-evaluation-bounds}, the mean-value formula on the segment
joining $\theta_{-i,t}^\lambda$ to
$\theta_t^\lambda+d_{i,t}^\lambda$, and
\eqref{eq:contractive-response-remainder} give
\begin{equation}\label{eq:contractive-score-approximation-general}
 \sup_{(t,\lambda)\in\cH_n}
 |\widetilde\Score_n(t,\lambda)-\Score_n^{\LOO}(t,\lambda)|
 \leq\frac{C_{{\rm ALO},\mu}}{(n-1)^2}.
\end{equation}

We next bound the deletion-chord curvature.  Fix $i\in[n]$ with $h_i\ne0$,
$0\leq\alpha_0\leq(n-1)^{-1}$, and
$P=P_n+\alpha_0h_i$.  For
$\widetilde h_i=h_i/\|h_i\|_{\TV}$, the bounds in
\eqref{eq:contractive-primitive-bounds} imply, along the path trained under
$P$,
\[
 \left\|\int\nabla\ell^{\rm tr}_{\lambda,z}\,\dd\widetilde h_i(z)\right\|
 \leq G_{\rm tr},\qquad
 \left\|\int\nabla^2\ell^{\rm tr}_{\lambda,z}\,\dd\widetilde h_i(z)\right\|_{\op}
 \leq H_{\rm tr}.
\]
Variation of constants in
\cref{eq:first-data-response,eq:second-data-response} and
\eqref{eq:contractive-propagator} gives, uniformly over $t\geq0$ and
$\lambda\in\Lambda$,
\begin{equation}\label{eq:contractive-measure-response-bounds}
\begin{aligned}
 \|u_t^\lambda[\widetilde h_i]\|
 &\leq\frac{G_{\rm tr}}\mu,\\
 \|u_t^{(2),\lambda}[\widetilde h_i,\widetilde h_i]\|
 &\leq\frac{2H_{\rm tr}G_{\rm tr}}{\mu^2}
       +\frac{J_3G_{\rm tr}^2}{\mu^3}.
\end{aligned}
\end{equation}
The population averages of the two evaluation derivatives in
\eqref{eq:contractive-evaluation-bounds} are bounded by $G_{\rm ev}$ and
$H_{\rm ev}$.  Substitution of
\eqref{eq:contractive-measure-response-bounds} into
\eqref{eq:risk-second-response}, followed by the degree-two homogeneity of the
second response in the signed direction, gives
\begin{equation}\label{eq:contractive-cpp-general}
 \sup_{\substack{(t,\lambda)\in\cH_n\\
                  0\leq\alpha_0\leq(n-1)^{-1}}}
 |\psi_{i,t,\lambda}''(\alpha_0)|
 \leq C_{PP,\mu}\|h_i\|_{\TV}^2.
\end{equation}
For $h_i=0$, the deletion chord is constant and
\eqref{eq:contractive-cpp-general} holds with a zero left-hand side.
The first response in \eqref{eq:first-data-response} is linear in the signed
direction at $P_n$, so the argument in
\cref{app:continuous-measure-responses} verifies
\eqref{eq:average-first-order-cancellation}.  Applying
\cref{thm:jackknife-cancellation} with \eqref{eq:contractive-cpp-general}
yields
\begin{equation}\label{eq:contractive-cross-size-general}
 \sup_{(t,\lambda)\in\cH_n}
 |\Risk_S^-(t,\lambda)-\Risk_S(t,\lambda)|
 \leq\frac{2C_{PP,\mu}}{(n-1)^2}.
\end{equation}

For the exact-LOO concentration term, fix $N\in\{n-1,n\}$ and neighboring
$\mathsf s_N,\mathsf s_N'\in\mathsf Z^N$.  The coupled paths and every segment
joining them lie in $\mathcal K$.  Write
$x_t=\theta_t^{P_{\mathsf s_N},\lambda}$ and
$y_t=\theta_t^{P_{\mathsf s_N'},\lambda}$.  Splitting the gradient difference
at $y_t$ gives
\[
 \dot x_t-\dot y_t
 =-\overline H_t(x_t-y_t)
   -\{\nabla F_{P_{\mathsf s_N},\lambda}(y_t)
       -\nabla F_{P_{\mathsf s_N'},\lambda}(y_t)\},
\]
where the segment Hessian $\overline H_t\succeq\mu I$.  The neighboring
empirical-gradient difference in braces has norm at most $2G_{\rm tr}/N$.
Variation of constants, followed by the evaluation-gradient bound, therefore gives
\begin{equation}\label{eq:contractive-stability-general}
 \sup_{t\geq0}\|x_t-y_t\|
 \leq\frac{2G_{\rm tr}}{\mu N},
 \qquad
 \beta_N\leq\frac{2G_{\rm ev}G_{\rm tr}}{\mu N}.
\end{equation}

For $P=P_{\mathsf s_N}$, define
$v_t^{P,\lambda}=\nabla F_{P,\lambda}(\theta_t^{P,\lambda})$.  Differentiating
along the flow gives
\[
 \dot v_t^{P,\lambda}
 =-\nabla^2F_{P,\lambda}(\theta_t^{P,\lambda})v_t^{P,\lambda}.
\]
Since $\|v_0^{P,\lambda}\|\leq G_{\rm tr}$,
\eqref{eq:contractive-propagator} yields
\begin{equation}\label{eq:contractive-speed-general}
 \|v_t^{P,\lambda}\|\leq G_{\rm tr}e^{-\mu t},
 \qquad t\geq0.
\end{equation}
For $\lambda,\lambda'\in\Lambda$, subtract the two flow equations and use
the segment Hessian of $F_{P,\lambda}$ together with the first inequality in
\eqref{eq:contractive-lambda-moduli}.  Variation of constants gives
\begin{equation}\label{eq:contractive-lambda-flow}
 \sup_{t\geq0}
 \|\theta_t^{P,\lambda}-\theta_t^{P,\lambda'}\|
 \leq\frac{L_{\rm tr}^{\lambda}}\mu
       \|\lambda-\lambda'\|.
\end{equation}
Equations~\eqref{eq:contractive-evaluation-bounds},
\eqref{eq:contractive-lambda-moduli},
\eqref{eq:contractive-speed-general}, and
\eqref{eq:contractive-lambda-flow} imply, for
$p=(t,\lambda),p'=(t',\lambda')\in\cH_n$, every
$\mathsf s_N\in\mathsf Z^N$, and every $z\in\mathsf Z$,
\begin{equation}\label{eq:contractive-loss-modulus}
 |L_{N,p}(\mathsf s_N,z)-L_{N,p'}(\mathsf s_N,z)|
 \leq C_{{\rm hyp},\mu}d_\mu(p,p').
\end{equation}

The transformed time interval in \eqref{eq:contractive-intrinsic-metric} has
length at most $\mu^{-1}$.  The map
$(t,\lambda)\mapsto(\mathfrak t_\mu(t),\lambda)$ takes $\cH_n$ into
$\mathbb R^{d_\lambda+1}$ equipped with the norm
$(a,b)\mapsto|a|+\|b\|$ and preserves $d_\mu$.  The image has diameter at
most $D_\mu$.  For $r>0$, take a maximal $r$-separated subset of the image.
The open norm balls of radius $r/2$ around the selected points are disjoint and
are contained in a ball of radius $D_\mu+r/2$ around any one selected point.
Comparing their $(d_\lambda+1)$-dimensional volumes and using maximality to
obtain an $r$-net whose centers belong to $\cH_n$ gives
\begin{equation}\label{eq:contractive-covering-general}
 \mathcal N(\cH_n,d_\mu,r)
 \leq\left(1+\frac{2D_\mu}{r}\right)^{d_\lambda+1}
 \leq\left(1+\frac{4D_\mu}{r}\right)^{d_\lambda+1}.
\end{equation}
Use \cref{thm:loo-concentration} with $r=n^{-1/2}$,
the range length $B_{\rm rng}$, the stability bound
\eqref{eq:contractive-stability-general}, the modulus
\eqref{eq:contractive-loss-modulus}, and the covering bound
\eqref{eq:contractive-covering-general}.  Equation~\eqref{eq:bounded-difference-c}
then gives
\[
 c_n^-
 \leq\frac{B_{\rm rng}+4G_{\rm ev}G_{\rm tr}/\mu}{n}.
\]
With probability at least $1-\delta$,
\begin{equation}\label{eq:contractive-loo-concentration-general}
\begin{aligned}
 \sup_{(t,\lambda)\in\cH_n}
 |\Score_n^{\LOO}(t,\lambda)-\Risk_S^-(t,\lambda)|
 &\leq\frac{2C_{{\rm hyp},\mu}}{\sqrt n}\\
 &\quad+\frac{B_{\rm rng}+4G_{\rm ev}G_{\rm tr}/\mu}{\sqrt{2n}}\\
 &\qquad\times
 \left[(d_\lambda+1)\log\{1+4D_\mu\sqrt n\}
       +\log\frac2\delta\right]^{1/2}.
\end{aligned}
\end{equation}
On the event in \eqref{eq:contractive-loo-concentration-general}, add the
three bounds
\eqref{eq:contractive-score-approximation-general},
\eqref{eq:contractive-loo-concentration-general}, and
\eqref{eq:contractive-cross-size-general}.  The resulting right-hand side is
\eqref{eq:contractive-risk-radius}, which proves
\eqref{eq:contractive-risk-curve}.
\end{proof}

For fixed $d_\lambda$ and fixed constants in
\eqref{eq:contractive-risk-constants}, the radius in
\eqref{eq:contractive-risk-radius} is
$O(\sqrt{\log n/n}+n^{-2})$.  Removing the finite-net logarithm requires the
additional mixed-stability premise of \cref{thm:chained-loo-concentration};
positive curvature alone supplies only the ordinary stability in
\eqref{eq:contractive-stability-general}.

For a direct verification of the first inequality in
\eqref{eq:contractive-primitive-bounds}, suppose
$\ell^{\rm tr}_{\lambda,z}=\bar\ell^{\rm tr}_{\lambda,z}+R_\lambda$ on
$\mathcal K$, where
$\nabla^2\bar\ell^{\rm tr}_{\lambda,z}\succeq-\Gamma I$ and
$\nabla^2R_\lambda\succeq(\Gamma+\mu)I$ uniformly over
$(z,\lambda,\theta)\in\mathsf Z\times\Lambda\times\mathcal K$.
Every probability measure in \eqref{eq:contractive-measure-class} then gives
$\nabla^2F_{P,\lambda}\succeq\mu I$ on $\mathcal K$.

\subsection{Finite-horizon gradient-flow stability bounds}
For the neighboring-flow and candidate-index checks, fix $N\geq1$ and finite
deterministic $G_{\rm tr},K_{\rm ev}\geq0$.  For
$\mathsf s_N\in\mathsf Z^N$, let $P_{\mathsf s_N}$ denote the empirical
measure of $\mathsf s_N$ and write
$\theta_t^{\mathsf s_N,\lambda}=\theta_t^{P_{\mathsf s_N},\lambda}$ whenever
the flow exists.  For $p=(t,\lambda)\in\cH$, define the map
$L_{N,p}:\mathsf Z^N\times\mathsf Z\to\mathbb R$ by
$L_{N,p}(\mathsf s_N,z)=
\ell^{\rm ev}_{\lambda,z}(\theta_t^{\mathsf s_N,\lambda})$, and let
$\beta_N$ and $L_N^{\rm hyp}$ be the constants of
$\{L_{N,p}:p\in\cH\}$ in
\cref{eq:uniform-stability,eq:hyper-lipschitz-constant}.

\subsubsection{Neighboring-flow conditions}
\label{app:gf-neighboring-flow-conditions}

For every neighboring
$\mathsf s_N,\allowbreak \mathsf s_N'\in\mathsf Z^N$ and $\lambda\in\Lambda$, assume
the coupled flows exist on $[0,T]$ from a common initialization.  Assume that,
for every $t\in[0,T]$, $F_{P_{\mathsf s_N},\lambda}$ is $C^2$ on an open
neighborhood of the segment joining $\theta_t^{\mathsf s_N,\lambda}$ and
$\theta_t^{\mathsf s_N',\lambda}$.  For $0\leq s\leq t\leq T$,
define the segment Hessian and fundamental solution
$\overline H_{\mathsf s_N,\mathsf s_N',\lambda}(t),
\Phi_{\mathsf s_N,\mathsf s_N',\lambda}(t,s):\Theta\to\Theta$ by
\begin{equation}\label{eq:stability-segment-hessian}
\begin{aligned}
 \overline H_{\mathsf s_N,\mathsf s_N',\lambda}(t)
 &=\int_0^1\nabla_\theta^2F_{P_{\mathsf s_N},\lambda}
 \bigl(\theta_t^{\mathsf s_N',\lambda}
       +u(\theta_t^{\mathsf s_N,\lambda}-\theta_t^{\mathsf s_N',\lambda})\bigr)\,\dd u,\\
 \partial_t\Phi_{\mathsf s_N,\mathsf s_N',\lambda}(t,s)
 &=-\overline H_{\mathsf s_N,\mathsf s_N',\lambda}(t)\Phi_{\mathsf s_N,\mathsf s_N',\lambda}(t,s),
 \qquad \Phi_{\mathsf s_N,\mathsf s_N',\lambda}(s,s)=I.
\end{aligned}
\end{equation}
Define the deterministic extended-nonnegative amplification factor
$\mathfrak A_{N,T}\in[0,\infty]$ by
\begin{equation}\label{eq:amplification-factor}
 \mathfrak A_{N,T}
 =\sup_{\substack{\lambda\in\Lambda,\ \mathsf s_N,\mathsf s_N'\in\mathsf Z^N,\ \mathsf s_N\sim \mathsf s_N'\\
                   0\leq t\leq T}}
   \int_0^t\|\Phi_{\mathsf s_N,\mathsf s_N',\lambda}(t,s)\|_{\op}\,\dd s.
\end{equation}

Suppose $G_{\rm tr}$ and $K_{\rm ev}$ bound, uniformly over neighboring
$\mathsf s_N,\mathsf s_N'\in\mathsf Z^N$, $\lambda\in\Lambda$,
$t\in[0,T]$, and $z\in\mathsf Z$, the norm
$\|\nabla\ell^{\rm tr}_{\lambda,z}(\theta_t^{\mathsf s_N,\lambda})\|$
by $G_{\rm tr}$ and the Lipschitz constant of
$\ell^{\rm ev}_{\lambda,z}$ on every coupled segment by $K_{\rm ev}$.

\subsubsection{Candidate-index conditions}
\label{app:gf-candidate-index-conditions}

Assume a finite deterministic $C_{\rm met}\geq0$ satisfies
\begin{equation}\label{eq:candidate-metric-comparison}
 |t-t'|+\|\lambda-\lambda'\|
 \leq C_{\rm met}
 d_{\cH}((t,\lambda),(t',\lambda')),
 \qquad (t,\lambda),(t',\lambda')\in\cH.
\end{equation}
Assume all size-$N$ sample flows exist on $[0,T]$ from a common
$\lambda$-independent initialization.  Assume
\[
 \|\nabla\ell^{\rm tr}_{\lambda,z}
      (\theta_t^{\mathsf s_N,\lambda})\|\leq G_{\rm tr}
\]
uniformly over $\mathsf s_N\in\mathsf Z^N$, $t\in[0,T]$,
$\lambda\in\Lambda$, and $z\in\mathsf Z$.  Uniformly over
$\mathsf s_N\in\mathsf Z^N$, $t,t'\in[0,T]$,
$\lambda,\lambda',\widetilde\lambda\in\Lambda$, and $z\in\mathsf Z$, assume
$\theta\mapsto\ell^{\rm ev}_{\widetilde\lambda,z}(\theta)$ is
$K_{\rm ev}$-Lipschitz on the segment joining
$\theta_t^{\mathsf s_N,\lambda}$ and
$\theta_{t'}^{\mathsf s_N,\lambda'}$.  Suppose finite deterministic
$L_{\rm tr}^{\lambda},L_{\rm ev}^{\lambda}\geq0$ satisfy, uniformly over
$\mathsf s_N\in\mathsf Z^N$, $t\in[0,T]$, $z\in\mathsf Z$, and
$\lambda,\lambda'\in\Lambda$,
\begin{equation}\label{eq:finite-horizon-lambda-moduli}
\begin{aligned}
 \|\nabla F_{P_{\mathsf s_N},\lambda}
      (\theta_t^{\mathsf s_N,\lambda'})
   -\nabla F_{P_{\mathsf s_N},\lambda'}
      (\theta_t^{\mathsf s_N,\lambda'})\|
 &\leq L_{\rm tr}^{\lambda}\|\lambda-\lambda'\|,\\
 |\ell^{\rm ev}_{\lambda,z}(\theta_t^{\mathsf s_N,\lambda'})
   -\ell^{\rm ev}_{\lambda',z}(\theta_t^{\mathsf s_N,\lambda'})|
 &\leq L_{\rm ev}^{\lambda}\|\lambda-\lambda'\|.
\end{aligned}
\end{equation}
Assume $F_{P_{\mathsf s_N},\lambda}$ is $C^2$ around every segment
joining same-sample flows at the same time and two hyperparameter values, and
that, for a deterministic $\underline\mu_{\rm hyp}\in\mathbb R$, the
corresponding segment Hessian satisfies
\begin{equation}\label{eq:hyperparameter-segment-hessian}
 \int_0^1\nabla^2F_{P_{\mathsf s_N},\lambda}
 \bigl(\theta_t^{\mathsf s_N,\lambda'}
       +u(\theta_t^{\mathsf s_N,\lambda}
          -\theta_t^{\mathsf s_N,\lambda'})\bigr)\,\dd u
 \succeq\underline\mu_{\rm hyp}I
\end{equation}
uniformly over $\mathsf s_N\in\mathsf Z^N$, $t\in[0,T]$, and
$\lambda,\lambda'\in\Lambda$.  Define
\begin{equation}\label{eq:flow-modulus-constant}
 C_{{\rm mod},T}
 =C_{\rm met}\max\left\{
 K_{\rm ev}G_{\rm tr},
 L_{\rm ev}^{\lambda}
 +K_{\rm ev}L_{\rm tr}^{\lambda}
   \int_0^T e^{-\underline\mu_{\rm hyp}s}\,\dd s
 \right\}.
\end{equation}

\begin{proposition}[Finite-horizon stability and modulus]
\label[proposition]{prop:gf-stability}
Under the neighboring-flow conditions in
\cref{app:gf-neighboring-flow-conditions},
\begin{equation}\label{eq:beta-flow}
 \beta_N\leq\frac{2K_{\rm ev}G_{\rm tr}}{N}\mathfrak A_{N,T}.
\end{equation}
If, in addition,
$\overline H_{\mathsf s_N,\mathsf s_N',\lambda}(t)
\succeq\underline\mu I$ for a deterministic
$\underline\mu\in\mathbb R$, uniformly over neighboring
$\mathsf s_N,\mathsf s_N'\in\mathsf Z^N$, $\lambda\in\Lambda$, and
$t\in[0,T]$, then
\begin{equation}\label{eq:amplification-mu}
 \mathfrak A_{N,T}\leq\int_0^T e^{-\underline\mu s}\,\dd s.
\end{equation}
Under the candidate-index conditions in
\cref{app:gf-candidate-index-conditions}, the hyperparameter-Lipschitz
constant satisfies
\begin{equation}\label{eq:hyper-lipschitz-flow}
 L_N^{\rm hyp}\leq C_{{\rm mod},T}.
\end{equation}
When $N=n-1$, the modulus condition \eqref{eq:hyper-modulus} holds with
\begin{equation}\label{eq:omega-flow}
 \omega_{n-1}(r)=C_{{\rm mod},T}r,
 \qquad r\geq0.
\end{equation}
\end{proposition}

\begin{proof}[Proof of \cref{prop:gf-stability}]
Fix $\mathsf s_N\sim \mathsf s_N'$ and $\lambda\in\Lambda$.  After relabeling the
coordinate of a possible replacement, write
\[
 \mathsf s_N=(z_1,z_2,\ldots,z_N),\qquad
 \mathsf s_N'=(z_1',z_2,\ldots,z_N),
\]
and define the neighboring-flow difference path $e^\lambda:[0,T]\to\Theta$
by $e_t^\lambda=\theta_t^{\mathsf s_N,\lambda}
-\theta_t^{\mathsf s_N',\lambda}$.
The segment
fundamental theorem of calculus and
\cref{eq:stability-segment-hessian} give
\[
 \nabla F_{P_{\mathsf s_N},\lambda}(\theta_t^{\mathsf s_N,\lambda})
 -\nabla F_{P_{\mathsf s_N},\lambda}(\theta_t^{\mathsf s_N',\lambda})
 =\overline H_{\mathsf s_N,\mathsf s_N',\lambda}(t)e_t^\lambda.
\]
Along $\theta_t^{\mathsf s_N',\lambda}$, let
$\mathfrak b_t^{\mathsf s_N,\mathsf s_N',\lambda}\in\Theta$ denote the empirical-field
forcing discrepancy:
\begin{align*}
 \mathfrak b_t^{\mathsf s_N,\mathsf s_N',\lambda}
 &:=\nabla F_{P_{\mathsf s_N},\lambda}(\theta_t^{\mathsf s_N',\lambda})
   -\nabla F_{P_{\mathsf s_N'},\lambda}(\theta_t^{\mathsf s_N',\lambda})\\
 &=\frac1N\{\nabla\ell^{\rm tr}_{\lambda,z_1}
                (\theta_t^{\mathsf s_N',\lambda})
             -\nabla\ell^{\rm tr}_{\lambda,z_1'}
                (\theta_t^{\mathsf s_N',\lambda})\},
 &\|\mathfrak b_t^{\mathsf s_N,\mathsf s_N',\lambda}\|&\leq\frac{2G_{\rm tr}}N.
\end{align*}
Consequently,
\begin{equation}\label{eq:gf-stability-error-variation}
 \dot e_t^\lambda
 =-\overline H_{\mathsf s_N,\mathsf s_N',\lambda}(t)e_t^\lambda
   -\mathfrak b_t^{\mathsf s_N,\mathsf s_N',\lambda},
 \qquad e_0^\lambda=0,
 \qquad e_t^\lambda=-\int_0^t
       \Phi_{\mathsf s_N,\mathsf s_N',\lambda}(t,s)
       \mathfrak b_s^{\mathsf s_N,\mathsf s_N',\lambda}\,\dd s.
\end{equation}
\Cref{eq:gf-stability-error-variation} and
\cref{eq:amplification-factor} imply
\[
 \sup_{0\leq t\leq T}\|e_t^\lambda\|
 \leq\frac{2G_{\rm tr}}N\mathfrak A_{N,T}.
\]
The segment joining $\theta_t^{\mathsf s_N,\lambda}$ to
$\theta_t^{\mathsf s_N',\lambda}$ is a coupled segment, so the assumed
$K_{\rm ev}$-Lipschitz bound gives, uniformly over $z\in\mathsf Z$ and
$0\leq t\leq T$,
\[
 |\ell^{\rm ev}_{\lambda,z}(\theta_t^{\mathsf s_N,\lambda})
   -\ell^{\rm ev}_{\lambda,z}(\theta_t^{\mathsf s_N',\lambda})|
 \leq K_{\rm ev}\|e_t^\lambda\|
 \leq\frac{2K_{\rm ev}G_{\rm tr}}N\mathfrak A_{N,T}.
\]
Taking the supremum in \eqref{eq:uniform-stability} proves
\cref{eq:beta-flow}.

To bound the amplification factor, fix $0\leq s\leq t\leq T$
and $v\in\Theta$.  For $r\in[s,t]$, the uniform bound
$\overline H_{\mathsf s_N,\mathsf s_N',\lambda}(r)
\succeq\underline\mu I$ gives
\[
\begin{aligned}
 \frac12\frac{\dd}{\dd r}
 \|\Phi_{\mathsf s_N,\mathsf s_N',\lambda}(r,s)v\|^2
 &=-\big\langle \Phi_{\mathsf s_N,\mathsf s_N',\lambda}(r,s)v,\\
 &\hspace{17mm}\overline H_{\mathsf s_N,\mathsf s_N',\lambda}(r)
 \Phi_{\mathsf s_N,\mathsf s_N',\lambda}(r,s)v\big\rangle\\
 &\leq-\underline\mu
 \|\Phi_{\mathsf s_N,\mathsf s_N',\lambda}(r,s)v\|^2.
\end{aligned}
\]
Multiplying the resulting scalar inequality by
$e^{2\underline\mu(r-s)}$ gives
\[
 \frac{\dd}{\dd r}
 \left\{e^{2\underline\mu(r-s)}
 \|\Phi_{\mathsf s_N,\mathsf s_N',\lambda}(r,s)v\|^2\right\}\leq0.
\]
Since $\Phi_{\mathsf s_N,\mathsf s_N',\lambda}(s,s)=I$, it follows directly
that $\|\Phi_{\mathsf s_N,\mathsf s_N',\lambda}(t,s)\|_{\op}
 \leq e^{-\underline\mu(t-s)}$.  Uniformly in $0\leq t\leq T$,
\[
 \int_0^t e^{-\underline\mu(t-s)}\,\dd s
 \leq
 \begin{cases}
  (1-e^{-\underline\mu T})/\underline\mu,&\underline\mu>0,\\
  T,&\underline\mu=0,\\
  (e^{|\underline\mu|T}-1)/|\underline\mu|,&\underline\mu<0.
 \end{cases}
\]
Taking the supremum over $\lambda\in\Lambda$, neighboring
$\mathsf s_N,\mathsf s_N'\in\mathsf Z^N$, and $0\leq t\leq T$ in
\eqref{eq:amplification-factor} proves
\cref{eq:amplification-mu}.

For the candidate modulus, fix
$\mathsf s_N\in\mathsf Z^N$ and
$p=(t,\lambda),p'=(t',\lambda')\in\cH$.  The empirical objective is the
average of the training losses, so the $G_{\rm tr}$ bound gives
\begin{equation}\label{eq:finite-horizon-flow-speed}
 \|\dot\theta_s^{\mathsf s_N,\lambda}\|
 =\|\nabla F_{P_{\mathsf s_N},\lambda}
        (\theta_s^{\mathsf s_N,\lambda})\|
 \leq G_{\rm tr},
 \qquad 0\leq s\leq T.
\end{equation}
Consequently,
\begin{equation}\label{eq:finite-horizon-time-flow-modulus}
 \|\theta_t^{\mathsf s_N,\lambda}
       -\theta_{t'}^{\mathsf s_N,\lambda}\|
 \leq G_{\rm tr}|t-t'|.
\end{equation}

For the hyperparameter increment, define $q:[0,T]\to\Theta$ by
$q_s=\theta_s^{\mathsf s_N,\lambda}
-\theta_s^{\mathsf s_N,\lambda'}$ and
\[
\begin{aligned}
 \overline H_{{\rm hyp},s}
 &:=\int_0^1\nabla^2F_{P_{\mathsf s_N},\lambda}
 \bigl(\theta_s^{\mathsf s_N,\lambda'}
       +u(\theta_s^{\mathsf s_N,\lambda}
          -\theta_s^{\mathsf s_N,\lambda'})\bigr)\,\dd u
 \in\mathcal L(\Theta),\\
 b_s
 &:=\nabla F_{P_{\mathsf s_N},\lambda}
       (\theta_s^{\mathsf s_N,\lambda'})
   -\nabla F_{P_{\mathsf s_N},\lambda'}
       (\theta_s^{\mathsf s_N,\lambda'})
 \in\Theta.
\end{aligned}
\]
The segment fundamental theorem of calculus and
\eqref{eq:finite-horizon-lambda-moduli} give
\[
 \dot q_s=-\overline H_{{\rm hyp},s}q_s-b_s,
 \qquad
 \|b_s\|\leq L_{\rm tr}^{\lambda}\|\lambda-\lambda'\|.
\]
When $q_s\ne0$, taking the inner
product with $q_s/\|q_s\|$ and using
\eqref{eq:hyperparameter-segment-hessian} gives the required Dini bound.
When $q_s=0$,
$D^+\|q_s\|=\|\dot q_s\|=\|b_s\|
\leq L_{\rm tr}^{\lambda}\|\lambda-\lambda'\|$.  In both cases,
\[
 D^+\|q_s\|
 \leq-\underline\mu_{\rm hyp}\|q_s\|
       +L_{\rm tr}^{\lambda}\|\lambda-\lambda'\|,
 \qquad q_0=0.
\]
Khalil's scalar comparison lemma
\cite[Lemma~3.4]{khalil2002nonlinear} therefore yields, uniformly over
$0\leq s\leq T$,
\begin{equation}\label{eq:finite-horizon-lambda-flow-modulus}
 \|\theta_s^{\mathsf s_N,\lambda}
       -\theta_s^{\mathsf s_N,\lambda'}\|
 \leq L_{\rm tr}^{\lambda}\|\lambda-\lambda'\|
       \int_0^T e^{-\underline\mu_{\rm hyp}u}\,\dd u.
\end{equation}
For $z\in\mathsf Z$, the triangle inequality followed by
\eqref{eq:finite-horizon-lambda-moduli},
\eqref{eq:finite-horizon-time-flow-modulus}, and
\eqref{eq:finite-horizon-lambda-flow-modulus} gives
\begin{align*}
 &|L_{N,p}(\mathsf s_N,z)-L_{N,p'}(\mathsf s_N,z)|\\
 &\quad\leq
 |\ell^{\rm ev}_{\lambda,z}(\theta_t^{\mathsf s_N,\lambda})
   -\ell^{\rm ev}_{\lambda,z}(\theta_{t'}^{\mathsf s_N,\lambda})|\\
 &\qquad+
 |\ell^{\rm ev}_{\lambda,z}(\theta_{t'}^{\mathsf s_N,\lambda})
   -\ell^{\rm ev}_{\lambda,z}(\theta_{t'}^{\mathsf s_N,\lambda'})|\\
 &\qquad+
 |\ell^{\rm ev}_{\lambda,z}(\theta_{t'}^{\mathsf s_N,\lambda'})
   -\ell^{\rm ev}_{\lambda',z}(\theta_{t'}^{\mathsf s_N,\lambda'})|\\
 &\quad\leq K_{\rm ev}G_{\rm tr}|t-t'|
 +\left\{L_{\rm ev}^{\lambda}
 +K_{\rm ev}L_{\rm tr}^{\lambda}
   \int_0^T e^{-\underline\mu_{\rm hyp}u}\,\dd u\right\}
   \|\lambda-\lambda'\|\\
 &\quad\leq C_{{\rm mod},T}d_{\cH}(p,p'),
\end{align*}
where the bound by $C_{{\rm mod},T}d_{\cH}(p,p')$ uses
\eqref{eq:candidate-metric-comparison} and
\eqref{eq:flow-modulus-constant}.  Taking the supremum over
$\mathsf s_N\in\mathsf Z^N$ and $z\in\mathsf Z$ proves
\eqref{eq:hyper-lipschitz-flow}.  At $N=n-1$,
\eqref{eq:hyper-lipschitz-flow} shows that the choice in
\eqref{eq:omega-flow} satisfies \eqref{eq:hyper-modulus}.
\end{proof}

For fixed $T$, suppose the conditions in
\cref{app:gf-neighboring-flow-conditions,app:gf-candidate-index-conditions}
hold uniformly in $N$.  Suppose also that
$\overline H_{\mathsf s_N,\mathsf s_N',\lambda}(t)
\succeq\underline\mu I$ uniformly over $N$, neighboring
$\mathsf s_N,\mathsf s_N'\in\mathsf Z^N$, $\lambda\in\Lambda$, and
$t\in[0,T]$, and that
$C_{\rm met},G_{\rm tr},K_{\rm ev},L_{\rm tr}^{\lambda},L_{\rm ev}^{\lambda},
\underline\mu$, and $\underline\mu_{\rm hyp}$ are independent of $N$.
Then
\eqref{eq:beta-flow}, \eqref{eq:amplification-mu}, and
\eqref{eq:omega-flow} give
$\beta_{n-1}=O((n-1)^{-1})$, $\omega_{n-1}(r)=O(r)$, and
$2\omega_{n-1}(n^{-1/2})=O(n^{-1/2})$ in
\eqref{eq:sn-definition}.  Direct integration gives
\[
 \int_0^T e^{-\underline\mu_{\rm hyp}s}\,\dd s
 =\begin{cases}
  (1-e^{-\underline\mu_{\rm hyp}T})/\underline\mu_{\rm hyp},
    &\underline\mu_{\rm hyp}>0,\\
  T,&\underline\mu_{\rm hyp}=0,\\
  (e^{|\underline\mu_{\rm hyp}|T}-1)/|\underline\mu_{\rm hyp}|,
    &\underline\mu_{\rm hyp}<0.
 \end{cases}
\]
When $\underline\mu_{\rm hyp}>0$,
$(1-e^{-\underline\mu_{\rm hyp}T})/\underline\mu_{\rm hyp}
\leq\underline\mu_{\rm hyp}^{-1}$.

A sufficient derivative bound for the mixed increment in
\eqref{eq:mixed-stability} uses a deterministic nonempty convex set containing
every flow along the hyperparameter chord hull.

\subsubsection{Mixed-increment conditions}
\label{app:gf-mixed-increment-conditions}

Fix $N\geq1$, let $\Lambda_\star=\operatorname{conv}(\Lambda)$, and fix a
deterministic nonempty convex set $\mathcal K\subseteq\Theta$.  Suppose
a finite deterministic $C_{\rm met}\geq0$ satisfies
\eqref{eq:candidate-metric-comparison}.  Suppose the training and evaluation
losses admit extensions to a deterministic common open set
$\mathcal O\subseteq\mathbb R^{d_\lambda}\times\Theta$ containing
$\Lambda_\star\times\mathcal K$.  Fix the extensions, continue to denote them
by $\ell^{\rm tr}_{\lambda,z}$ and $\ell^{\rm ev}_{\lambda,z}$, and suppose
$(\lambda,\theta)\mapsto\nabla_\theta
\ell^{\rm tr}_{\lambda,z}(\theta)$ is continuously differentiable and
$(\lambda,\theta)\mapsto\ell^{\rm ev}_{\lambda,z}(\theta)$ is continuously
differentiable with a continuous second derivative in $\theta$, for every
$z\in\mathsf Z$.  For
$\mathsf s_N=(z_1,\ldots,z_N)\in\mathsf Z^N$ and every
$(\lambda,\theta)\in\mathcal O$, define
\[
 F_{P_{\mathsf s_N},\lambda}(\theta)
 =\frac1N\sum_{j=1}^N\ell^{\rm tr}_{\lambda,z_j}(\theta)
\]
and suppose the gradient flow generated by $F_{P_{\mathsf s_N},\lambda}$
exists on $[0,T]$
from the same $\lambda$-independent initialization $\theta_0$ and remains in
$\mathcal K$, uniformly over $\mathsf s_N$ and
$\lambda\in\Lambda_\star$.  Denote the resulting flow by
$\theta_t^{P_{\mathsf s_N},\lambda}$.

Fix finite deterministic constants
\[
 \underline\mu_{\rm mix}\in\mathbb R,\qquad
 G_{\rm tr},H_{\rm tr},J_3,B_{{\rm tr},\lambda},
 J_{{\rm tr},\lambda},G_{\rm ev},H_{\rm ev},
 B_{{\rm ev},\lambda},M_{{\rm ev},\lambda}\in[0,\infty).
\]
For $\mathsf s_N\in\mathsf Z^N$, $z\in\mathsf Z$,
$\lambda\in\Lambda_\star$, and $\theta,\theta'\in\mathcal K$, assume
\begin{align}
 \nabla_\theta^2F_{P_{\mathsf s_N},\lambda}(\theta)
 &\succeq\underline\mu_{\rm mix}I,
 &\|\nabla_\theta\ell^{\rm tr}_{\lambda,z}(\theta)\|
 &\leq G_{\rm tr},\nonumber\\
 \|\nabla_\theta^2\ell^{\rm tr}_{\lambda,z}(\theta)\|_{\op}
 &\leq H_{\rm tr},
 &\|\nabla_\theta^2\ell^{\rm tr}_{\lambda,z}(\theta)
       -\nabla_\theta^2\ell^{\rm tr}_{\lambda,z}(\theta')\|_{\op}
 &\leq J_3\|\theta-\theta'\|,
 \label{eq:mixed-training-derivative-bounds}
\end{align}
and, with
$A_{\lambda,z}(\theta)
=D_\lambda\nabla_\theta\ell^{\rm tr}_{\lambda,z}(\theta)
\in\mathcal L(\mathbb R^{d_\lambda},\Theta)$,
\begin{align}
 \|A_{\lambda,z}(\theta)\|_{\op}
 &\leq B_{{\rm tr},\lambda},
 &\|A_{\lambda,z}(\theta)-A_{\lambda,z}(\theta')\|_{\op}
 &\leq J_{{\rm tr},\lambda}\|\theta-\theta'\|,
 \label{eq:mixed-training-lambda-bounds}\\
 \|\nabla_\theta\ell^{\rm ev}_{\lambda,z}(\theta)\|
 &\leq G_{\rm ev},
 &\|\nabla_\theta^2\ell^{\rm ev}_{\lambda,z}(\theta)\|_{\op}
 &\leq H_{\rm ev},\nonumber\\
 \|D_\lambda\ell^{\rm ev}_{\lambda,z}(\theta)\|_{\op}
 &\leq B_{{\rm ev},\lambda},
 &\|D_\lambda\ell^{\rm ev}_{\lambda,z}(\theta)
       -D_\lambda\ell^{\rm ev}_{\lambda,z}(\theta')\|_{\op}
 &\leq M_{{\rm ev},\lambda}\|\theta-\theta'\|.
 \label{eq:mixed-evaluation-derivative-bounds}
\end{align}
Here
$D_\lambda\ell^{\rm ev}_{\lambda,z}(\theta)
\in\mathcal L(\mathbb R^{d_\lambda},\mathbb R)$.  The operator norms in
\cref{eq:mixed-training-lambda-bounds,eq:mixed-evaluation-derivative-bounds}
use the Euclidean norm on $\mathbb R^{d_\lambda}$, the ambient norm on
$\Theta$, and absolute value on $\mathbb R$, as appropriate.

Define the finite constants
\begin{equation}\label{eq:mixed-flow-primitive-constants}
\begin{aligned}
 \phi_{{\rm mix},T}
 &=\int_0^T e^{-\underline\mu_{\rm mix}s}\,\dd s,
 &C_{q,T}&=2G_{\rm tr}\phi_{{\rm mix},T},\\
 C_{\dot q,T}&=H_{\rm tr}C_{q,T}+2G_{\rm tr},
 &U_{\lambda,T}&=B_{{\rm tr},\lambda}\phi_{{\rm mix},T},\\
 C_{u,T}
 &=\left\{(J_3C_{q,T}+2H_{\rm tr})U_{\lambda,T}
        +J_{{\rm tr},\lambda}C_{q,T}
        +2B_{{\rm tr},\lambda}\right\}\phi_{{\rm mix},T},\\
 C_{t,T}
 &=G_{\rm ev}C_{\dot q,T}+H_{\rm ev}G_{\rm tr}C_{q,T},\\
 C_{\lambda,T}
 &=G_{\rm ev}C_{u,T}
   +(M_{{\rm ev},\lambda}+H_{\rm ev}U_{\lambda,T})C_{q,T},\\
 C_{{\rm hyp},T}^{\rm der}
 &=C_{\rm met}\max\{G_{\rm ev}G_{\rm tr},
             B_{{\rm ev},\lambda}+G_{\rm ev}U_{\lambda,T}\},\\
 C_{{\rm mix},T}
 &=C_{\rm met}\max\{C_{t,T},C_{\lambda,T}\}.
\end{aligned}
\end{equation}
For $(t,\lambda)\in[0,T]\times\Lambda_\star$, define the extended
evaluated-loss map
$\overline L_{N,(t,\lambda)}:\mathsf Z^N\times\mathsf Z\to\mathbb R$ by
\[
 \overline L_{N,(t,\lambda)}(\mathsf s_N,z)
 =\ell^{\rm ev}_{\lambda,z}
   (\theta_t^{P_{\mathsf s_N},\lambda}).
\]
For $p\in\cH$, set $L_{N,p}=\overline L_{N,p}$.

\begin{proposition}[Finite-horizon mixed gradient-flow increments]
\label[proposition]{prop:gf-mixed-stability}
Under the mixed-increment conditions in
\cref{app:gf-mixed-increment-conditions}, the constants of the process
$\{L_{N,p}:p\in\cH\}$ in
\cref{eq:hyper-lipschitz-constant,eq:mixed-stability} satisfy
\begin{equation}\label{eq:mixed-flow-constants}
 L_N^{\rm hyp}\leq C_{{\rm hyp},T}^{\rm der},
 \qquad
 \beta_N^{\rm mix}\leq\frac{C_{{\rm mix},T}}N.
\end{equation}
In particular, at $N=n-1$,
\begin{equation}\label{eq:mixed-increment-primitive-bound}
 a_n^{\rm mix}\leq C_{{\rm hyp},T}^{\rm der}+C_{{\rm mix},T}.
\end{equation}
\end{proposition}

\begin{proof}[Proof of \cref{prop:gf-mixed-stability}]
Fix neighboring samples
$\mathsf s_N,\mathsf s_N'\in\mathsf Z^N$.  For every
$\eta\in\Lambda_\star$, define paths
$x^\eta,y^\eta,e^\eta:[0,T]\to\Theta$ by
\[
 x_t^\eta=\theta_t^{P_{\mathsf s_N},\eta},\qquad
 y_t^\eta=\theta_t^{P_{\mathsf s_N'},\eta},\qquad
 e_t^\eta=x_t^\eta-y_t^\eta.
\]
The exact empirical-field decomposition is
\[
\begin{aligned}
 \dot x_t^\eta-\dot y_t^\eta
 &=-\{\nabla F_{P_{\mathsf s_N},\eta}(x_t^\eta)
       -\nabla F_{P_{\mathsf s_N},\eta}(y_t^\eta)\}\\
 &\quad-\{\nabla F_{P_{\mathsf s_N},\eta}(y_t^\eta)
       -\nabla F_{P_{\mathsf s_N'},\eta}(y_t^\eta)\}.
\end{aligned}
\]
The lower-Hessian and individual-gradient bounds in
\eqref{eq:mixed-training-derivative-bounds} give
\[
 D^+\|e_t^\eta\|
 \leq-\underline\mu_{\rm mix}\|e_t^\eta\|+\frac{2G_{\rm tr}}N,
 \qquad e_0^\eta=0.
\]
Khalil's scalar comparison lemma
\cite[Lemma~3.4]{khalil2002nonlinear} gives
\[
 \|e_t^\eta\|
 \leq\frac{2G_{\rm tr}}N
       \int_0^t e^{-\underline\mu_{\rm mix}(t-s)}\,\dd s
 \leq\frac{C_{q,T}}N.
\]
The Hessian and individual-gradient bounds in
\eqref{eq:mixed-training-derivative-bounds} therefore give, uniformly over
$\eta\in\Lambda_\star$ and $0\leq t\leq T$,
\begin{equation}\label{eq:mixed-neighbor-flow-bounds}
 \|e_t^\eta\|\leq\frac{C_{q,T}}N,
 \qquad
 \|\dot x_t^\eta-\dot y_t^\eta\|
 \leq H_{\rm tr}\|e_t^\eta\|+\frac{2G_{\rm tr}}N
 \leq\frac{C_{\dot q,T}}N.
\end{equation}

For $\mathsf s\in\{\mathsf s_N,\mathsf s_N'\}$ and
$\eta\in\Lambda_\star$, compactness of the parameter--state path
$\{(\eta,\theta_t^{P_{\mathsf s},\eta}):0\leq t\leq T\}\subset\mathcal O$
and the $C^1$ vector-field extension give nearby flows through time $T$ and
differentiable dependence on an ambient neighborhood of $\eta$
\cite[Theorem~3.5 and Section~3.3]{khalil2002nonlinear}.  Thus the sensitivity
path
$u^{\mathsf s,\eta}:[0,T]\to
\mathcal L(\mathbb R^{d_\lambda},\Theta)$, defined by
$u_t^{\mathsf s,\eta}=D_\eta\theta_t^{P_{\mathsf s},\eta}$, satisfies
\begin{equation}\label{eq:mixed-lambda-sensitivity-equation}
 \dot u_t^{\mathsf s,\eta}
 =-\nabla_\theta^2F_{P_{\mathsf s},\eta}
       (\theta_t^{P_{\mathsf s},\eta})u_t^{\mathsf s,\eta}
  -D_\eta\nabla_\theta F_{P_{\mathsf s},\eta}
       (\theta_t^{P_{\mathsf s},\eta}),
 \qquad u_0^{\mathsf s,\eta}=0.
\end{equation}
Because $\theta_0$ is hyperparameter-independent,
$u_0^{\mathsf s,\eta}=0$ in
\eqref{eq:mixed-lambda-sensitivity-equation}.
Equations~\eqref{eq:mixed-training-derivative-bounds} and
\eqref{eq:mixed-training-lambda-bounds} give
\[
 D^+\|u_t^{\mathsf s,\eta}\|_{\op}
 \leq-\underline\mu_{\rm mix}\|u_t^{\mathsf s,\eta}\|_{\op}
       +B_{{\rm tr},\lambda},
 \qquad u_0^{\mathsf s,\eta}=0.
\]
Khalil's scalar comparison lemma
\cite[Lemma~3.4]{khalil2002nonlinear} therefore gives, uniformly over
$0\leq t\leq T$,
\[
 \|u_t^{\mathsf s,\eta}\|_{\op}
 \leq B_{{\rm tr},\lambda}
       \int_0^t e^{-\underline\mu_{\rm mix}(t-r)}\,\dd r
 \leq U_{\lambda,T}.
\]
Taking the supremum over the sample and hyperparameter indices gives
\begin{equation}\label{eq:mixed-lambda-sensitivity-bound}
 \sup_{\substack{\mathsf s_N\in\mathsf Z^N,\ \eta\in\Lambda_\star\\
                  0\leq t\leq T}}
 \|u_t^{\mathsf s_N,\eta}\|_{\op}
 \leq U_{\lambda,T}.
\end{equation}

Define the neighboring-sensitivity path
$w^\eta:[0,T]\to\mathcal L(\mathbb R^{d_\lambda},\Theta)$ by
$w_t^\eta=u_t^{\mathsf s_N,\eta}-u_t^{\mathsf s_N',\eta}$.  The exact
add-and-subtract decompositions are
\[
\begin{aligned}
 &\nabla_\theta^2F_{P_{\mathsf s_N},\eta}(x_t^\eta)
  -\nabla_\theta^2F_{P_{\mathsf s_N'},\eta}(y_t^\eta)\\
 &\quad=\{\nabla_\theta^2F_{P_{\mathsf s_N},\eta}(x_t^\eta)
          -\nabla_\theta^2F_{P_{\mathsf s_N},\eta}(y_t^\eta)\}
       +\{\nabla_\theta^2F_{P_{\mathsf s_N},\eta}(y_t^\eta)
          -\nabla_\theta^2F_{P_{\mathsf s_N'},\eta}(y_t^\eta)\},\\
 &D_\eta\nabla_\theta F_{P_{\mathsf s_N},\eta}(x_t^\eta)
  -D_\eta\nabla_\theta F_{P_{\mathsf s_N'},\eta}(y_t^\eta)\\
 &\quad=\{D_\eta\nabla_\theta F_{P_{\mathsf s_N},\eta}(x_t^\eta)
          -D_\eta\nabla_\theta F_{P_{\mathsf s_N},\eta}(y_t^\eta)\}
       +\{D_\eta\nabla_\theta F_{P_{\mathsf s_N},\eta}(y_t^\eta)
          -D_\eta\nabla_\theta F_{P_{\mathsf s_N'},\eta}(y_t^\eta)\}.
\end{aligned}
\]
Equations~\eqref{eq:mixed-training-derivative-bounds},
\eqref{eq:mixed-training-lambda-bounds}, and
\eqref{eq:mixed-neighbor-flow-bounds} therefore give
\begin{align*}
 &\|\nabla_\theta^2F_{P_{\mathsf s_N},\eta}(x_t^\eta)
       -\nabla_\theta^2F_{P_{\mathsf s_N'},\eta}(y_t^\eta)\|_{\op}
 \leq\frac{J_3C_{q,T}+2H_{\rm tr}}N,\\
 &\|D_\eta\nabla_\theta F_{P_{\mathsf s_N},\eta}(x_t^\eta)
       -D_\eta\nabla_\theta F_{P_{\mathsf s_N'},\eta}(y_t^\eta)\|_{\op}
 \leq\frac{J_{{\rm tr},\lambda}C_{q,T}
                  +2B_{{\rm tr},\lambda}}N.
\end{align*}
Define
$H_{x,t}^\eta,H_{y,t}^\eta\in\mathcal L(\Theta)$ and
$a_{x,t}^\eta,a_{y,t}^\eta\in
\mathcal L(\mathbb R^{d_\lambda},\Theta)$ by
\[
\begin{aligned}
 H_{x,t}^\eta
 &=\nabla_\theta^2F_{P_{\mathsf s_N},\eta}(x_t^\eta),
 &H_{y,t}^\eta
 &=\nabla_\theta^2F_{P_{\mathsf s_N'},\eta}(y_t^\eta),\\
 a_{x,t}^\eta
 &=D_\eta\nabla_\theta F_{P_{\mathsf s_N},\eta}(x_t^\eta),
 &a_{y,t}^\eta
 &=D_\eta\nabla_\theta F_{P_{\mathsf s_N'},\eta}(y_t^\eta).
\end{aligned}
\]
Subtracting the two equations in
\eqref{eq:mixed-lambda-sensitivity-equation} gives
\[
 \dot w_t^\eta
 =-H_{x,t}^\eta w_t^\eta
  -(H_{x,t}^\eta-H_{y,t}^\eta)
      u_t^{\mathsf s_N',\eta}
  -(a_{x,t}^\eta-a_{y,t}^\eta),
 \qquad w_0^\eta=0.
\]
By \eqref{eq:mixed-training-derivative-bounds}, the same energy estimate and
integrating-factor argument show that the propagator generated by
$-H_{x,t}^\eta$ has operator norm at most
$e^{-\underline\mu_{\rm mix}(t-s)}$.  Variation of constants and
\eqref{eq:mixed-lambda-sensitivity-bound} give
\[
 \|w_t^\eta\|_{\op}
 \leq\frac1N\int_0^t e^{-\underline\mu_{\rm mix}(t-s)}
 \left\{(J_3C_{q,T}+2H_{\rm tr})U_{\lambda,T}
       +J_{{\rm tr},\lambda}C_{q,T}
       +2B_{{\rm tr},\lambda}\right\}\,\dd s
 \leq\frac{C_{u,T}}N.
\]
Hence
\begin{equation}\label{eq:mixed-neighbor-sensitivity-bound}
 \sup_{\substack{\eta\in\Lambda_\star\\0\leq t\leq T}}
 \|w_t^\eta\|_{\op}
 \leq\frac{C_{u,T}}N.
\end{equation}

For $z\in\mathsf Z$, define the neighboring evaluated-loss gap
$g_z:[0,T]\times\Lambda_\star\to\mathbb R$ by
\[
 g_z(t,\eta)
 =\ell^{\rm ev}_{\eta,z}(x_t^\eta)
  -\ell^{\rm ev}_{\eta,z}(y_t^\eta).
\]
The chain rule and
\cref{eq:mixed-training-derivative-bounds,eq:mixed-evaluation-derivative-bounds,eq:mixed-neighbor-flow-bounds,eq:mixed-lambda-sensitivity-bound,eq:mixed-neighbor-sensitivity-bound} give
\begin{align}
 |\partial_tg_z(t,\eta)|
 &\leq G_{\rm ev}\|\dot x_t^\eta-\dot y_t^\eta\|
       +H_{\rm ev}\|e_t^\eta\|\,\|\dot y_t^\eta\|
 \leq\frac{C_{t,T}}N,
 \label{eq:mixed-loss-time-derivative}\\
 \|D_\eta g_z(t,\eta)\|_{\op}
 &\leq G_{\rm ev}\|w_t^\eta\|_{\op}
       +H_{\rm ev}\|e_t^\eta\|\,
          \|u_t^{\mathsf s_N',\eta}\|_{\op}
       +M_{{\rm ev},\lambda}\|e_t^\eta\|
 \leq\frac{C_{\lambda,T}}N.
 \label{eq:mixed-loss-lambda-derivative}
\end{align}
The bound $\|\dot y_t^\eta\|\leq G_{\rm tr}$ used in
\eqref{eq:mixed-loss-time-derivative} follows from the individual-gradient
bound in \eqref{eq:mixed-training-derivative-bounds}.

For $p=(t,\lambda),p'=(t',\lambda')\in\cH$, integrating first in time and
then along the line segment from $\lambda$ to $\lambda'$, which lies in
$\Lambda_\star$, and applying
\eqref{eq:mixed-loss-time-derivative},
\eqref{eq:mixed-loss-lambda-derivative}, and
\eqref{eq:candidate-metric-comparison} gives
\[
 |g_z(p)-g_z(p')|
 \leq\frac1N\{C_{t,T}|t-t'|
                 +C_{\lambda,T}\|\lambda-\lambda'\|\}
 \leq\frac{C_{{\rm mix},T}}N d_{\cH}(p,p').
\]
Taking the suprema in \eqref{eq:mixed-stability} proves
$\beta_N^{\rm mix}\leq C_{{\rm mix},T}/N$ in
\eqref{eq:mixed-flow-constants}.

For one sample, the chain rule gives
\[
\begin{aligned}
 \partial_t\overline L_{N,(t,\lambda)}(\mathsf s_N,z)
 &=\left\langle\nabla_\theta
      \ell^{\rm ev}_{\lambda,z}(\theta_t^{P_{\mathsf s_N},\lambda}),
      \dot\theta_t^{P_{\mathsf s_N},\lambda}\right\rangle,\\
 D_\lambda\overline L_{N,(t,\lambda)}(\mathsf s_N,z)[v]
 &=D_\lambda\ell^{\rm ev}_{\lambda,z}
      (\theta_t^{P_{\mathsf s_N},\lambda})[v]
   +\left\langle
      \nabla_\theta\ell^{\rm ev}_{\lambda,z}
      (\theta_t^{P_{\mathsf s_N},\lambda}),
      u_t^{\mathsf s_N,\lambda}v\right\rangle,
      \qquad v\in\mathbb R^{d_\lambda}.
\end{aligned}
\]
Thus
\begin{equation}\label{eq:mixed-single-sample-derivative-bounds}
\begin{aligned}
 |\partial_t\overline L_{N,(t,\lambda)}(\mathsf s_N,z)|
 &\leq G_{\rm ev}G_{\rm tr},\\
 \|D_\lambda\overline L_{N,(t,\lambda)}(\mathsf s_N,z)\|_{\op}
 &\leq B_{{\rm ev},\lambda}+G_{\rm ev}U_{\lambda,T}.
\end{aligned}
\end{equation}
Integrating the time and hyperparameter derivative bounds in
\eqref{eq:mixed-single-sample-derivative-bounds} first in time and then along
the line segment from $\lambda$ to $\lambda'$, and applying
\eqref{eq:candidate-metric-comparison}, gives
\[
 |\overline L_{N,p}(\mathsf s_N,z)
   -\overline L_{N,p'}(\mathsf s_N,z)|
 \leq C_{{\rm hyp},T}^{\rm der}d_{\cH}(p,p').
\]
Restricting to $p,p'\in\cH$ and taking the supremum in
\eqref{eq:hyper-lipschitz-constant} proves
$L_N^{\rm hyp}\leq C_{{\rm hyp},T}^{\rm der}$ in
\eqref{eq:mixed-flow-constants}.  Substituting $N=n-1$ into
\eqref{eq:mixed-increment-constant} proves
\eqref{eq:mixed-increment-primitive-bound}.
\end{proof}

If $C_{\rm met}$ and the constants in
\cref{eq:mixed-training-derivative-bounds,eq:mixed-training-lambda-bounds,eq:mixed-evaluation-derivative-bounds,eq:mixed-flow-primitive-constants} are bounded uniformly in $N$ at fixed
$T$, then \cref{eq:mixed-flow-constants,eq:mixed-increment-primitive-bound}
give $a_n^{\rm mix}=O(1)$.

\bibliographystyle{imsart-number}
\bibliography{references}

\end{document}